\documentclass[11pt]{article}
\usepackage[utf8]{inputenc}

\usepackage{natbib}
\usepackage{wustyle}

\usepackage[colorlinks,citecolor=blue,urlcolor=blue,linkcolor=blue,linktocpage=true,pagebackref=true]{hyperref}
\renewcommand*{\backref}[1]{}
\renewcommand*{\backrefalt}[4]{(\small%
    \ifcase #1 not cited%
          \or cited on page~#2%
          \else cited on pages #2%
    \fi%
    )}
    
\usepackage{multirow}
\usepackage{makecell}

\usepackage[none]{hyphenat}
\newcommand{\blfootnote}[1]{%
  \begingroup
  \renewcommand\thefootnote{}%
  \footnote{#1}%
  \addtocounter{footnote}{-1}%
  \endgroup
}

\title{Smoothness Adaptivity in Constant-Depth Neural Networks: Optimal Rates via Smooth Activations}

\author{
    Yuhao Liu$^{1}$\quad
    Zilin Wang$^{2}$\quad
    Lei Wu$^{2,3,4}$\quad
    Shaobo Zhang$^{2}$
      \\
    \\
  $^1$Department of Mathematical Sciences, Tsinghua University
    \\
  $^2$School of Mathematical Sciences, Peking University 
    \\
  $^3$Center for Machine Learning Research, Peking University
  \\
  $^4$AI for Science Institute, Beijing\\[.5em]
  \texttt{yh-liu21@mails.tsinghua.edu.cn},\quad  \texttt{wangzilin@stu.pku.edu.cn}\\ 
  \texttt{leiwu@math.pku.edu.cn},\quad  \texttt{zhangshaobo@stu.pku.edu.cn}\\ 
}
\date{\vspace*{-2em}}

\begin{document}

\maketitle
\blfootnote{All authors contributed equally, and the order follows the alphabetical convention.}

\begin{abstract}
Smooth activation functions are ubiquitous in modern deep learning, yet their theoretical advantages over non-smooth counterparts remain poorly understood. In this work, we study both approximation and statistical  properties of neural networks with smooth activations for learning functions in the Sobolev space $W^{s,\infty}([0,1]^d)$ with $s>0$.
We prove that constant-depth networks equipped with smooth activations achieve {\it smoothness adaptivity}: increasing width alone suffices to attain the minimax-optimal approximation  and estimation error rates (up to logarithmic factors).
In contrast, for non-smooth activations such as ReLU, smoothness adaptivity is fundamentally limited by depth: the attainable approximation order is bounded by depth, and higher-order smoothness requires proportional depth growth.
These results identify activation smoothness as a fundamental mechanism, complementary to depth, for achieving optimal rates over Sobolev function classes.
Technically, our analysis is based on a  multi-scale approximation framework that yields explicit neural network approximators with  controlled parameter norms and model size. This complexity control ensures statistical learnability under empirical risk minimization (ERM) and avoids the impractical $\ell^0$-sparsity constraints commonly required in prior analyses.
\end{abstract}

\doparttoc
\faketableofcontents
\part{}

\vspace*{-1.5em}
\section{Introduction}

Neural networks constitute a central model class in modern machine learning, with applications spanning computer vision, natural language processing, scientific computing, and generative modeling~\citep{krizhevsky2012imagenet,vaswani2017attention,raissi2019physics,radford2019language}. A central theoretical question concerns how architectural design gives rise to their observed performance. Critically, activation functions introduce the nonlinearities that enable  the composition of representations across  layers. Clarifying their structural role and their interaction with other architectural dimensions such as width and depth is therefore essential for a principled understanding of neural network expressivity and generalization.

In the early development of neural networks, sigmoid-type activations were introduced as differentiable surrogates of hard threshold units~\citep{mcculloch1943logical,rumelhart1986learning}.
However, their saturation in the tails leads to near-zero derivatives, resulting in vanishing gradients and hindering the training of deep networks~\citep{hochreiter1991untersuchungen,bengio1994learning}. 
To mitigate this issue, the non-smooth Rectified Linear Unit (ReLU) was introduced and quickly became dominant due to its simplicity and its avoidance of gradient saturation \citep{nair2010rectified,glorot2011deep}. Since then, ReLU has played a crucial role not only in influential architectures such as AlexNet~\citep{krizhevsky2012imagenet} and ResNet~\citep{he2016deep}, but also in much of the theoretical analysis of deep learning~\citep{schmidt2020nonparametric,weinan2019barron}.

Recently,  activation design has witnessed a renewed embrace of smooth activations after years of ReLU dominance---though not as a return to classical sigmoid-type functions. Contemporary choices such as Gaussian Error Linear Unit (GELU)~\citep{hendrycks2016gaussian} and Sigmoid Linear Unit (SiLU)~\citep{ramachandran2017searching, elfwing2018sigmoid}, together with their gated variants like SwiGLU~\citep{shazeer2020glu}, are designed to preserve smoothness while mitigating gradient saturation.
This emphasis on smoothness is particularly evident in scientific computing, where higher-order derivatives  are often required, for example in neural PDE solvers~\citep{raissi2019physics,weinan2018deep,lu2021learning}. Moreover, smooth activations have also become standard components of modern large-scale models. They are adopted in language models such as GPT~\citep{radford2019language}, LLaMA~\citep{grattafiori2024llama3herdmodels}, and DeepSeek~\citep{liu2024deepseek}, and are widely used in vision transformers and diffusion models~\citep{dosovitskiy2021image,ho2020denoising}.

This raises a natural question:
\begin{center}
{\it What are the theoretical advantages of smooth activations over their non-smooth counterparts?}
\end{center}
Early studies of neural networks with smooth activations were rooted in classical approximation theory~\citep{devore1989optimal,mhaskar1996neural,pinkus1999approximation}. These works established approximation error bounds for shallow networks, but typically without explicit control of parameter norms or model complexity. Thus, the approximation properties of smooth activations were not analyzed under complexity constraints that are essential for learning guarantees.
In contrast, much of modern deep learning theory has focused on networks with non-smooth activations, particularly ReLU~\citep{yarotsky2017error,schmidt2020nonparametric}. Consequently, the extent to which activation smoothness itself contributes to approximation efficiency under explicit complexity control has not been systematically characterized.

\subsection{Our Contributions}

In this work, we take a step toward addressing this gap. We study neural networks with smooth activation functions for learning  functions in the Sobolev space  $W^{s,\infty}([0,1]^d)$ for arbitrary  $s>0$. Our analysis is constructive: we explicitly build  network approximators with carefully controlled complexity and derive corresponding finite-sample estimation rates. Together, these results provide a unified characterization of approximation and statistical performance under smooth activations. The main contributions are summarized below.

\begin{itemize}
    \item \textbf{Smoothness adaptivity at constant depth.}
We prove that {\it constant-depth} (depth $6$ or $7$, depending on the metric) 
neural networks equipped with smooth activation functions 
achieve the optimal approximation rate
$
    O(N^{-s/d})
$
for arbitrary smoothness $s>0$, where $N$ denotes the total number of parameters.
Building upon this constructive approximation, we further establish that learning with ERM attains the minimax-optimal estimation rate
$
    O(n^{-2s/(2s+d)})
$
up to logarithmic factors, where $n$ is the sample size. Thus, both approximation and statistical optimality  are achieved at constant depth.
Unlike prior works \citep{bauer2019deep, schmidt2020nonparametric, suzukiadaptivity, ohn2019smooth}, this adaptivity is ``automatic''---it requires neither increasing the network depth  nor imposing intractable $\ell^0$-sparsity constraints.

    \item \textbf{Depth bottleneck for non-smooth activations.}
We establish an approximation lower bound for constant-depth ReLU networks. It shows that non-smooth activations 
cannot attain the minimax-optimal rate for arbitrary smoothness at fixed depth; rather, their achievable rate is intrinsically limited by depth. 
This yields a provable separation between smooth and non-smooth activations. Complementary numerical experiments  demonstrate that shallow networks with smooth activations exhibit faster generalization convergence when learning smooth targets, empirically supporting the theoretical separation.
\end{itemize}

In summary, our results offer a smoothness-adaptivity perspective on the approximation and generalization advantages of smooth activations over their non-smooth counterparts. This perspective provides a principled explanation for the widespread empirical adoption of smooth activations in modern architectures.

Technically, our analysis builds on two main ingredients: 
(i) a novel multiscale approximation scheme for piecewise constant functions that eliminates the need for sparsity constraints (Appendix~\ref{subsec:app_pwcf}); and 
(ii) a weighted superposition principle that lifts localized approximation guarantees to global $L^\infty$ error bounds (Appendix~\ref{subapp:L-infty}).

\paragraph*{Rethinking the role of depth.}
Our findings motivate a reconsideration of the role of depth in existing deep learning theory.
A substantial body of work establishes  generalization guarantees for deep ReLU networks~\citep{yarotsky2017error,liang2017deep,schmidt2020nonparametric,kohler2021rate,suzukiadaptivity}, reinforcing the view that increasing depth is essential for achieving smoothness adaptivity~\citep{telgarsky2016benefits,vardi2020neural}.
By contrast, we show that when smooth activations are employed, constant depth suffices to attain optimal rates over Sobolev classes.
Taken together, these results indicate that depth is not the only mechanism underlying smoothness adaptivity; activation regularity itself provides an alternative and theoretically sufficient route.

\section{Related Work}
The theoretical study of neural networks began with the universal approximation theorem~\citep{cybenko1989approximation,hornik1989multilayer,hornik1991approximation}, establishing that shallow networks can approximate continuous functions on compact domains arbitrarily well. Later work quantified the associated approximation and estimation rates along two main directions.

One line of work studies how neural networks mitigate the curse of dimensionality. A representative result is Barron’s theorem~\citep{barron1993universal}, which identifies function classes with dimension-independent approximation rates. Subsequent work has significantly extended this framework~\citep{devore1998nonlinear,kurkova2002bounds,bach2017breaking,ma2018priori,klusowski2018approximation,weinan2019barron,siegel2020approximation,wu2022spectral,caragea2023neural,siegel2024sharp,chen2025duality}.

Another line of research focuses on neural network approximation over classical function spaces, including Hölder and Sobolev spaces~\citep{mhaskar1995degree,mhaskar1996neural,pinkus1999approximation}. Here, the key question is whether networks can adapt to target function smoothness and attain the corresponding minimax-optimal rates. 
Our work contribute to this line by examining how activation smoothness interacts with depth and width in enabling smoothness adaptivity. We next provide a detailed comparison with prior work below, summarized in Table~\ref{tab:compare-results}.

\paragraph{Approximation results.}
Existing approximation results reveal two distinct mechanisms for achieving smoothness adaptivity.
\begin{itemize}
    \item One mechanism relies on increasing network depth. 
\citet{yarotsky2017error} and \citet{liang2017deep} show that deep ReLU networks with $N$ nonzero parameters achieve the optimal approximation rate $\tilde{O}(N^{-s/d})$ over $W^{s,\infty}([0,1]^d)$~\citep{devore1989optimal}. 
These results establish depth as a fundamental driver of smoothness adaptivity and have inspired a substantial body of subsequent work on deep network approximation~\citep{petersen2018optimal,Bolcskei2019optimal,shen2019nonlinear,Shen_2020,guhring2020error,kohler2021rate,lu2021deep,suzuki2021deep,gribonval2022approximation,hon2022simultaneous,kohler2022estimation,shen2022optimal,kohler2023estimation,siegel2023optimal,yang2023nearly,zhang2024classification,zhang2024deep,liu2025deep,yang2025deep}. 
However, the corresponding constructions require the network depth to grow with the target accuracy $\epsilon$ or smoothness $s$.
\item A second mechanism arises in classical studies of smooth activation functions~\citep{mhaskar1996neural,pinkus1999approximation}.  It is shown that shallow networks with infinitely differentiable, non-polynomial activations can attain optimal approximation rates. 
Yet these constructions lack explicit complexity control and often involve extremely large parameter magnitudes.
Recent works attempt to incorporate complexity control either through $\ell^0$-sparsity constraints~\citep{de2021approximation} or by restricting attention to higher-order non-smooth activations such as ReLU$^k$~\citep{mao2024approximation,yang2025optimal}. 
In these settings, smoothness adaptivity remains tied to depth growth or structural constraints.
\end{itemize}

In contrast, we show that constant-depth networks equipped with general smooth activations achieve full smoothness adaptivity—namely, the optimal approximation rate for arbitrarily high smoothness orders—while maintaining explicit norm control and without imposing intractable $\ell^0$-sparsity constraints. 
In our framework, smoothness adaptivity emerges from activation regularity rather than depth growth.

\paragraph{Generalization results.}
Beyond approximation, a natural question is whether these mechanisms also lead to optimal finite-sample estimation rates. 
For deep ReLU networks, minimax-optimal rates over Sobolev-type spaces have been established~\citep{schmidt2020nonparametric,suzukiadaptivity}. 
These guarantees, however, rely either on $\ell^0$-sparsity constraints or on increasing network depth; even constructions that avoid explicit sparsity still require depth growth~\citep{kohler2021rate}. 
For networks with smooth activations, generalization theory is comparatively less developed. 
Existing results typically impose $\ell^0$-sparsity constraints~\citep{bauer2019deep,ohn2019smooth} or provide non-adaptive guarantees, such as those for $\mathrm{ReLU}^k$ networks~\citep{yang2024nonparametric,yang2025optimal}. 
Building on our complexity-controlled approximation results, we establish minimax-optimal estimation rates for constant-depth networks with general smooth activations, without intractable $\ell^0$-sparsity constraint and without requiring depth growth.

\begin{table}[!t]
    \centering
    \renewcommand{\arraystretch}{1.45}
    \caption{\textbf{Comparison of approximation and learning results under different activation functions and architectural regimes.} 
    We compare prior work with our results along three key dimensions: depth requirement, $\ell^0$-sparsity constraint, and norm control. 
    Unless otherwise stated, the listed methods achieve the optimal approximation rate $O(N^{-s/d})$ and nearly optimal estimation rate $\tilde{O}(n^{-2s/(2s+d)})$, where $N$ and $n$ denote the number of non-zero parameters and samples, respectively. In summary, existing results typically require either (i) depth growing with target accuracy or smoothness, (ii) sparsity constraints, or (iii) smoothness saturation. 
    Our result is the only one that simultaneously achieves constant depth, explicit norm control, no sparsity constraints, and adaptivity to arbitrarily high smoothness orders.
}
    \vspace{6pt}
    \resizebox{0.95\textwidth}{!}{ 
    \begin{tabular}{|c|c|c|c|c|c|}
    \hline
    Activation & Reference & Depth &
    \makecell[c]{Free of\\ $\ell^0$-sparsity} &
    \makecell[c]{Norm\\ control} &
    Remark\\
    \hline
    \multirow{4}{*}{ReLU} & \makecell[c]{~\\[-9pt]\citet{yarotsky2017error}, \\[1pt] \citet{liang2017deep}, \\[1pt] \citet{schmidt2020nonparametric}\\[-10pt]~} & $O(\log(\frac{1}{\epsilon}))$ & $\times$ & $\checkmark$ &\\
    \cline{2-6}
    & \makecell[c]{~\\[-9pt] \citet{kohler2021rate}\\[-10pt]~} & $O(\log(\frac{1}{\epsilon}))$ & $\checkmark$ & $\checkmark$ & \\
    \cline{2-6}
    & \makecell[c]{\\[-9pt]\citet{petersen2018optimal}, \\[1pt] \citet{nakada2020adaptive}\\[-10pt]~} & $O(s\log(s))$ & $\times$ & $\checkmark$ & \\
    \cline{2-6}
    & \makecell[c]{~\\[-9pt] \citet{yang2024nonparametric}\\[-10pt]~} & $2$ & $\checkmark$ & $\checkmark$ & $s<\frac{d+3}{2}$ \\ 
    \hline
    {$\mathrm{ReLU}^k$} & \makecell[c]{~\\[-9pt]\citet{petrushev1998approximation}, \\[1pt] \citet{pinkus1999approximation}, \\[1pt] \citet{yang2025optimal}\\[-10pt]~} & $2$ & $\checkmark$ & $\checkmark$ & $s < \frac{2k+d+1}{2}$ \\
    \hline 
    \multirow{4}{*}{Smooth} & \makecell[c]{~\\[-9pt]\citet{mhaskar1996neural}, \\[1pt] \citet{pinkus1999approximation}\\[-10pt]~} & $2$ & $\checkmark$ & $\times$ & \makecell[c]{No learning\\ guarantees}\\
    \cline{2-6}
    & \makecell[c]{~\\[-9pt]\citet{de2021approximation}, \\[1pt] \citet{bauer2019deep}\\[-10pt]~} & $3$ & $\times$ & $\checkmark$ & \\
    \cline{2-6}
    & \makecell[c]{~\\[-9pt] \citet{ohn2019smooth}\\[-10pt]~} & $O(\log(\frac{1}{\epsilon}))$ & $\times$  & $\checkmark$ &  \\
    \cline{2-6}
    & \textbf{Ours} & $6,7$ & $\checkmark$  & $\checkmark$ & \\
    \hline
    \end{tabular}
    }
\label{tab:compare-results}
\end{table}

\section{Preliminaries}
\label{sec:preliminary}

\paragraph*{Notations.}

We write $\mathbb{N} \coloneqq \{0,1,2,\dots\}$ for the set of non-negative integers, and
$\mathbb{N}^d \coloneqq \{(\alpha_1,\dots,\alpha_d) : \alpha_i\in\mathbb{N},\, i=1,\dots,d\}$
for its $d$-fold Cartesian product. For a multi-index $\balpha=(\alpha_1,\dots,\alpha_d)\in\mathbb{N}^d$,
we define $|\balpha|\coloneqq \sum_{i=1}^d \alpha_i$.
For a positive integer $K$, we write $[K]\coloneqq\{1,2,\dots,K\}$ and $[K]^d$ for its $d$-fold Cartesian product.
For non-negative functions $f$ and $g$, we write $f(x)\lesssim g(x)$ (equivalently, $f(x)=O(g(x))$)
to mean that there exists a constant $C>0$ such that $f(x)\le C\,g(x)$ for all $x$ under consideration.
We write $f(x)=\tilde{O}(g(x))$ to suppress polylogarithmic factors, i.e.,
$f(x)=O\bigl(g(x)\polylog(x)\bigr)$.
Moreover, we write $f(x)\eqsim g(x)$ if both $f(x)\lesssim g(x)$ and $g(x)\lesssim f(x)$ hold.
For a region $D\subset\mathbb{R}^d$, the indicator function is denoted by $\mathbbm{1}_{D}(\cdot)$,
i.e., $\mathbbm{1}_{D}(\bx)=1$ if $\bx\in D$ and $0$ otherwise.

\paragraph*{The target function class.}
Throughout the paper, we fix the domain $\Omega = [0,1]^d$ and consider the Sobolev space $W^{s,\infty}(\Omega)$ with  $s > 0$, which serves as the target function class in our analysis. Concretely, $W^{s,\infty}(\Omega)$ is defined as follows:
    \begin{itemize}
        \item (Integer-order Sobolev spaces) For $s \in \mathbb{N}$, the space $W^{s,\infty}(\Omega)$ consists of functions $u \in L^{\infty}(\Omega)$ such that the weak derivatives $D^{\balpha} u \in L^{\infty}(\Omega)$ for all $\balpha \in \mathbb{N}^d$ with $|\balpha| \leq s$. The norm is given by
        \begin{equation*}
            \|u\|_{W^{s,\infty}(\Omega)} = \max_{|\balpha| \leq s} \|D^{\balpha} u\|_{L^{\infty}(\Omega)}, \quad u \in W^{s,\infty}(\Omega).
        \end{equation*}
        \item (Fractional-order Sobolev spaces) Let $s=m+\zeta$, where $m=\lfloor s \rfloor$ and $\zeta \in (0,1)$. 
For $u \in W^{m,\infty}(\Omega)$, define the Hölder seminorm
\[
[u]_{C^{m,\zeta}(\Omega)}
=
\max_{|\balpha|=m}
\sup_{\bx \neq \by \in \Omega}
\frac{|D^{\balpha} u(\bx)-D^\balpha u(\by)|}{\|\bx-\by\|_2^\zeta}.
\]
We define
\[
W^{s,\infty}(\Omega) = \{u\in W^{m,\infty}(\Omega)\,:\, [u]_{C^{m,\zeta}(\Omega)}<\infty\}
\]
equipped with the norm
\[
\|u\|_{W^{s,\infty}(\Omega)}
=
\|u\|_{W^{m,\infty}(\Omega)}
+
[u]_{C^{m,\zeta}(\Omega)}.
\]
    \end{itemize}

\paragraph*{The neural network model.}
We consider the model of fully-connected   networks
with activation function $\phi:\mathbb{R}\to\mathbb{R}$,
applied componentwise.
For depth $L\ge 1$, width $M\ge 1$, and parameter bound $B>0$, let
\begin{equation*}
\mathcal{H}^{\phi, L}(d_{\mathrm{in}},d_{\mathrm{out}},M,B)
=
\Big\{
\bx\mapsto 
\bW_L\,\phi\!\big(
\bW_{L-1}\phi(\cdots \phi(\bW_1\bx+\bb_1)\cdots)+\bb_{L-1}
\big)
+\bb_L
:\ (\bW_\ell,\bb_\ell)_{\ell=1}^L\in\Theta
\Big\},
\end{equation*}
where the parameter set $\Theta=\Theta(L,d_{\mathrm{in}},d_{\mathrm{out}},M,B)$ is given by
\[
\Theta :=
\Big\{
(\bW_\ell,\bb_\ell)_{\ell=1}^L:
\; \bW_1\in\mathbb{R}^{M\times d_{\mathrm{in}}},\;
\bW_\ell\in\mathbb{R}^{M\times M}\ (2\le \ell\le L-1),\;
\bW_L\in\mathbb{R}^{d_{\mathrm{out}}\times M},
\]
\[
\qquad\qquad
\bb_\ell\in\mathbb{R}^{M}\ (1\le \ell\le L-1),\;
\bb_L\in\mathbb{R}^{d_{\mathrm{out}}},
\;
\|\bW_\ell\|_{\infty,\infty}\le B,\;
\|\bb_\ell\|_\infty\le B\ (1\le \ell\le L)
\Big\},
\]
where  $\|\bW\|_{\infty,\infty} := \max_{i,j} |W_{ij}|$  and
$\|\bb\|_\infty := \max_i |b_i|$.

\paragraph*{The activation assumptions.}
Our primary objective is to understand how structural properties 
of the activation function influence the approximation power 
of neural networks. Specifically, we impose the following assumptions on $\phi$.
\begin{assumption}[Smoothness]
     $\phi$ is infinitely differentiable and is not a polynomial.
    \label{ass:smooth}
\end{assumption}

\begin{assumption}[Lipschitz continuity]
     $|\phi(x) - \phi(y)| \leq \|\phi\|_{\mathrm{Lip}} |x-y|$ for all $x, y\in \RR$. 
    \label{ass:Lip}
\end{assumption}
This  ensures that the activation grows at most linearly at infinity. It excludes activation functions with super-linear growth, such as SwiGLU~\citep{shazeer2020glu}, whose tail behavior is quadratic. 
We emphasize that this assumption is imposed for technical simplicity; our analysis can be extended to activations with polynomial growth tails with minor modifications.

\begin{assumption}
    The activation function $\phi$ satisfies one of the following:
    \begin{itemize}
        \item (Heaviside-like) There exists a constant $C_1 > 0$ such that 
        \begin{equation}
            \lvert \phi(t) - H(t) \rvert \leq C_1 \min \left\{ 1, \frac{1}{\lvert t \rvert} \right\}, \quad \forall\, t \in \mathbb{R},
            \label{eq:Heaviside-like}
        \end{equation}
        where  $H(t) = \mathbbm{1}_{[0,\infty)}(t)$ is the Heaviside function.
        \item (ReLU-like) There exists a constant $C_2 > 0$ such that 
        \begin{equation}
            \lvert \phi(t) - \mathrm{ReLU}(t) \rvert \leq C_2, \quad \forall\, t \in \mathbb{R},
            \label{eq:ReLU-like}
    \end{equation}
    where $\mathrm{ReLU}(t) = \max\{t,0\}$ is the ReLU function.
    \end{itemize}
    \label{ass:piece}
\end{assumption}

\paragraph*{Examples.}
Most activation functions used in practice satisfy the above assumptions. We list several representative examples together with their explicit forms:
\begin{align*}
{\mathrm{sigmoid}}(t) &= \frac{1}{1+e^{-t}},\qquad
\tanh(t) = \frac{e^{t}-e^{-t}}{e^{t}+e^{-t}},\\
\mathrm{SiLU}(t) &= t\,{\mathrm{sigmoid}}(t),\quad 
\mathrm{GELU}(t) = t\,\Phi(t),
\end{align*}
where $\Phi(t)=\int_{-\infty}^{t}\frac{1}{\sqrt{2\pi}}e^{-u^2/2}\,\dd u$ is the cumulative distribution function of $\cN(0,1)$.
In particular, $\mathrm{GELU}$ and $\mathrm{SiLU}$ are ReLU-like, whereas $\phi_{\tanh}(\cdot)=\tfrac12(1+\tanh(\cdot))$ and the sigmoid function are Heaviside-like.
See Appendix~\ref{app:verification} for detailed verification.

\section{Approximation Theory: Smoothness as an Alternative to Depth}
\label{subsec: construct-app}

The following theorem establishes an $L^2$-approximation  of functions in $W^{s,\infty}([0,1]^d)$ using constant-depth neural networks.
\begin{theorem}[$L^2$ approximation] \label{thm:app_L2}
    Let $\phi$ satisfy Assumptions~\ref{ass:smooth}--\ref{ass:piece}. For any $s>0$ and any $f^\star \in W^{s,\infty}([0,1]^d)$ with 
$\|f^\star\|_{W^{s,\infty}([0,1]^d)} \le 1$, and for any $\epsilon \in (0,1)$,  there exists a constant-depth neural network
\[
g \in \mathcal{H}^{\phi,L}(d,1,M_{\epsilon},B_{\epsilon})
\]
with
    \begin{equation}
       L=6, \qquad M_{\epsilon} \lesssim \left( \frac{1}{\epsilon}\right)^{\frac{d}{2s}}, \quad B_{\epsilon} \lesssim \left( \frac{1}{\epsilon}\right)^{\max \left\{\frac{d^2}{2s}+d, \frac{d}{s}+2, \frac{d+4}{2s}+5, \lceil s \rceil\right\}},
        \label{eq:norm_width_app_L2}
    \end{equation}
    such that 
    \begin{equation*}
        \left\| g - f^\star\right\|_{L^2([0,1]^d)} \leq \epsilon.
    \end{equation*}
\end{theorem}
A proof sketch of this theorem is provided in Section~\ref{sec:proof_sketch_of_theorem_ref_thm_app_l2}, and the detailed proof can be found in Appendix~\ref{app:approximation}.

Specifically, to achieve approximation accuracy $\epsilon$, the required network width satisfies 
$M_\epsilon \lesssim \epsilon^{-\frac{d}{2s}}$. 
Since the depth is constant, the total number of parameters scales as 
$N = O(M_\epsilon^2) = O(\epsilon^{-\frac{d}{s}})$. 
Equivalently, in terms of the parameter budget $N$, this yields the approximation rate
\[
\|g - f^\star\|_{L^2([0,1]^d)} \lesssim N^{-s/d},
\]
which matches the optimal stable approximation rate for $W^{s,\infty}([0,1]^d)$~\citep{devore1998nonlinear}. In comparison, \citet{de2021approximation} obtain a comparable approximation order only under explicit $\ell^0$-sparsity constraints on the number of nonzero parameters, while the total parameter count grows at a strictly higher polynomial order in $\epsilon^{-1}$. Furthermore, our construction guarantees that the parameter norms are polynomially controlled, i.e., $B_\epsilon = O(\mathrm{poly}(\epsilon^{-1}))$, a vital property that ensures statistical learnability from finite samples.

\begin{remark}[Constant depth]
Theorem~\ref{thm:app_L2} shows that smooth activation functions enable constant-depth networks to adapt to arbitrarily high smoothness orders of the target function. Notably, the required depth remains fixed (here $L=6$), independent of both the target accuracy $\epsilon$ and the smoothness level $s$.
By contrast, existing approximation results for non-smooth activations typically require the network depth to grow either with the target accuracy (e.g., \citet{yarotsky2017error}) or with the smoothness level (e.g., \citet{petersen2018optimal,lu2021deep}).
\end{remark}

\begin{remark}[Width--norm trade-off]
Theorem~\ref{thm:app_L2} provides a sufficient joint scaling of network width and parameter norm to achieve a target approximation accuracy. In particular, it guarantees the existence of $(M_\epsilon, B_\epsilon)$ with polynomial dependence on $\epsilon^{-1}$.
It, however, does not fully characterize the trade-off between width and parameter norm. That is, we do not determine the best achievable approximation accuracy under prescribed width $W$ and parameter-norm bound $B$. Such a characterization would provide a more refined understanding of the approximation process. For our purposes, polynomial control of both width and norm suffices, as it enables the subsequent generalization analysis.
\end{remark}

As a refinement of Theorem~\ref{thm:app_L2}, we also establish an $L^{\infty}$ approximation result.
\begin{theorem}[$L^\infty$ approximation]
    Suppose $\phi$ and $f^\star$ satisfy the assumptions of Theorem~\ref{thm:app_L2}. Then, for any $\epsilon\in(0,1)$, there exists a neural network
    \begin{equation*}
        g \in \mathcal{H}^{\phi,L} (d, 1, M_{\epsilon}, B_{\epsilon}),
    \end{equation*}
    with 
    \begin{equation}
        L=7, \quad M_{\epsilon} \lesssim \left( \frac{1}{\epsilon}\right)^{\frac{d}{2s}}, \quad B_{\epsilon} \lesssim \left( \frac{1}{\epsilon}\right)^{\max \left\{ \frac{d^2}{2s}+d, \frac{d}{s}+2, \frac{d+4}{2s}+1, \lceil s \rceil, \frac{6}{s}+4\right\}} ,
        \label{eq:app-Linfty-width-norm}
    \end{equation}
    such that 
    \begin{equation*}
        \|g - f^\star\|_{L^{\infty}([0,1]^d)} \leq \epsilon. 
    \end{equation*}
    \label{thm:app_Linfty}
\end{theorem}
Compared with the $L^2$ approximation guarantee in Theorem~\ref{thm:app_L2}, the $L^\infty$ result only requires a modest increase in depth from $6$ to $7$. The proof for $L^\infty$ approximation follows a strategy similar to that of Theorem~\ref{thm:app_L2}, with the addition of a single layer designed to implement a weighted superposition principle. This mechanism effectively upgrades localized approximation guarantees to a global $L^\infty$ bound. For a complete derivation, we refer the reader to Appendix~\ref{subapp:L-infty}.

\section{Learning Theory: Achieving Optimal Risk without Sparsity}
\label{subsec: generalization}

We now leverage the above constructive approximation results  to derive  generalization bounds  for learning target functions in $W^{s,\infty}([0,1]^d)$.

Let $\{(\bx_i,y_i)\}_{i=1}^n$ be i.i.d.\ samples generated according to
\begin{equation}
\label{eq:data_nonparametric}
y_i = f^\star(\bx_i) + \xi_i, 
\qquad i=1,\dots,n,
\end{equation}
where  $\bx_i \stackrel{iid}{\sim} \rho$ with $\rho$ being a distribution  supported on $[0,1]^d$ 
and the noises $\xi_i \sim \mathcal{N}(0,\sigma^2)$ are independent of $\bx_i$. We consider ERM over the network class:
\begin{equation} 
    \widehat{f}_n = \argmin_{f \in \mathcal{H}^{\phi,L}(d, 1,M_n,B_n)} \frac{1}{n} \sum_{i=1}^n \Big( y_i -(\mathbb{T}_F f)(\bx_i)\Big)^2.
    \label{eq:ERM}
\end{equation}
Here $\mathbb{T}_{F}$ denotes the truncation operator, defined for $F>0$ by
\begin{equation*}
    (\mathbb{T}_{F} f)(\bx) = \begin{cases}
        f(\bx), &\text{if} \;|f(\bx)| \leq F,\\
        \operatorname{sign}(f(\bx)) F, & \text{if}\; |f(\bx)|>F.
    \end{cases}
\end{equation*}
This truncation ensures uniform boundedness of the hypothesis class and is standard in generalization  analysis of ERM estimators.

\begin{theorem}
\label{thm:sample-com}
 Let $\phi$ satisfy Assumptions~\ref{ass:smooth}--\ref{ass:piece} and assume the noise variance $\sigma^2\gtrsim 1$. Fix any $s>0$ and let $f^\star \in W^{s,\infty}([0,1]^d)\cap C([0,1]^d)$ satisfy $\|f^\star\|_{W^{s,\infty}([0,1]^d)}\le 1$. Then, for any positive integer $n$, choosing
    \begin{equation}
        L= 7, \quad M_n \eqsim  n^{\frac{d}{4s+2d}} , \quad  B_n \eqsim n^{\max \left\{ \frac{d}{2}, 1, \frac{2s+d+4}{2(2s+d)}, \frac{s \lceil s\rceil}{2s+d},\frac{4s+6}{2s+d} \right\}} , \quad F = 2,
        \label{eq:est_LMBF}
    \end{equation}
    we have
    \begin{equation}
        \mathbb{E}_{} \left[ \left\| \mathbb{T}_{F} \widehat{f}_n - f^\star\right\|_{L^2(\rho)}^2  
        \right] \lesssim  n^{-\frac{2s}{2s + d}} \log n,
        \label{eq: gen_rate_upper}
    \end{equation}
     where  the expectation is taken over the sampling of the training data.
     \label{thm:est}
\end{theorem}
We provide a proof sketch of this theorem in Section~\ref{sec:proof_sketch_of_theorem_ref_thm_est}, with the complete derivation deferred to Appendix~\ref{app:est}. The fundamental insight of the proof lies in the fact that the constructive approximation in Theorem~\ref{thm:app_L2} allows for precise control over the complexity of the hypothesis class.

It is well known that the minimax optimal rate 
for learning functions in the Sobolev class 
$W^{s,\infty}([0,1]^d)$ 
is $O(n^{-\frac{2s}{2s+d}})$ 
\citep{stone1982optimal,tsybakov2009introduction}. Therefore,
Theorem~\ref{thm:sample-com} shows that, 
when equipped with smooth activation functions, 
constant-depth neural networks achieve this optimal rate 
(up to logarithmic factors). In contrast, for non-smooth activations,  achieving the same rate typically requires the network depth  to grow with the sample size~\citep{schmidt2020nonparametric,suzukiadaptivity,ohn2019smooth}. Compared to these prior works, we successfully remove the requirement of $\ell^0$-sparsity constraint, thereby rendering the ERM practically implementable.

In Theorem~\ref{thm:sample-com}, empirical risk minimization is performed over neural networks whose parameters satisfy an $\ell^\infty$-norm constraint. A slight modification of the proof shows that the same optimal risk bound holds under the more commonly used $\ell^2$-norm constraint (Theorem~\ref{thm:est-l2} in Appendix~\ref{app:est}). Given the close connection between $\ell^2$ regularization and weight decay techniques widely used in practice, this extension further aligns our theoretical guarantees with standard training procedures.

\begin{remark}
In Theorem~\ref{thm:sample-com}, no structural assumption is imposed on the input distribution $\rho$. 
This distribution-free guarantee stems from the $L^\infty$ approximation result (Theorem~\ref{thm:app_Linfty}). 
The underlying reason is that $L^\infty$ control provides uniform pointwise error bounds, which immediately control the $\rho$-weighted $L^2$ risk for any probability measure $\rho$. 
By contrast, $L^2$ approximation with respect to the Lebesgue measure only ensures average control under a reference measure and does not preclude the error from concentrating on regions where $\rho$ assigns substantial mass. 
Consequently, employing the $L^2$ approximation result (Theorem~\ref{thm:app_L2}) would in general necessitate the additional assumption that $\rho$ admits a uniformly bounded density in order to transfer the approximation result to a generalization guarantee.
\end{remark}

\section{The Depth Bottleneck for Non-Smooth Activations}
\label{sec:separation}

We first establish a quantitative limitation of constant-depth ReLU networks,  whose proof is deferred to Appendix~\ref{app:lower-bound}.

\begin{proposition}[Approximation lower bound for constant-depth ReLU networks]
\label{prop:relu_lower_bound}
    Fix a depth $L\ge 2$ and a smoothness parameter $s>0$. Then there exists a constant
    $C_{s,L}>0$, depending only on $s$ and $L$, such that for every $M\ge 2$,
    \[
    \sup_{\|f^\star\|_{W^{s,\infty}([0,1])}\le 1}\;\,
    \inf_{g\in \mathcal{H}^{\mathrm{ReLU},L}(1,1,M,\infty)}
    \|g-f^\star\|_{L^2([0,1])}
    \;\ge\;
    C_{s,L}\,(M\log M)^{-2\min\{L-1,s\}} .
    \]
\end{proposition}

For fixed depth $L$, the total number of parameters satisfies 
$
N \asymp M^2 L \asymp M^2.
$
Therefore, up to logarithmic factors, the approximation rate is lower bounded by 
$N^{-\min\{L-1,s\}}$, which saturates at order $N^{-(L-1)}$ once $s > L-1$. 
This shows that constant-depth ReLU networks cannot achieve smoothness adaptivity for arbitrarily large $s$ by increasing width alone. This stands in sharp contrast to Theorem~\ref{thm:app_L2}, where constant-depth networks with smooth activations achieve approximation rates of order $N^{-s}$  for arbitrary $s>0$. Hence, smooth activations enable full smoothness adaptivity even at fixed depth, whereas ReLU networks exhibit an intrinsic smoothness ceiling determined by depth. For simplicity, we consider only the one-dimensional case. We expect analogous bounds to hold for general $d \in \mathbb{N}^+$, although a full treatment is left for future work.

Intuitively, this saturation phenomenon stems from the piecewise linear structure of ReLU networks. 
For fixed depth $L$, a ReLU network represents a piecewise linear function whose number of linear regions grows at most polynomially in the width $M$, with the polynomial degree determined by $L$ (independent of $s$). 
Consequently, the effective approximation order is fundamentally limited, preventing the network from fully exploiting higher-order smoothness.
If the depth $L$ is allowed to grow with the total number of parameters (or equivalently, with the target accuracy) or with the smoothness level $s$, then smoothness adaptivity can be recovered~\citep{yarotsky2017error,liang2017deep,lu2021deep}. 
For example, \citet{lu2021deep} establish that the approximation error of ReLU networks satisfies
$
O(M^{-2s/d}),
$
provided that the depth satisfies $L \gtrsim s^2$.

\paragraph*{Numerical evidence for generalization superiority.}
Having established a sharp approximation separation at constant depth, we now investigate its implications for finite-sample learning. 
Deriving a sharp \emph{generalization} lower bound for constant-depth neural networks with non-smooth activation functions---analogous to Proposition~\ref{prop:relu_lower_bound}---is technically challenging. 
The main difficulty lies in the fact that classical information-theoretic tools for lower bounds, such as Fano's inequality~\citep{cover1999elements,tsybakov2009introduction}, are formulated in a minimax framework over \emph{all} estimators and therefore do not directly capture model-specific structural constraints (e.g., fixed depth and non-smooth activations). 
Obtaining model-specific lower bounds of this type remains an interesting open direction.

Nevertheless, we provide empirical evidence  that supports the generalization separation.
We generate a smooth target function using random Fourier features and learn it using  two-layer neural networks equipped with various activation functions. Training is performed using full-batch Adam optimizer to minimize the empirical risk. For each activation, the learning rate and  $\ell^2$-regularization hyperparameter are tuned  over the same grid, and we report the best achieved performance. Further implementation details are deferred to  Appendix~\ref{subapp:est_lower}.

\begin{figure}[htbp]
    \centering
    \includegraphics[width=0.4\textwidth]{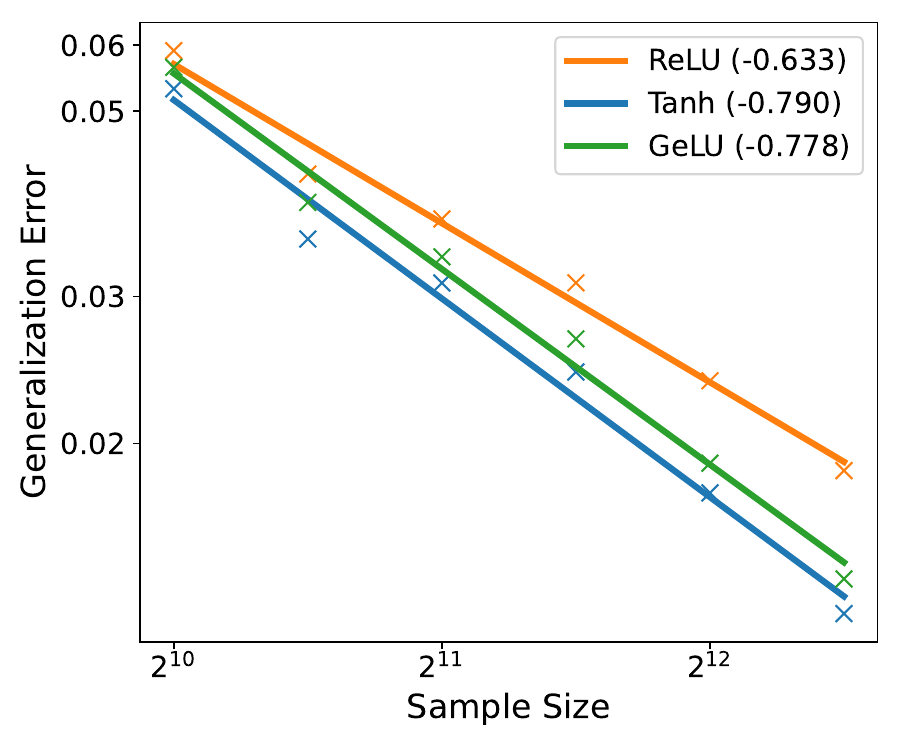}
    \caption{\textbf{Generalization error versus sample size for two-layer networks trained with different activation functions.} Markers denote the measured generalization errors at each sample size (averaged over 5 runs), and solid lines show least-squares fits of the form $E(n)\propto n^{-\alpha}$. The fitted exponents $\alpha$, reported in the legend, indicate a faster decay of the generalization error for smooth activations as the sample size increases.}
    \label{fig:log_log_convergence}
\end{figure}

Figure~\ref{fig:log_log_convergence} shows the log--log scaling of generalization error versus sample size. 
Smooth activations (tanh and GELU) exhibit a steeper decay slope than ReLU, consistent with our theory.
While optimization effects cannot be completely ruled out, these empirical results support that smooth activations enable constant-depth networks to better exploit target smoothness, thereby improving sample efficiency when learning smooth functions.

\section{Proof Sketches}
\subsection{Proof Sketch of Theorem~\ref{thm:app_L2}}
\label{sec:proof_sketch_of_theorem_ref_thm_app_l2}

We approximate $f^\star$ using piecewise polynomials as an intermediate representation. This reduces the problem to three building blocks: (i) monomials, (ii) piecewise constant functions, and (iii) products of these two components; see Figure~\ref{fig:app-L2} for an illustration. Steps (i) and (iii) are implemented via finite-difference approximations of derivatives, a classical technique in neural network approximation \citep{pinkus1999approximation}.

For step (ii), we employ a multiscale construction based on a coarse-to-refined grid partition. This allows us to represent a piecewise constant function with $K^{2d}$ refined cells using a constant-depth network of width $O(K^d)$. Concretely, let
\[
C(\bx)=\sum_{\bi\in[K]^d}\sum_{\bj\in[K]^d} c_{\bi,\bj}\,\mathbbm{1}_{\Omega^{K}_{\bi,\bj}} (\bx)
\]
be a piecewise constant function on the refined partition $\{\Omega^{K}_{\bi,\bj}\}_{\bi,\bj\in[K]^d}$, where $\{\Omega^{K}_{\bi}\}_{\bi\in[K]^d}$ denotes the coarse partition and, for each coarse cell $\Omega_{\bi}^{K}$, the sets $\{\Omega^{K}_{\bi,\bj}\}_{\bj\in[K]^d}$ form its $K^d$ refined subcells (see Figure \ref{fig:app-L2}). Denote by $\ba_{\bi}^{K}$ a fixed reference point of $\Omega_{\bi}^{K}$ (e.g., its lower-left corner). Then one can write
\[
C(\bx) =\sum_{\bj \in [K]^d} \left( \sum_{\bi \in [K]^d} c_{\bi, \bj} \mathbbm{1}_{\Omega_{\bi}^K}(\bx) \right) \mathbbm{1}_{\Omega_{ \bone,\bj}^{ K}} \left(\bx - \sum_{\bi \in [K]^d} \ba_{\bi}^K  \mathbbm{1}_{\Omega_{\bi}^K} (\bx)\right).
\]
where $\bx-\ba_{\bi}^{K}$ maps $\bx$ to its local position within that coarse cell. Observe that the constituent components include functions $\sum_{\bi \in [K]^d} c_{\bi, \bj} \mathbbm{1}_{\Omega_{\bi}^K} $ and $\sum_{\bi \in [K]^d} \ba_{\bi}^K \mathbbm{1}_{\Omega_{\bi}^K} $, which are piecewise constant with respect to the same coarse grid partition of size $K^d$, and the collection of $K^d$ indicator functions $\mathbbm{1}_{\Omega_{\bone, \bj}^{ K}} $ on refined cells. Consequently, each component admits an efficient approximation by a constant-depth network of width $O(K^d)$, thereby establishing an overall width bound of order $O(K^d)$ for the approximation of $C$. Figures~\ref{fig:app-L2}(b) and (c) provide a visual illustration of this approximation for $d = 1$ and $K = 2$.
See Appendix \ref{subsec:app_pwcf} for details. 

\begin{remark}
The multiscale decomposition in step (ii) is crucial for controlling the network width. A naive construction that directly assigns one unit (or a small block) to each of the $K^{2d}$ refined cells typically requires width $O(K^{2d})$, leading to a much larger parameter count and increased model complexity. In contrast, the multiscale strategy keeps the width at $O(K^d)$, which is necessary to obtain the optimal approximation rate without imposing intractable $\ell^0$-sparsity constraints. Similar ideas also appear in \citet{kohler2021rate}.
\end{remark}

\begin{figure}[!htbp]
    \centering
    \includegraphics[width=0.79\linewidth]{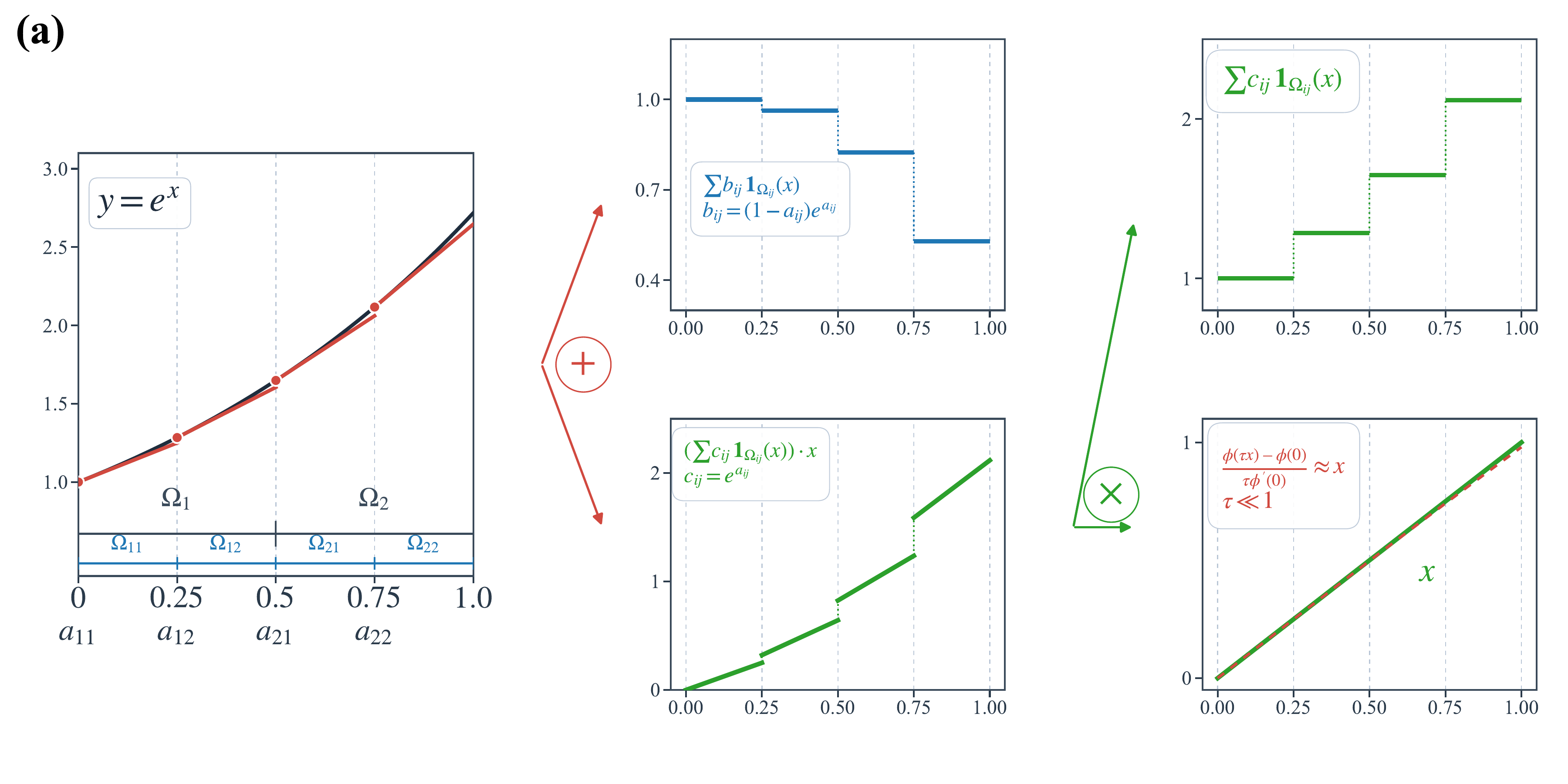} \\  

    \includegraphics[width=0.79\linewidth]{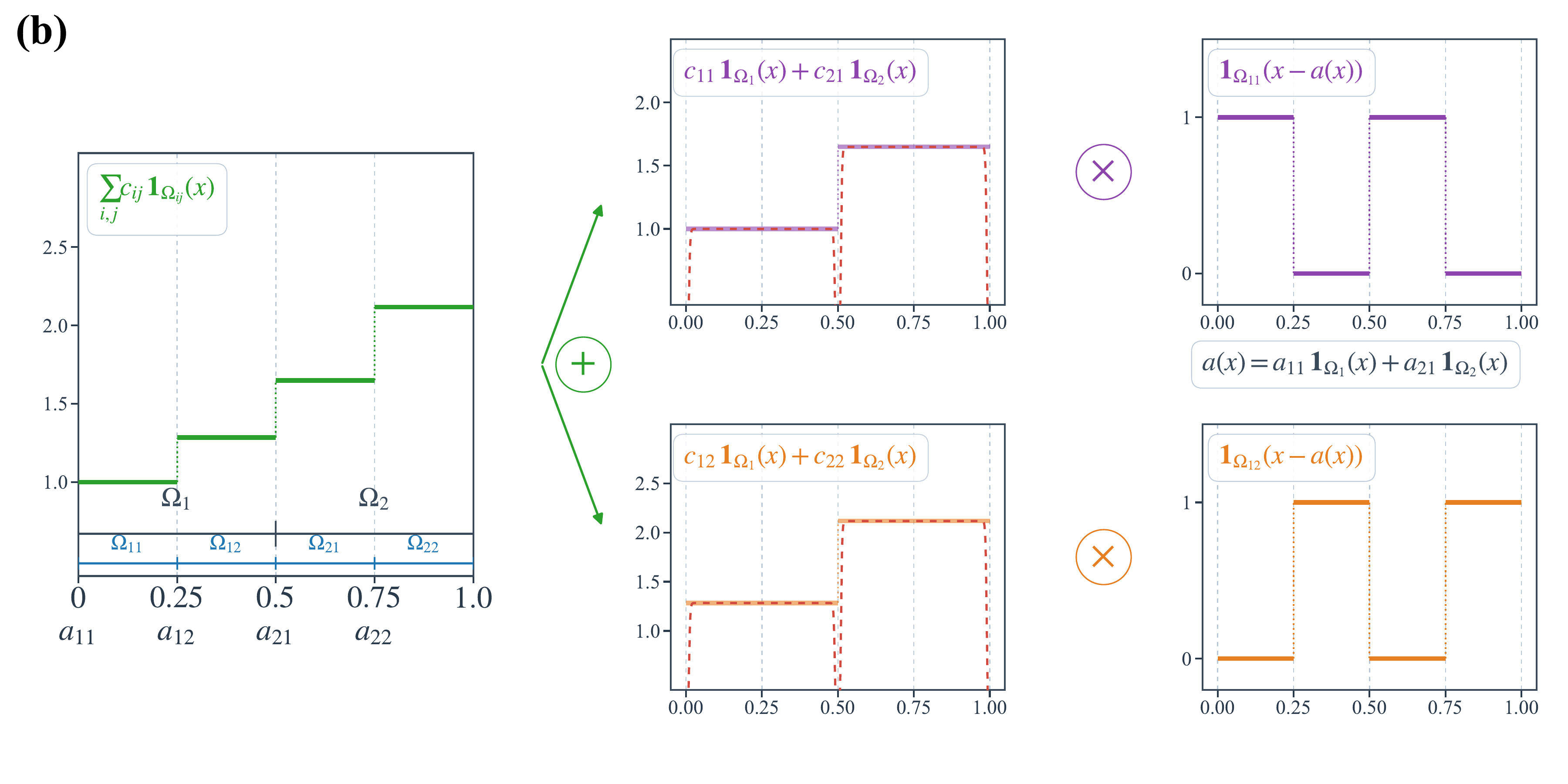} \\

    \includegraphics[width=0.79\linewidth]{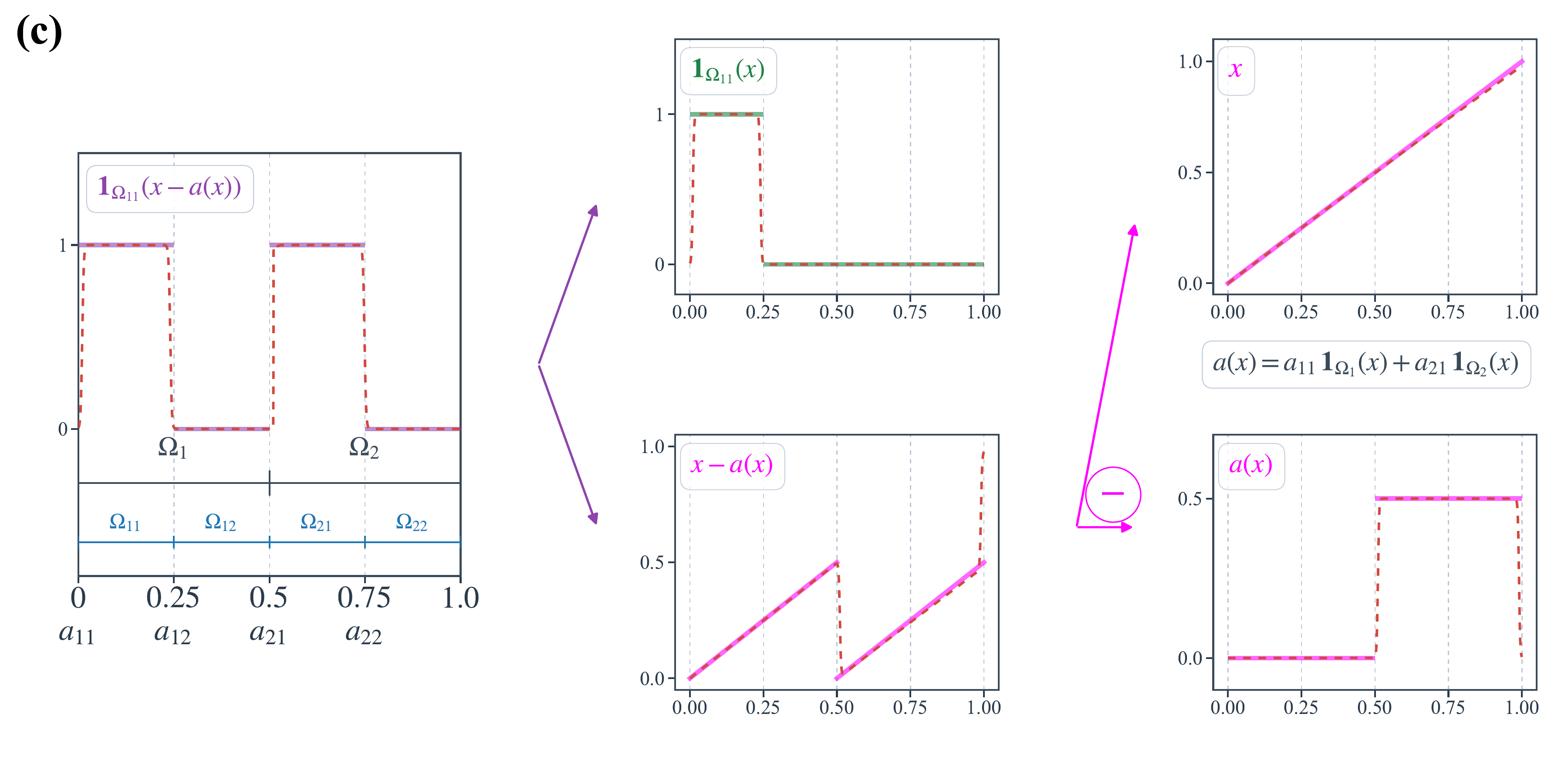}
    
    \caption{Illustration of the approximator construction for $f^\star$ in Theorem~\ref{thm:app_infinite_local} with $d=1$ and   
    $K = 2$. (a) Approximate $f^\star$ by piecewise polynomials, realized as the product of global polynomials and  piecewise constant functions. (b) The $4$-piece piecewise constant function on refined cells is decomposed into a summation of two $2$-piece functions defined over coarse cells, multiplied by refined-cell indicator functions.
    (c) The refined-cell indicator functions are realized by taking the extracted relative position information $x - a(x)$ as the input for the  reference indicators $\mathbbm{1}_{[0, 0.25]}$ and $\mathbbm{1}_{[0.25, 0.5]}$, which correspond to the refined cells contained within the leftmost coarse cell $[0, 0.5]$. }
    \label{fig:app-L2}
\end{figure}

\subsection{Proof Sketch of Theorem~\ref{thm:app_Linfty}}
\label{sec:proof_sketch_of_theorem_ref_thm_app_linfty}

Firstly, employing proof techniques analogous to those in Theorem~\ref{thm:app_L2}, we construct a family of $2^d$ neural networks $\{{g}_{\bv}\}_{\bv \in [2]^d}$ that achieve uniform approximation accuracy $\epsilon$ on the shifted interior regions $\Omega_{\mathrm{int}}^{K, \bv}(\delta)$, while maintaining bounded outputs on the complement (the shifted band regions $\Omega_{\mathrm{band}}^{K,\bv}(\delta)$). Simultaneously, we introduce a set of bounded weight functions $\{w_{\bv}\}_{\bv \in [2]^d}$ satisfying the partition of unity property, $\sum_{\bv \in [2]^d} w_{\bv}(\bx) = 1$ for $\bx \in [0,1]^d$, which vanish locally on their associated band regions (i.e., $w_{\bv}(\bx) = 0$ for $\bx \in \Omega_{\mathrm{band}}^{K, \bv}(\delta)$). The target function $f^\star$ is then decomposed as $f^\star = \sum_{\bv \in [2]^d} f^\star w_{\bv}$ and approximated by the weighted combination $\sum_{\bv \in [2]^d} g_{\bv} w_{\bv}$. The global error is controlled by this decomposition: for any fixed $\bv$, the approximation error is small on $\Omega_{\mathrm{int}}^{K, \bv}(\delta)$ due to the accuracy of $g_{\bv}$ and the boundedness of $w_{\bv}$, while on $\Omega_{\mathrm{band}}^{K, \bv}(\delta)$, the potentially large approximation error of $g_{\bv}$ is strictly suppressed because $w_{\bv}$ vanishes. The construction is completed by adding an additional layer to approximate the product $g_{\bv} w_{\bv}$. An illustration for the $d=1$ case is provided in Figure~\ref{fig:app-Linfty}, and the detailed proof is presented in Appendix~\ref{subapp:L-infty}.

\begin{figure}[htbp]
    \centering
    \includegraphics[width=0.95\linewidth]{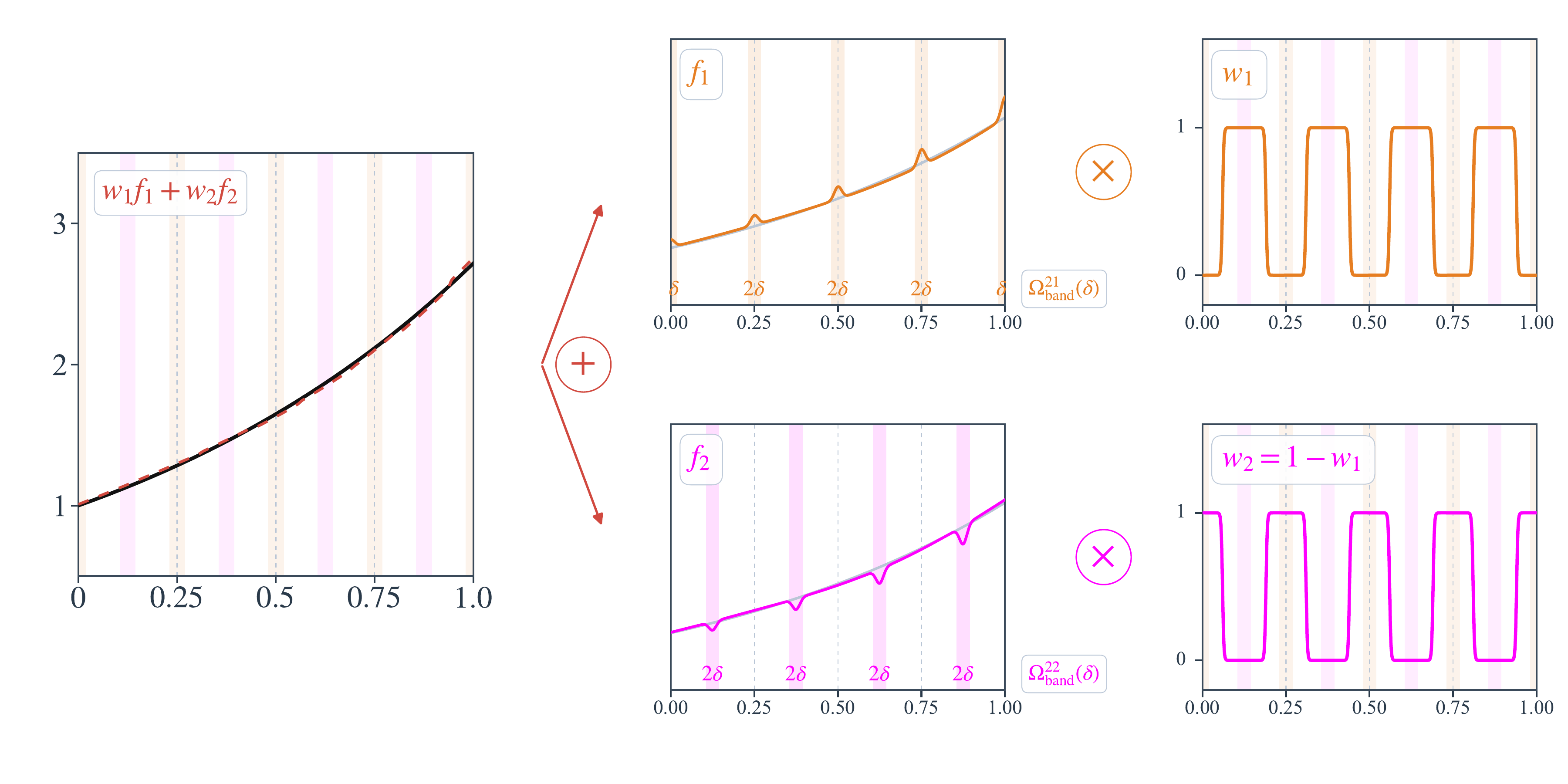}
    \caption{Illustration of the $L^{\infty}([0,1])$ approximation strategy for $f^\star$ detailed in Theorem~\ref{thm:app_Linfty}. Large approximation errors of $f_i$ on the bands $\Omega_{\mathrm{band}}^{2,i}(\delta)$ are nullified by the vanishing weight functions $w_i(x)$. Since the weights constitute a partition of unity, the global reconstruction $w_1f_1 + w_2f_2$ maintains the desired approximation accuracy across the entire domain $[0,1]$. }
    \label{fig:app-Linfty}
\end{figure}

\subsection{Proof Sketch of Theorem~\ref{thm:est}}
\label{sec:proof_sketch_of_theorem_ref_thm_est}

To achieve the approximation accuracy using the neural network approximation in Theorem~\ref{thm:app_Linfty}, the complexity of the model class with input $\bx \in [0,1]^d$, which contains the solution, can be characterized by the following logarithmic covering number 
\begin{equation}
    \log \mathcal{N} \left(\tau, \mathcal{H}^{\phi,7}(d, 1, M_{\epsilon}, B_{\epsilon}), \left\|\cdot\right\|_\infty \right) \lesssim \left( \frac{1}{\epsilon}\right)^{\frac{d}{s}} \left(\log \frac{1}{\tau} + \log \frac{1}{\epsilon} \right).
    \label{eq:est_sketch_log_covering}
\end{equation}
Next, applying Lemma~\ref{lem:aos_lem}, we derive the following bound on the generalization error: 
\begin{equation*}
    \mathbb{E} \left[ \left\| \mathbb{T}_{F} \widehat{f}_n - f^\star\right\|_{L^2(\rho)}^2  
    \right] \lesssim  \epsilon^2 + \frac{1}{n} \left( \frac{1}{\epsilon}\right)^{\frac{d}{s}}  \left( \log  \frac{1}{\tau} + \log \frac{1}{\epsilon}\right)+ \tau.
\end{equation*}
By balancing the tradeoff between approximation accuracy and model complexity, we choose $\epsilon \eqsim n^{-\frac{s}{2s+d}}, \tau \eqsim n^{-\frac{2s}{2s+d}}$, which leads to the following generalization error bound: 
\begin{equation*}
     \mathbb{E} \left[ \left\| \mathbb{T}_{F} \widehat{f}_n - f^\star\right\|_{L^2(\rho)}^2  
    \right] \lesssim n^{-\frac{2s}{2s+d}} \log n.
\end{equation*}
The detailed proof is provided in Appendix~\ref{app:est}.

\section{Conclusions}

We have developed a unified, constructive analysis of both approximation and generalization for  neural networks equipped with smooth activation functions over the Sobolev space $W^{s,\infty}([0,1]^d)$. 
We constructed explicit neural network approximators whose parameter norms are carefully controlled. 
These networks attain the minimax-optimal approximation rate for arbitrary smoothness $s>0$, thereby demonstrating \emph{smoothness adaptivity} at fixed depth. 
The norm-controlled construction enables a sharp statistical analysis, showing that empirical risk minimization over this model class achieves the minimax-optimal estimation rate (up to logarithmic factors).
Moreover, we established approximation lower bounds for ReLU networks, showing that their smoothness adaptivity is fundamentally limited by depth.
Taken together, these results reveal that depth is not the sole mechanism for achieving smoothness adaptivity; activation smoothness provides an alternative route.

Looking ahead, several important directions remain open. 
First, while our results establish statistical optimality in the noisy  regression setting, the  learning behavior  in the noiseless regime remains unclear~\citep{chen2025duality}.
Second, in scientific computing applications such as PDE solvers, performance is often evaluated under stronger norms, such as  Sobolev norms. 
The approximation and estimation rates of neural networks with smooth activations under these stronger norms, as well as their potential optimality in this regime, remain largely unexplored. 
Clarifying these questions would further illuminate the role of activation smoothness in high-accuracy numerical and scientific learning tasks.

\section*{Acknowledgments}
This work was supported by the National Key R\&D Program of China (No.~2022YFA1008200) and the National Natural Science Foundation of China (NSFC 12522120). The authors thank Juncai He, Juno Kim, and Zikai Shen for helpful discussions.

\bibliography{ref.bib}
\bibliographystyle{bibstyle}

\newpage

\appendix
\part{Appendix}
\parttoc

\section{Technical Preliminaries}

\subsection{Additional Notations}
\begin{itemize}
    \item Let $\mathbb{Z}$ denote the set of integers. Let $\mathbb{N} \coloneqq  \{0, 1, 2, \dots\}$ denote the set of natural numbers. Accordingly, $\mathbb{N}^d$ denotes the set of $d$-dimensional multi-indices $\bi=(i_1, \dots, i_d)$ where each component $i_\ell \in \mathbb{N}$ for $\ell=1, \dots,d$.
    We denote $\mathbb{N}_+ \coloneqq \mathbb{N} \setminus \{0\}$ as the set of positive integers.
    \item For any integer \( K \geq 1 \), we denote the set \( [K] \coloneqq  \{1, 2, \dots, K\} \) and the set \( [\tilde{K}] \coloneqq  \{0, 1, \dots, K\} \). Accordingly, \( [K]^d \) (and similarly \( [\tilde{K}]^d \)) denotes the set of \( d \)-dimensional multi-indices \( \bi = (i_1, \dots, i_d) \), where each component \( i_l \in [K] \) (or \( i_l \in [\tilde{K}] \)) for \( l = 1, \dots, d \).

    \item For non-negative functions $f$ and $g$, we write $f(x) \lesssim g(x)$ or $f(x) = O(g(x))$ to indicate that there exists a constant $C$ relying only on the dimension $d$, smoothness $s$, and activation function $\phi$ such that $f(x) \leq C g(x)$. We use $f(x) = \tilde{O}(g(x))$ to suppress polylogarithmic factors, i.e. $f(x) = O(g(x)\polylog(x))$. Also, we write $f(x)\eqsim g(x)$ if both $f(x)\lesssim g(x)$ and $g(x)\lesssim f(x)$ hold.
    \item For a given neural network $g$, let $\theta(g)$ denote the parameter vector comprising all its weight matrices and bias vectors.  We denote the maximum parameter magnitude by the infinity norm $\|\theta(g)\|_\infty$
    
    \item For a matrix $\bA = (a_{ij})_{i \in[m], j\in [n]}\in \RR^{m \times n}$, the norm $\|\cdot\|_{\infty,\infty}$ is given by $\|\bA\|_{\infty,\infty} = \max_{i \in [m],j \in [n]}  |a_{ij}|$ and the norm $\|\cdot\|_{1, \infty}$ is given by $\|\bA\|_{1, \infty} = \max_{i \in [m]}  \sum_{j=1}^{n} |a_{ij}|$.
    \item A $d$-dimensional multi-index is a tuple $\balpha = (\alpha_1, \dots, \alpha_d) \in \mathbb{N}^d$. Several related notations are listed below:
    \begin{itemize}
        \item $|\balpha| \coloneqq  \sum_{i=1}^d \alpha_i$;
        \item $\balpha! \coloneqq  \prod_{i=1}^d \alpha_i!$;
        \item $\bx^{\balpha} \coloneqq  x_1^{\alpha_1} \cdots x_d^{\alpha_d}$, where $\bx = (x_1, \dots, x_d) \in \mathbb{R}^d$;
        \item $D^{\balpha} \coloneqq  \frac{\partial^{|\balpha|}}{\partial x_1^{\alpha_1} \cdots \partial x_d^{\alpha_d}}$;
        \item $\binom{k}{\balpha} \coloneqq  \frac{k!}{\balpha!}$, where $k = |\balpha|$.
    \end{itemize}
    \item For any $x \in \mathbb{R}$, let $\lfloor x \rfloor \coloneqq  \max\{n : n \le x, n \in \mathbb{Z}\}$ and $\lceil x \rceil \coloneqq  \min\{n : n \ge x, n \in \mathbb{Z}\}$.
    \item Let $\mathbbm{1}_{\Omega}(\bx)$ denote the indicator function of the region $\Omega$, i.e., $\mathbbm{1}_{\Omega}(\bx)=1$ if $\bx \in \Omega$ and $0$ otherwise.
\end{itemize}

\subsection{Verification of Smooth Activation Assumptions}
\label{app:verification}
In this section, we formally verify that widely adopted smooth activation functions---including sigmoid, $\frac{1}{2}(\tanh + 1)$, SiLU, and GELU---satisfy Assumptions~\ref{ass:smooth},~\ref{ass:Lip} and~\ref{ass:piece}. The condition of being infinitely differentiable and non-polynomial, as stipulated in Assumption~\ref{ass:smooth}, is trivially satisfied for these aforementioned activations. Regarding Assumption~\ref{ass:Lip}, observe that for all considered functions, the first derivative $\phi^{\prime}$ satisfies the uniform bound $\sup_{t \in \mathbb{R}} |\phi^{\prime}(t)| < 2$. Consequently, by the mean value theorem, these activations possess the Lipschitz continuity property required by Assumption~\ref{ass:Lip} with $\|\phi\|_{\mathrm{Lip}} = 2$. The remainder of this section focuses on verifying the approximation properties outlined in Assumption~\ref{ass:piece}.


\paragraph{Heaviside-like case.}

We  verify that the sigmoid and $\frac{1}{2}(\tanh + 1)$ activation functions satisfy the Heaviside-like condition.
\begin{itemize}
    \item \textbf{Sigmoid function:} Consider the sigmoid activation $\sigma(t) = (1+e^{-t})^{-1}$. We observe
    \begin{itemize}
    \item For $t > 0$, 
    \begin{equation*}
        |\sigma(t) - H(t)| = \left| \frac{1}{1+e^{-t}} - 1 \right| = \frac{e^{-t}}{1+e^{-t}} < e^{-t}.
    \end{equation*}
    \item For $t < 0$, 
    \begin{equation*}
        |\sigma(t) - H(t)| = \left| \frac{1}{1+e^{-t}} - 0 \right| = \frac{1}{1+e^{-t}} = \frac{e^t}{1+e^t} < e^{-|t|}.
    \end{equation*}
\end{itemize}
    Note that $e^{-|t|} \leq  \min \{\frac{1}{e |t|},1\}$, the assumption holds with $C_1 = 1$.

    \item \textbf{Tanh-based function:} Consider the activation function $\phi(\cdot) = \frac{1}{2}(\tanh (\cdot)+ 1)$. We analyze its approximation to $H(\cdot)$ as follows:
    \begin{itemize}
    \item For $t > 0$,
    \begin{equation*}
        |\phi(t) - H(t)| = \left| \frac{1}{2}(\tanh(t) + 1) - 1 \right| = \frac{1}{2}(1 - \tanh(t)) = \frac{e^{-t}}{e^t + e^{-t}} < e^{-2t}.
    \end{equation*}
    \item For $t < 0$,
    \begin{equation*}
        |\phi(t) - H(t)| = \left| \frac{1}{2}(\tanh(t) + 1) - 0 \right| = \frac{1}{2}(\tanh(t) + 1) = \frac{e^t}{e^t + e^{-t}} < e^{2t} = e^{-2|t|}.
    \end{equation*}
\end{itemize}
    Note that $e^{-2|t|} \leq  \min \{\frac{1}{2e |t|},1\}$, the assumption holds with $C_1 = 1$.
\end{itemize}

\paragraph{ReLU-like case.}
We verify that the GELU and SiLU activation functions satisfy the ReLU-like condition.
\begin{itemize}
    \item \textbf{SiLU:} Consider the SiLU activation $\phi(t) = t\sigma(t)$, where $\sigma$ is the sigmoid function. The deviation is as follows:
    \begin{itemize}
    \item For $t \geq 0$, 
    \begin{equation*}
        |\phi(t) - \text{ReLU}(t)| = |t\sigma(t) - t| = t(1 - \sigma(t)) = \frac{t e^{-t}}{1+e^{-t}} = \frac{t}{1+e^t}.
    \end{equation*}
    \item For $t < 0$,
    \begin{equation*}
        |\phi(t) - \text{ReLU}(t)|= |t\sigma(t) - 0| = |t|\sigma(t) = \frac{|t|}{1+e^{-t}} = \frac{|t|}{1+e^{|t|}}.
    \end{equation*}
\end{itemize}
    Note that $f(u) = \frac{u}{1+e^u}\le 1$ for $u \geq 0$, the assumption holds with $C_2 = 1$.

    \item \textbf{GELU:} Consider the GELU activation $\phi(t) = t \Phi(t)$, where $\Phi$ is the cumulative distribution function of the standard normal distribution defined as
    \begin{equation*}
        \Phi(t) = \frac{1}{\sqrt{2\pi}} \int_{-\infty}^t e^{-x^2/2} dx.
    \end{equation*}
    The error satisfies
    \begin{itemize}
    \item For $t \geq 0$, 
    \begin{equation*}
        |\phi(t) - \text{ReLU}(t)| = |t\Phi(t) - t| = t(1 - \Phi(t)).
    \end{equation*}
    \item For $t < 0$,
    \begin{equation*}
        |\phi(t) - \text{ReLU}(t)| = |t\Phi(t) - 0| = |t|\Phi(t).
    \end{equation*}
\end{itemize}
    By symmetry, we have $\Phi(-|t|) = 1 - \Phi(|t|)$; thus, $|G(t) - \text{ReLU}(t)| = |t|(1-\Phi(|t|))$ for all $t \in \mathbb{R}$. 
    We estimate the tail integral as follows:
    \begin{equation*}
        1 - \Phi(t) = \frac{1}{\sqrt{2\pi}} \int_t^\infty e^{-x^2/2} \mathrm{d}x \leq \frac{1}{\sqrt{2\pi}} \int_t^\infty \frac{x}{t} e^{-x^2/2} \mathrm{d}x
        \leq \frac{1}{t\sqrt{2\pi}} \left[ -e^{-x^2/2} \right]_t^\infty = \frac{e^{-t^2/2}}{t\sqrt{2\pi}}.
    \end{equation*}
    Multiplying by $t$, we obtain $|\phi(t) - \text{ReLU}(t)| \leq \frac{1}{\sqrt{2\pi}} e^{-t^2/2} \leq \frac{1}{\sqrt{2\pi}}$, the assumption holds with $C_2 = (2\pi)^{-1/2}$.
\end{itemize}

\section{Approximation Theory: Proofs and Technical Details}
\label{app:approximation}
In this section, we present the construction for approximators for $f^{\star}$ given in Theorem~\ref{thm:app_L2} and ~\ref{thm:app_Linfty}.

\subsection{Domain Partition Construction}
Our analysis require the following partition of the input domain.
\begin{itemize}
    \item Let $K \in \mathbb{N}_{+}$. For any multi-index $\bi = (i_1, \cdots, i_d) \in [K]^{d}$, the $\bi$-th coarse cell of the uniform $K^d$ grid on the hypercube $[0,1]^d$ is the axis-aligned cell
     \begin{equation*}
        \Omega^{K}_{\bi} \coloneqq  \prod_{l=1}^d \left[ \frac{i_l-1}{K}, \frac{i_l}{K} \right),
    \end{equation*}
    whose lower-left corner is
    \begin{equation*}
        \ba_{\bi}^{K} \coloneqq  \left( \frac{i_1 -1}{K}, \cdots, \frac{i_d -1}{K}\right) \in \mathbb{R}^d.
    \end{equation*}
    For any $\bj = (j_1 , \cdots, j_d) \in [K]^d$, the $\bj$-th refined cell of the uniform $K^d$ subgrid of $\Omega_{\bi}^{K}$ is defined by 
    \begin{equation*}
        \Omega_{\bi, \bj}^{K} \coloneqq  \prod_{l=1}^d \left[ \frac{(i_l-1)K + j_l -1 }{K^2}, \frac{ (i_l-1)K + j_l }{K^2} \right).
    \end{equation*}
    Figure~\ref{fig:notation} depicts the spatial configuration of the cells $\Omega_{\bi}^K, \Omega_{\bi,\bj}^K$ and the corner point $\ba_{\bi}^K$.
    

    \item For $K \in \mathbb{N}_{+}$ and $\delta \in (0, \frac{1}{3K})$, we define the interior region associated with the coarse cell $\Omega_{\bi}^{K}$ by
    \begin{equation*}
        \Omega_{\bi,\text{int}}^{K}(\delta) \coloneqq  \prod_{l=1}^d \left( \frac{i_l-1}{K} + \delta , \frac{ i_l}{K} - \delta\right).
    \end{equation*}
    The corresponding band region is given by 
    \begin{equation*}
        \Omega_{\bi,\text{band}}^{K}(\delta) \coloneqq  \Omega_{\bi}^{K}  \setminus \Omega_{\bi,\text{int}}^{K}(\delta).
    \end{equation*}
    Similarly, for $\delta \in (0, \tfrac{1}{3K^{2}})$, we define the interior region associated with the refined cell $\Omega_{\bi,\bj}^{K}$ as
    \begin{equation*}
        \Omega_{\bi, \bj, \mathrm{int}}^{ K}(\delta) \coloneqq   \prod_{l=1}^d \left( \frac{(i_l-1)K + j_l -1 }{K^2} + \delta, \frac{(i_l-1)K + j_l }{K^2} - \delta \right),
    \end{equation*}
    and the associated band region is
    \begin{equation*}
        \Omega_{\bi, \bj , \mathrm{band}}^{ K}(\delta) \coloneqq  \Omega_{\bi, \bj}^{ K} \setminus \Omega_{\bi, \bj, \mathrm{int}}^{ K}(\delta) .
    \end{equation*}
    Figure~\ref{fig:notation} depicts the spatial configuration of the interior and band regions.

    \begin{figure}[htbp]
    \centering
    \includegraphics[width=0.95\textwidth]{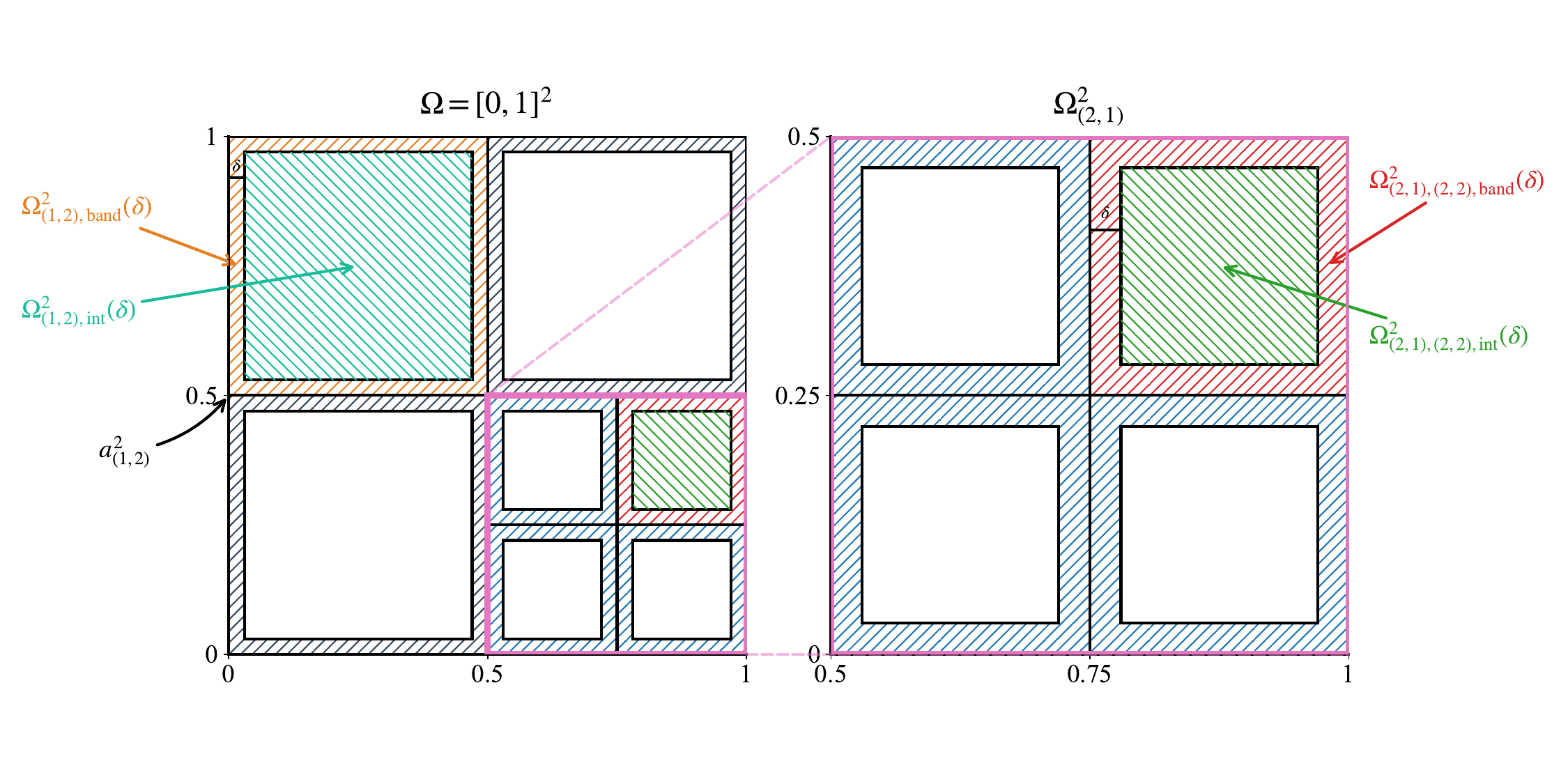}
    \caption{Visualization of the hierarchical grid structure ($K=2, d=2$), detailing the coarse and refined cells along with their respective interior and band regions.}
    \label{fig:notation}
    \end{figure}
    \item For $ K \in \mathbb{N}_{+}$, $\bv = (v_1, \cdots, v_d) \in [2]^d$ and $\bi , \bj \in [\tilde{K}]^d $, we define the shifted refined cell:
    \begin{equation*}
        \Omega_{\bi, \bj}^{ K, \bv} \coloneqq  \left(\prod_{l=1}^d \left[ \frac{\left[2(i_l-1)K + 2j_l +v_l -1 \right]}{2K^2},  \frac{\left[2(i_l-1)K + 2j_l +v_l +1 \right]}{2K^2}\right)\right) \bigcap[0,1]^d,
    \end{equation*}
    where $\bi, \bj \in \{0, 1, \cdots, K\}^d$.  For \( \delta \in \left( 0, \frac{1}{6K^2} \right) \), we define the associated interior region as:
    \begin{equation*}
         \Omega_{\bi, \bj, \mathrm{int}}^{ K, \bv}(\delta) \coloneqq  \left(\prod_{l=1}^d \left( \frac{\left[2(i_l-1)K + 2j_l +v_l -1 \right]}{2K^2} + \delta ,  \frac{\left[2(i_l-1)K + 2j_l +v_l +1 \right]}{2K^2} - \delta\right)\right) \bigcap[0,1]^d,
    \end{equation*}
    and the associated band region is:
    \begin{equation*}
        \Omega_{\bi, \bj, \mathrm{band}}^{ K, \bv}(\delta): =  \Omega_{\bi, \bj}^{ K, \bv} \setminus  \Omega_{\bi, \bj, \mathrm{int}}^{ K, \bv}(\delta).
    \end{equation*}
    Finally, we define the overall shift interior and band regions as:
    \begin{equation*}
        \Omega_{\mathrm{int}}^{K, \bv}(\delta): = \bigcup_{\bi, \bj \in [\tilde{K}]^d}  \Omega_{\bi, \bj, \mathrm{int}}^{ K, \bv}(\delta), \quad \Omega_{\mathrm{band}}^{K, \bv}(\delta): = \bigcup_{\bi, \bj \in [\tilde{K}]^d}  \Omega_{\bi, \bj, \mathrm{band}}^{ K, \bv}(\delta).
    \end{equation*}
\end{itemize}

\subsection{Bramble-Hilbert Lemma}
In this subsection, we approximate the target function $f^{\star}$ by a piecewise polynomial.
A closely related conclusion can be found in \citet{brenner2008mathematical}; for the self-contained presentation, we restate the argument and provide a proof.

\begin{lemma}[Bramble-Hilbert lemma]
    For $s>0$, let $f^{\star} \in W^{s, \infty}(\Omega)$. Then, there exists a piecewise polynomial $p$ of order $\lceil s \rceil -1$ on the partition $\{\Omega^{K}_{\bi,\bj}\}_{\bi,\bj \in [K]^d}$ such that
    \begin{equation*}
        \|f^{\star} - p\|_{L^\infty(\Omega)} \leq c_1(s,d) \|f^{\star}\|_{W^{s,\infty}(\Omega)} K^{-2s}.
    \end{equation*}
    Specifically, the piecewise polynomial $p$ can be written as
    \begin{equation*}
        p(\bx) = \sum_{\bi \in [K]^d, \bj\in[K]^d} p_{\bi, \bj}(\bx) \mathbbm{1}_{\Omega_{\bi,\bj}^K}(\bx),
    \end{equation*}
    where each polynomial $p_{\bi,\bj}(\bx) = \sum_{|\balpha| < \lceil s \rceil} a_{\balpha,\bi,\bj} \bx^{\balpha}
    $ has order $\lceil s \rceil -1$ and the coefficients satisfy
    \begin{equation*}
        |a_{\balpha,\bi,\bj}| \leq c_2(s,d) \|f^{\star}\|_{W^{s,\infty}(\Omega)}, \quad \forall \bi,\bj \in [K]^d, |\balpha| < \lceil s \rceil.
    \end{equation*}
    Here, $c_1(s,d)$ and $c_2(s,d)$ are two constants that depend only on the smoothness $s$ and dimension $d$. 
    \label{lem:piece_poly_app}
\end{lemma}
\begin{proof}
    We construct a polynomial on each small cube $\Omega^{K}_{\bi,\bj}$ to approximate $f^{\star}$ locally. For ease of notation, we neglect the subscript $\bi,\bj$ when there is no confusion. Let $\tilde{\Omega}=\Omega_{\bi,\bj}^K$ be a cube with side length $h$ ($h = K^{-2}$ for partition $\{\Omega^{K}_{\bi,\bj}\}_{\bi,\bj \in [K]^d}$) and $f^{\star} \in W^{s,\infty}(\tilde{\Omega})$. We are going to construct a polynomial $p$ of order $\lceil s \rceil -1$ on $\tilde{\Omega}$ such that
    \begin{equation*}
        \|f^{\star} - p\|_{L^\infty(\tilde{\Omega})} \leq c_1(s,d) h^{s} \|f^{\star}\|_{W^{s,\infty}(\tilde{\Omega})},
    \end{equation*}
    and the coefficients of $p$ satisfy
    \begin{equation*}
        |a_{\balpha}| \leq c_2(s,d) \|f^{\star}\|_{W^{s,\infty}(\tilde{\Omega})}.
    \end{equation*}
     We then give the construction. Let $\psi \in C^\infty(\mathbb{R}^d)$ be a cut-off function satisfying the following conditions:
    \begin{itemize}
        \item $\psi$ is supported on $\tilde{\Omega}$;
        \item $\psi$ is non-negative, i.e., $\psi(\by) \geq 0$ for all $\by\in \mathbb{R}^d$;
        \item $\int_{\tilde{\Omega}} \psi(\by) \dd \by = 1$.
    \end{itemize}
    Then we define the averaged Taylor polynomial of order $\lceil s \rceil -1$ of $f^{\star}$ as
    \begin{equation*}
        \left(Q^{\lceil s \rceil -1} f^{\star} \right)(\cdot) = \int_{\tilde{\Omega}} T^{\lceil s \rceil -1}_{\by} f^{\star}(\cdot) \psi(\by) \dd \by,
    \end{equation*}
     where $T^{\lceil s \rceil -1}_{\by} f^{\star}$ is the Taylor polynomial of order $\lceil s \rceil -1$ of $f^{\star}$ at $\by$:
     \begin{equation*}
         \left(T_{\by}^{\lceil s \rceil -1} f^{\star} \right)(\cdot) = \sum_{|{\balpha}| < \lceil s \rceil} \frac{1}{{\balpha}!} D^{\balpha} f^{\star}(\by) (\cdot - \by)^{\balpha}.
     \end{equation*}
     We claim that $p= Q^{\lceil s \rceil -1} f^{\star}$ satisfies the desired approximation property and coefficient bound. In the following, we will treat the cases of integer $s$ and non-integer $s$ separately.

     \paragraph{Case 1: $s=m$ is an integer.} In this case, we have $\lceil s \rceil = m$. We note that the target function $f^{\star}$ admits the following Taylor expansion with integral remainder:
    \begin{equation*}
        f^{\star}(\bx) = \left(T^{m-1}_{\by} f^{\star} \right)(\bx) + \sum_{|{\balpha}| = m} \frac{1}{{\balpha}!} (\bx - \by)^{\balpha} \int_0^1 m t^{m-1} D^{\balpha} f^{\star}(\bx + t(\by - \bx)) \dd t.
    \end{equation*}
    The above Taylor formula is classical for $C^\infty$ functions. Since here $f^{\star} \in W^{m,\infty}(\tilde{\Omega})$ does not necessarily guarantee the pointwise existence of $D^{\balpha} f^{\star}$, it should be understood in the weak sense.
    By integrating the above equation against the cut-off function $\psi(\cdot)$ over $\tilde{\Omega}$, we obtain
    \begin{equation*}
        \begin{aligned}
            f^{\star}(\bx) -( Q^{m-1} f^{\star})(\bx) &= \int_{\tilde{\Omega}} \left( f^{\star}(\bx) - (T^{m-1}_{\by} f^{\star})(\bx) \right) \psi(\by) \dd \by \\
            &= \sum_{|{\balpha}| = m} \int_{\tilde{\Omega}} \frac{1}{{\balpha}!} (\bx - \by)^{\balpha} 
            \left(\int_0^1 mt^{m-1} D^{\balpha} f^{\star}(\bx + t(\by - \bx)) \dd t \right) \psi(\by) \dd \by\\
            &\leq \sum_{|{\balpha}| = m} \frac{1}{{\balpha}!} \sup_{\by} |\bx - \by|^{|{\balpha}|} \|f^{\star}\|_{W^{m,\infty}(\tilde{\Omega})} \int_{\tilde{\Omega}} |\psi(\by)| \dd \by \cdot \int_{0}^{1} m t^{m-1} \dd t\\
            &\leq c_1(m,d) h^{m} \|f^{\star}\|_{W^{m,\infty}(\tilde{\Omega})}.
        \end{aligned}
    \end{equation*}
    To bound the coefficients, we expand the polynomial $Q^{m-1} f^{\star}$ and rewrite it as
    \begin{equation*}
        \begin{aligned}
            (Q^{m-1} f^{\star})(\bx) &= \sum_{|{\balpha}| < m} \frac{1}{{\balpha}!} \left( \int_{\tilde{\Omega}} D^{\balpha} f^{\star}(\by) \psi(\by) \dd \by \right) (\bx - \by)^{\balpha} \\
            &= \sum_{|{\balpha}| < m} \sum_{{\bgamma}+{\bbeta}={\balpha}} \frac{1}{{\balpha}!} \binom{{\balpha}}{{\bgamma}} \bx^{\bgamma} (-1)^{|{\bbeta}|} \left( \int_{\tilde{\Omega}} \by^{\bbeta} D^{\balpha} f^{\star}(\by) \psi(\by) \dd \by \right)\\
            &= \sum_{|{\bgamma}| < m} \bx^{\bgamma} \underbrace{\left( \sum_{|{\balpha}| < m, {\balpha} \geq {\bgamma}} \frac{(-1)^{|{\balpha} - {\bgamma}|}}{({\balpha} - {\bgamma})! {\bgamma}!} \int_{\tilde{\Omega}} \by^{{\balpha} - {\bgamma}} D^{\balpha} f^{\star}(\by) \psi(\by) \dd \by \right)}_{\eqqcolon a_{\bgamma}}.
        \end{aligned}
    \end{equation*}
    The coefficients $a_{\bgamma}$ can be controlled as follows:
    \begin{equation*}
        \begin{aligned}
            |a_{\bgamma}| &\leq \sum_{|{\balpha}| < m, {\balpha} \geq {\bgamma}} \frac{1}{({\balpha} - {\bgamma})! {\bgamma}!} \int_{\tilde{\Omega}} |\by^{{\balpha} - {\bgamma}}| |D^{\balpha} f^{\star}(\by)| |\psi(\by)| \dd \by \\
            &\leq \sum_{|{\balpha}| < m, {\balpha} \geq {\bgamma}} \frac{1}{({\balpha} - {\bgamma})! {\bgamma}!}  \|D^{\balpha} f^{\star}\|_{L^\infty(\tilde{\Omega})} \int_{\tilde{\Omega}} |\psi(\by)| \dd \by \\
            &\leq c_2{(m,d)} \|f^{\star}\|_{W^{m,\infty}(\tilde{\Omega})}.
        \end{aligned}
    \end{equation*}   
    \paragraph{Case 2: $s=m+\zeta$ is a non-integer with $m=\lfloor s \rfloor$ and $\zeta \in (0,1)$.} In this case, we have $\lceil s \rceil = m+1$. Again, we expand $f^{\star}$ to order $m-1$:
    \begin{equation*}
         f^{\star}(\bx) = (T^{m-1}_{\by} f^{\star})(\bx) + \sum_{|{\balpha}| = m} \frac{1}{{\balpha}!} (\bx - \by)^{\balpha} \int_0^1 m t^{m-1} D^{\balpha} f^{\star}(\bx + t(\by - \bx)) \dd t.
    \end{equation*}
    By integrating against $\psi$, we have
    \begin{equation*}
        \begin{aligned}
            &f^{\star}(\bx) - (Q^{m} f^{\star})(\bx)\\
            =& \int_{\tilde{\Omega}} \left( f^{\star}(\bx) - (T^{m}_{\by} f^{\star})(\bx) \right) \psi(\by) \dd \by \\
            =& \sum_{|{\balpha}| = m} \int_{\tilde{\Omega}} \frac{1}{{\balpha}!} (\bx - \by)^{\balpha} 
            \left(\int_0^1 m t^{m-1} D^{\balpha} f^{\star}(\bx + t(\by - \bx)) \dd t \right) \psi(\by) \dd \by\\
            &- \sum_{|{\balpha}| = m} \int_{\tilde{\Omega}} \frac{1}{{\balpha}!} (\bx - \by)^{\balpha} D^{\balpha} f^{\star}(\by) \psi(\by) \dd \by \\
            =& \sum_{|{\balpha}| = m} \int_{\tilde{\Omega}} \frac{1}{{\balpha}!} (\bx - \by)^{\balpha} 
            \left(\int_0^1 m t^{m-1} \left(D^{\balpha} f^{\star}(\bx + t(\by - \bx)) - D^{\balpha} f^{\star}(\by)\right) \dd t \right) \psi(\by) \dd \by\\
            \leq& \sum_{|{\balpha}| = m} \int_{\tilde{\Omega}} \frac{1}{{\balpha}!} |\bx - \by|^{|{\balpha}|}
            \left(\int_0^1 m t^{m-1} \|f^{\star}\|_{W^{s,\infty}(\tilde{\Omega})} |\bx + t(\by - \bx) - \by|^{\zeta} \dd t \right) |\psi(\by)| \dd \by\\
            \leq& \sum_{|{\balpha}| = m} \frac{1}{{\balpha}!} \left(\sup_{\by} |\bx-\by|^{m+\zeta}\right) \|f^{\star}\|_{W^{s,\infty}(\tilde{\Omega})} \int_{\tilde{\Omega}} |\psi(\by)| \dd \by \cdot \int_{0}^{1} m t^{m-1} (1 - t)^{\zeta} \dd t\\
            \leq&\; c_1{(s,d)} h^{s} \|f^{\star}\|_{W^{s,\infty}(\tilde{\Omega})}.
        \end{aligned}
    \end{equation*}
    The bound for the coefficients can be obtained in the same way as in Case 1. In this way, we complete the proof of the lemma.
\end{proof}
Next, to approximate piecewise polynomials, we construct approximations for monomials, piecewise constant functions, and their products in the subsequent subsections.


\subsection{Approximation of Monomials}
\label{subsec:app_monomial}
In this subsection, we use a shallow neural network to approximate monomials by implementing a central difference scheme. Related constructions appear in \citet{pinkus1999approximation}. We restate the argument here to clarify that the method admits explicit norm control and to keep the presentation self-contained.
\begin{lemma}
    Let $m \in \mathbb{N}_{+}$ and $k \in \mathbb{N}$. For $\bq = (q_1, \cdots, q_m)$, define
    \begin{equation*}
        B_m(\bq;k) \coloneqq  \frac{1}{2^m} \sum_{\bnu \in \{ \pm1\}^m} \left( \prod_{i=1}^m \nu_i\right) \left[ S_{\bnu}(\bq) \right]^k,
    \end{equation*}
    where $S_{\bnu}(\bq)$ is defined by 
    \begin{equation*}
        S_{\bnu}(\bq)= \sum_{i=1}^m \nu_i q_i.
    \end{equation*}
    Then the following holds: 
    \begin{itemize}
        \item If $k<m$,  we have $B_m(\bq, k) = 0$.
        \item If $k=m$, we have $B_m(\bq, m) = m ! \prod_{i=1}^m q_i$.
    \end{itemize}
    \label{lemma:monomials}
\end{lemma}
\begin{proof}
    By the multinomial theorem, we expand $[S_{\bnu}(\bq)]^k$ as follows:
    \begin{equation*}
        [S_{\bnu}(\bq)]^k = \sum_{|\balpha| = k} \binom{k}{ \balpha} \left( \prod_{i=1}^m \nu_i^{\alpha_i}\right) \bq^{\balpha},
    \end{equation*}
     Substituting this into the definition of $B_m$, we obtain
    \begin{equation*}
        B_m(\bq; k) = \sum_{|\balpha | = k} \left( \frac{1}{2^m} \binom{k}{\balpha} \sum_{\bnu \in \{ \pm1\}^m} \left( \prod_{i=1}^m \nu_i^{1 + \alpha_i} \right)\right) \bq^{\balpha}.
    \end{equation*}
     The inner sum over $\bnu$ vanishes unless each exponent $1+\alpha_i$ is even, i.e. unless $\alpha_i$ is odd for all $i=1,\dots,m$. Hence,
    \begin{equation*}
        B_m(\bq; k) = \sum_{\substack{|\alpha|=k \\ \alpha_{i} \text { all odd }}} \binom{k}{\balpha} \bq^{\balpha}.
    \end{equation*}
    If $k<m$, such a multi-index $\balpha$ cannot exist: requiring all $\alpha_i \geq 1$ forces $|\balpha| \geq m > k$. Thus $B_m(\bq;k)=0$ in this case. If $k=m$, the only admissible multi-index is $\balpha=(1,1,\dots,1)$. Therefore,
    \begin{equation*}
        B_m(\bq;m) 
        = m! \, q_1 q_2 \cdots q_m,
    \end{equation*}
    which completes the proof.
\end{proof}
Using Lemma~\ref{lemma:monomials}, we can employ the central difference scheme to approximate $q_1 q_2 \cdots q_m$. 
\begin{lemma}\label{lemma:Tm}
    Let $m \in \mathbb{N}_{+}$, 
    and suppose $\phi \in C^{m+1}(\mathbb{R})$ and $x_0 \in \mathbb{R}$
    satisfy $\phi^{(m)}(x_0) \neq 0$.
    For any vector $\bq = (q_1, \cdots, q_m) \in \mathbb{R}^m$ and  step size $0<h<1$, define the function $T_{m}^{(x_0)}$ as
    \begin{equation*}
        T_{m}^{(x_0)}(\bq, h) = \frac{1}{2^m h^m \phi^{(m)}(x_0)} \sum_{\bnu \in \{ \pm 1\}^m} \left( \prod_{i=1}^m \nu_i \right) \phi \left(x_0 + h \sum_{i=1}^m \nu_i q_i\right).
    \end{equation*}
    Then,
    \begin{equation*}
        \left|T_{m}^{(x_0)}(\bq, h) - \prod_{i=1}^m q_i \right| \leq \frac{h A^{m+1}}{ (m+1)!|\phi^{(m)}(x_0)|} \sup_{|t-x_0| \leq A} \left| \phi^{(m+1)}(t)\right|,
    \end{equation*}
    where $A\coloneqq \sum_{i=1}^m |q_i|$. 
\end{lemma}
\begin{proof}
    By Taylor’s theorem with the integral remainder, for any $t \in \mathbb{R}$,
    \begin{equation*}
        \phi(x_0 + t)
        = \sum_{k=0}^{m-1} \frac{\phi^{(k)}(x_0)}{k!} t^k
        + \frac{t^m}{(m-1)!} 
          \int_0^1 (1-s)^{m-1} \phi^{(m)}(x_0 + s t)\, \mathrm{d}s.
    \end{equation*}
    Setting $t = h S_{\bnu}(\bq) $, where $S_{\bnu}(\bq)= \sum_{i=1}^m \nu_i q_i$, and inserting this expansion into the definition of $T_{m}^{(x_0)}$ gives
    \begin{equation}
         T_{m}^{(x_0)}(\bq,h) = \frac{1}{h^m \phi^{(m)}(x_0)} \sum_{k=0}^{m-1} \frac{ h^k\phi^{(k)}(x_0)}{k!} B_m(\bq; k) + R, 
         \label{eq:expansion_Tm}
    \end{equation} 
    where 
    \begin{equation*}
        B_{{m}}(\bq; k) = \frac{1}{2^m} \sum_{\bnu \in  \{\pm1\}^m} \left( \prod_{i=1}^m \nu_i\right) [S_{\bnu}(\bq)]^k,
    \end{equation*}
    and the remainder term $R$ is 
    \begin{equation*}
        R = \frac{1}{2^m (m-1)! \phi^{(m)}(x_0)} \sum_{\bnu \in \{\pm 1\}^m} \left( \prod_{i=1}^m \nu_i\right) [S_{\bnu}(\bq)]^m\int_0^1 (1 - s)^{m-1} \phi^{(m)}(x_0 + s h S_{\bnu}(\bq)) \mathrm{d } s. 
    \end{equation*}

    By Lemma~\ref{lemma:monomials}, one has $B_m(\bq; k) = 0$ for all $k<m$. Hence all lower-order contributions vanish, and only the remainder term in~\eqref{eq:expansion_Tm} remains. We decompose the argument of $\phi^{(m)}$ in $R$ as
    \begin{equation*}
        \phi^{(m)}(x_0 + s h S_{\bnu}(\bq)) = \phi^{(m)}(x_0)  +  \left[\phi^{(m)}(x_0 + s h S_{\bnu}(\bq)) - \phi^{(m)}(x_0) \right].
    \end{equation*}
    Using Lemma~\ref{lemma:monomials} again and the identities
    \begin{equation*}
        B_m(\bq;m) = \frac{1}{2^m} \sum_{\bnu \in \{\pm 1\}^m} \left( \prod_{i=1}^m \nu_i \right) [S_{\bnu}(\bq)]^m = m! \prod_{i=1}^m q_i, \quad \quad  \int_0^1 (1 - s)^{m-1} \mathrm{d} s = \frac{1}{m},
    \end{equation*}
    we obtain the following error estimate
    \begin{equation*}
        \begin{aligned}
             \left| T_{m}^{(x_0)}(\bq, h) - \prod_{i=1}^m q_i\right| \leq \frac{1}{2^m (m-1)! |\phi^{(m)}(x_0)|} & \sum_{\bnu \in \{ \pm 1\}^m}   \Bigg[|S_{\bnu}(\bq)|^m \int_0^1 (1 -s)^{m-1} \\
             &\times  \left|\phi^{(m)}(x_0 + s h S_{\bnu}(\bq)) - \phi^{(m)}(x_0) \right| \mathrm{d} s\Bigg] .
        \end{aligned}
    \end{equation*}
    By the mean value theorem,
    \begin{equation*}
        \left|
          \phi^{(m)}(x_0 + s h S_{\bnu}(\bq)) - \phi^{(m)}(x_0)
        \right|
        \leq s h |S_{\bnu}(\bq)|
          \sup_{|t - x_0| \leq s h |S_{\bnu}(\bq)|}
          \left| \phi^{(m+1)}(t) \right|
        \leq s h A
          \sup_{|t - x_0| \leq  A}
          \left| \phi^{(m+1)}(t) \right|,
    \end{equation*}
    where the last inequality comes from $0<h<1, 0<s<1$ and $|S_{\bnu}(\bq)| \leq A\coloneqq  \sum_{i=1}^m |q_i|$. Using the bounds $|S_{\bnu}(\bq)|^m \leq A^{m}$ and 
    \begin{equation*}
         \int_0^1 (1 - s)^{m-1} s \mathrm{d} s = \frac{1}{m (m+1)},
    \end{equation*}
    we finally obtain
    \begin{equation*}
        \left|
          T_{m}^{(x_0)}(\bq,h) - \prod_{i=1}^m q_i
        \right|
        \leq 
        \frac{h A^{m+1}}{(m+1)! \, |\phi^{(m)}(x_0)|}
        \sup_{|t - x_0| \leq  A}
        \left| \phi^{(m+1)}(t) \right|.
    \end{equation*}
\end{proof}
With Lemma \ref{lemma:monomials}, we can construct two-layer neural networks to approximate all monomials of the form $x_1^{\alpha_1} x_2^{\alpha_2} \cdots x_d^{\alpha_d}$. 
\begin{lemma}\label{lemma:all_monomials}
    Let $d \in \mathbb{N}_{+}$ and let $\balpha = (\alpha_1, \ldots, \alpha_d) \in \mathbb{N}^d$ satisfy
    $\|\balpha\|_0 \coloneqq  \sum_{j=1}^{d} \alpha_j = m \ge 1$.
    Assume $\phi \in C^{m+1}(\mathbb{R})$ is non-polynomial. Then for any $Q>0$ and 
    sufficiently small $\epsilon \in (0,1)$, there exists a neural network
    \begin{equation*}
        g \in \mathcal{H}^{\phi,2}(d, 1, 2^m , B_{\epsilon}),
    \end{equation*}
    with  $B_{\epsilon}$ satisfying
    \begin{equation}
        B_{\epsilon} = \left( \frac{1}{\epsilon}\right)^m  \frac{(mQ)^{m(m+1)} \left[\sup_{|t-t_0| \leq m Q} |\phi^{(m+1)}(t)| \right]^{m}}{ [2(m+1)!]^m |\phi^{(m)}(t_0)|^{m+1}}.
        \label{eq:B_eps_1}
    \end{equation}
    for some $t_0\in\RR$, such that
    \begin{equation*}
         \sup_{\bx\in[-Q,Q]^d}
    \bigl|g(\bx)-x_1^{\alpha_1}\cdots x_d^{\alpha_d}\bigr|<\epsilon .
    \end{equation*}
\end{lemma}
\begin{proof}
    Since $\phi \in C^{m+1}(\RR)$ and is not a polynomial, there exists $t_0 \in \mathbb{R}$ such that $\phi^{(m)}(t_0) \neq 0$. Partition the index set ${1, \cdots, m}$ into disjoint subsets $I_1, \cdots, I_d$ defined by
   \begin{equation*}
       I_l = \left\{ i \;:\; 1+ \sum_{j=0}^{l-1}  \alpha_j \leq i \leq \sum_{l=0}^l \alpha_j \right\}, \quad 1 \leq l \leq d,
   \end{equation*}
   where  $\alpha_0 = 0$. Then $|I_l| = \alpha_l$ for every $l$. Define the neural network 
   \begin{equation*}
       g(\bx) = \frac{1}{2^m h^m \phi^{(m)}(t_0)} \sum_{\bnu \in \{\pm 1\}^m } \left( \prod_{i=1}^m \nu_i\right) \phi \left( t_0 + h \sum_{l=1}^d \left(\sum_{j \in I_l}  \nu_j \right) x_l  \right),
   \end{equation*}
   where $h>0$ is a parameter to be chosen. Setting
   \begin{equation*}
       q_j = x_l,  \quad j \in I_l , \quad \mathrm{for} \; l =1, \cdots, d,
   \end{equation*}
    then expression inside $\phi$ becomes $t_0 + h S_{\bnu}(\bq)$ with $S_{\bnu}(\bq) = \sum_{i=1}^m \nu_i q_i$. Thus,  by Lemma~\ref{lemma:Tm}, we have 
   \begin{equation*}
       |g(x) - x_1^{\alpha_2} x_2^{\alpha_2} \cdots x_d^{\alpha_d}| \leq \frac{h A^{m+1}}{ (m+1)!|\phi^{(m)}(t_0)|} \sup_{|t-t_0| \leq A} \left| \phi^{(m+1)}(t)\right|,
   \end{equation*}
   where $A = \sum_{i=1}^m |q_i| = \sum_{l=1}^d \alpha_l |x_l|$. Since $|x_l| \leq Q$, we have $A \leq mQ$. Choose
   \begin{equation*}
       h = \frac{\epsilon (m+1)! |\phi^{(m)}(t_0)| }{(mQ)^{m+1} [\sup_{|t-t_0| \leq mQ}|\phi^{(m+1)}(t)|]}. 
   \end{equation*}
   With this choice,
   \begin{equation*}
      \sup_{\bx \in [-Q,Q]^d} |g(\bx) - x_1^{\alpha_1} x_2^{\alpha_2} \cdots x_d^{\alpha_d} | \leq \epsilon.
   \end{equation*}
   There are $2^m$ neurons in the hidden layer of $g$ and all weights and biases in $g$ are bounded in magnitude by
   \begin{equation*}
       \max \left\{t_0,\; \frac{1}{2^m h^m |\phi^{(m)}(t_0)|},\; h \left| \sum_{j \in I_l} \nu _j \right|\right\} = \frac{1}{2^m h^m | \phi^{(m)}(t_0)|},
   \end{equation*}
   for sufficiently small $\epsilon$. This equals $B_{\epsilon}$ as stated in~\eqref{eq:B_eps_1}. We complete the proof.
\end{proof}

\begin{remark}
    In Lemma~\ref{lemma:all_monomials}, the condition ``sufficiently small $\epsilon \in (0,1)$'' implies the existence of a threshold $\epsilon_0 \in (0,1)$ such that the statement holds for all $0 < \epsilon < \epsilon_0$. Crucially, $\epsilon_0$ is independent of the parameters $K$ and $\delta$, but is allowed to depend on  problem specifications, such  as activation $\phi$, dimension $d$, smoothness $s$, among others.  We adopt this convention throughout the subsequent analysis.
\end{remark}

\begin{remark}
    If the input amplitude $Q$ in Lemma~\ref{lemma:all_monomials} is independent of $\epsilon, \delta, K$, then norm  $B_{\epsilon}$ defined in~\eqref{eq:B_eps_1} satisfies $B_{\epsilon} \lesssim (1 / \epsilon)^m$.
\end{remark}

The construction in Lemma~\ref{lemma:all_monomials} relies on the existence of a point $x_0$ at which $\phi^{(m)}(x_0) \neq 0$. The following lemma establishes that for any smooth nonpolynomial function $\phi$, there exists such a point for all $m \in \mathbb{N}$.

\begin{lemma}
[\citet{corominas1954condiciones,donoghue1969distributions}]\label{lemma:nonpoly_derivative_nonzero}
    Let $\phi \in C^{\infty}(\mathbb{R})$. 
    If $\phi$ is not a polynomial, then there exists a point $t_0 \in \mathbb{R}$ such that
    \[
        \phi^{(m)}(t_0) \neq 0, \quad m \in \mathbb{N}.
    \]
\end{lemma}
In the subsequent analysis, if $\phi$ satisfies the smoothness assumption in Assumption~\ref{ass:piece}, we assume that the point $t_0$ is such that $\phi^{(m)}(t_0) \neq 0, m\in \mathbb{N}$.

As a corollary of Lemma~\ref{lemma:all_monomials}, we obtain an approximation of the identity function, which serves as a fundamental tool in the subsequent analysis.
\begin{corollary}
\label{corol:id-approx} Fix $Q>0$. Let $\phi$ be an activation function satisfying Assumption~\ref{ass:smooth} with $\phi'(t_0) \neq 0$. For any $L \geq 2$, and sufficiently small $\epsilon \in(0,1)$, there exists a neural network
\begin{equation*}
    g \in \mathcal{H}^{\phi,L}(1, 1, 2, B_{\epsilon}),
\end{equation*}
with $B_{\epsilon}$ satisfying
\begin{equation*}
    B_{\epsilon} \leq \frac{L-1}{\epsilon} \frac{Q^2 \sup_{|t- t_0| \leq Q} |\phi^{\prime \prime}(t)|}{4 |\phi^{\prime}(t_0)|^2},
\end{equation*}
such that
\begin{equation*}
    \sup_{x \in [-Q,Q]} |g(x) - x| < \epsilon.
\end{equation*}
\end{corollary}
\begin{proof}
   By Lemma~\ref{lemma:all_monomials}, there exist neural networks $\{\nu_i\}_{i=1}^{L-1}$ such that for each $i$,
    \begin{equation*}
        \sup_{x \in [-Q-1, Q+1]} |\nu_i(x) - x| \leq \frac{\epsilon}{L-1}.
    \end{equation*}
    Each network $\nu_i$ has depth 2, width 2, and parameter norm bounded by
    \begin{equation}
        \|\theta(\nu_i)\|_{\infty}\leq \frac{L-1}{\epsilon} \frac{Q^2 \sup_{|t- t_0| \leq Q} |\phi^{\prime \prime}(t)|}{4 |\phi^{\prime}(t_0)|^2}.
        \label{eq:id-norm-nui}
    \end{equation}
    We define the composite network $g \coloneqq \nu_{L-1} \circ \nu_{L-2} \circ \cdots \circ \nu_1$. Then for $x \in [-Q, Q]$, the approximation error is bounded by:
    \begin{equation*}
        |g(x) - x| \leq \sum_{l=2}^{L-1} |\nu_{l} \circ \cdots \circ \nu_1(x) - \nu_{l-1} \circ \cdots \circ \nu_1(x)| + |\nu_1(x) -x| \leq (L-1) \frac{\epsilon}{L-1} \leq \epsilon.
    \end{equation*}
    It follows that the depth of $g$ is $L$, its width is $2$. Since the parameter norm bound \eqref{eq:id-norm-nui} holds for each composition $\nu_i$, it also applies to $g$.
    This completes the proof.
\end{proof}

\subsection{Approximation of Piecewise Constants}
\label{subsec:app_pwcf}
In this subsection, we construct shallow neural network approximations for the following piecewise constant functions, consisting of $K^{2d}$ distinct pieces:
\begin{equation}
    C(\bx) = \sum_{\bi \in [K]^d, \bj \in [K]^d} c_{\bi, \bj} \mathbbm{1}_{\Omega^{K}_{\bi, \bj}} (\bx)
    \label{eq:pwcf}
\end{equation}
We begin by reformulating $C(\bx)$ using the following lemma. 
\begin{lemma}
    Let $C(\bx)$ be defined by~\eqref{eq:pwcf}. Then 
    \begin{equation}
        C(\bx) =\sum_{\bj \in [K]^d} \left( \sum_{\bi \in [K]^d} c_{\bi, \bj} \mathbbm{1}_{\Omega_{\bi}^K}(\bx) \right) \mathbbm{1}_{\Omega_{\bone, \bj}^{ K}} \left(\bx - \sum_{\bi \in [K]^d} \ba_{\bi}^K  \mathbbm{1}_{\Omega_{\bi}^K}(\bx)\right).
        \label{eq:reformulate_pwfc}
    \end{equation}
    Here, $\ba_{\bi}^K$ denotes the lower-left vertex associated with the cell $\Omega_{\bi}^K$. Specifically, for $\bi = (i_1,\ldots,i_d)$,
    \begin{equation*}
        \ba_{\bi}^{K} \coloneqq  \left( \frac{i_1 -1}{K}, \cdots, \frac{i_d -1}{K}\right) .
    \end{equation*}
    \label{lem:reformulat_pwfc}
\end{lemma}
\begin{proof}
    For any $\bi \in [K]^d, \bj \in [K]^d$, the following identity holds:
    \begin{equation*}
        \begin{aligned}
            \mathbbm{1}_{\Omega_{\bi, \bj}^K} (\bx) & = \mathbbm{1}_{\Omega_{\bi}^K} (\bx) \mathbbm{1}_{\Omega_{\bone, \bj}^{ K}} \left(\bx - \ba_{\bi}^K \right) \\
            & = \mathbbm{1}_{\Omega_{\bi}^K}(\bx)\mathbbm{1}_{\Omega_{\bone, \bj}^{ K}} \left(\bx - \sum_{\bi \in [K]^d} \ba_{\bi}^K  \mathbbm{1}_{\Omega_{\bi}^K} (\bx)\right).
        \end{aligned}
    \end{equation*}
    Substituting this representation of $\mathbbm{1}_{\Omega_{\bi,\bj}^K} (\bx )$ into~\eqref{eq:pwcf}, summing first over $\bi$ and subsequently over $\bj$, yields the desired expression for the piecewise constant function $C$ in~\eqref{eq:reformulate_pwfc}.
\end{proof}
With the reformulation of $C$ provided by Lemma~\ref{lem:reformulat_pwfc}, the problem of approximating a piecewise constant function with $K^{2d}$ pieces is reduced to approximating $K^d$ distinct piecewise constant functions, each of which consists of $K^d$ regions of the form $\{\Omega_{\bi}^K\}_{\bi \in [K]^d}$. More explicitly, for each $\bj \in [K]^d$, we consider the function
\begin{equation}
    C_{\bj}(\bx) = \sum_{\bi \in [K]^d} c_{\bi, \bj} \mathbbm{1}_{\Omega_{\bi}^K} \left( \bx \right), \quad \bj \in [K]^d.
    \label{eq:pwcf-coarse}
\end{equation}
It is important to note that all these $K^d$ piecewise constant functions share exactly the same partition $\{\Omega_{\bi}^K\}_{\bi \in [K]^d}$. 
Besides, we also need to extract the position of $\bx$ relative to the lower-left corner $\ba_{\bi}^K$ of the region $\Omega_{\bi}^K$ containing it, and subsequently determine which refined grid cell ${\Omega_{\bone, \bj}^{ K}}$ this relative coordinate falls into.

We first construct approximations of
indicator functions in one dimension.
\begin{lemma}
    Let $\phi$ satisfy Assumption~\ref{ass:piece}.
    For any $a<b$ and $\delta\in\bigl(0,\tfrac{b-a}{3}\bigr)$, and any $\epsilon >0$ sufficiently small, there exists a neural network
    \begin{equation*}
        g \in \mathcal{H}^{\phi,2} \left(1, 1, 2,   \frac{4C_1 (|a| + |b|+1)}{\epsilon \delta} \right),
    \end{equation*}
    where $C_1$ is the constant specified in Assumption~\ref{ass:piece}, such that 
    \begin{itemize}
        \item (Approximation) $|g(x) - \mathbbm{1}_{[a,b)} (x)| < \epsilon, \quad x \notin [a, a+ \delta] \cup [b-\delta, b]$.
        \item (Boundedness) $\|g\|_{L^{\infty}(\RR)} \leq 2 (C_1 + 1)$.
    \end{itemize}
    \label{lem:1D_indicator_app_Heaviside}
\end{lemma}

\begin{proof}
    Define 
    \begin{equation}
        g (x) = \phi \left( \beta \left( x - \left(a + \frac{\delta}{2}\right)\right) \right) - \phi \left( \beta \left( x - \left(b - \frac{\delta}{2}\right)\right) \right),
        \label{eq:construction_varphi1}
    \end{equation}
    where $\beta>0$ will be chosen later.
    For every $x \notin [a, a+ \delta] \cup [b-\delta, b]$, we have 
    \begin{equation*}
        \begin{aligned}
             \mathbbm{1}_{[a,b)} (x) &= \mathbbm{1}_{[a+\frac{\delta}{2},b-\frac{\delta}{2})} (x) \\
             &= H \left( \beta \left( x - \left(a + \frac{\delta}{2}\right)\right)  \right) - H\left( \beta \left( x - \left(b - \frac{\delta}{2}\right)\right)  \right),
        \end{aligned}
    \end{equation*}
    Hence, still for $x \notin [a, a+ \delta] \cup [b-\delta, b]$
    \begin{equation*}
            \left| g(x) - \mathbbm{1}_{[a,b)}(x)\right|  
             \leq \left| (\phi - H) \left( \beta \left( x - \left(a + \frac{\delta}{2}\right)\right)  \right) \right| + \left| (\phi - H) \left( \beta \left( x - \left(b - \frac{\delta}{2}\right)\right)  \right)\right|.
    \end{equation*}
    By the Heaviside-like assumption on $\phi$, each term is bounded by $2C_1 / (\beta \delta)$. Choosing $\beta = \frac{4C_1}{\delta \epsilon}$ therefore gives
    \begin{equation*}
         \left| g(x) - \mathbbm{1}_{[a,b)}(x)\right| \leq \epsilon, \quad x \notin  [a, a+ \delta] \cup [b-\delta, b].
    \end{equation*}
    Moreover, since $\|H\|_{L^{\infty}(\RR)} \leq 1$ and $\phi$ satisfies the Heaviside-like condition, we have $\|\phi\|_{L^{\infty}(\RR)} \leq C_1 + 1$. Consequently
    \begin{equation*}
        \|g\|_{L^{\infty}(\RR)} \leq 2 (C_1 + 1).
    \end{equation*}
    Finally, from the explicit construction \eqref{eq:construction_varphi1}, each parameter involved in $g$ is bounded in magnitude by 
    \begin{equation*}
        \max \left\{ \frac{4C_1 (|a| + |b|+1)}{\epsilon \delta}, 1\right\} = \frac{4C_1 (|a| + |b| + 1)}{\epsilon \delta},
    \end{equation*}
    for sufficiently small $\epsilon$. We complete the proof.
\end{proof}

\begin{lemma}
    With the same assumptions as in Lemma~\ref{lem:1D_indicator_app_Heaviside}, except that $\phi$ is a ReLU-like activation function, there exists a neural network
    \begin{equation*}
        g \in  \mathcal{H}^{\phi,2}\left(1, 1, 4,  \frac{4C_2 (|a|+|b|+1)}{\epsilon \delta} \right),
    \end{equation*}
    where $C_2$ is the constant specified in Assumption~\ref{ass:piece}, such that
    \begin{itemize}
        \item (Approximation) $|g(x) - \mathbbm{1}_{[a,b)} (x)| < \epsilon, \quad x \notin [a,a+\delta] \cup [b-\delta,b]$.
        \item (Boundedness) $\|g\|_{L^{\infty}(\RR)} \leq 2$ .
    \end{itemize}
    \label{lem:1D_indicator_app_ReLU}
\end{lemma}

\begin{proof}
    Define 
    \begin{equation}
        \begin{aligned}
            \psi(x) & = \frac{1}{\beta \delta} \, \mathrm{ReLU}(\beta(x-a)) - \frac{1}{\beta \delta}  \, \mathrm{ReLU} (\beta (x - (a+ \delta))) \\
            & + \frac{1}{\beta \delta}  \,\mathrm{ReLU}(\beta (x - (b - \delta))) + \frac{1}{ \beta \delta}  \, \mathrm{ReLU} (\beta (x - b)), 
        \end{aligned}
        \label{eq:construction_ReLU}
    \end{equation}
    where $\beta >0$ is a parameter to be chosen. One checks directly that
    \begin{equation}
        \psi(x) = \mathbbm{1}_{[a,b)} (x) ,\quad  x \notin [a,a+\delta] \cup [b-\delta,b]
        \label{eq:devi_psi}
    \end{equation}
    and that $\|\psi\|_{L^{\infty}(\RR)} \leq 1$.

    Now construct $g$ by replacing each ReLU activation in ~\eqref{eq:construction_ReLU} with $\phi$. The ReLU-like condition implies the uniform approximation bound
    \begin{equation*}
       |\psi(x) - g(x)| \leq \frac{4 C_2 }{\beta \delta}, \quad x \in \RR.
    \end{equation*}
    Choosing $\beta = \frac{4 C_2}{\epsilon \delta}$ and combining~\eqref{eq:devi_psi} gives 
    \begin{equation*}
        |g(x) - \mathbbm{1}_{[a,b)}(x)| \leq \epsilon, \quad  x \notin [a,a+\delta] \cup [b-\delta,b].
    \end{equation*}
    Furthermore, since $\|\psi\|_{L^{\infty}(\RR)} \leq 1$, it follows that for sufficiently small $\epsilon$, 
    \begin{equation*}
        \|g\|_{L^{\infty}(\RR)} \leq 1 + \epsilon < 2  .
    \end{equation*}
    Finally, using the explicit construction of $g$, every parameter in the network is bounded in magnitude by
    \begin{equation*}
        \max \left\{  \frac{1}{\beta \delta},  \beta (|a| + |b| + 1)\right\} = \max \left\{ \frac{\epsilon}{4 C_2} , \frac{4C_2 (|a| + |b| +1)}{\epsilon \delta}\right\} = \frac{4C_2 (|a|+ |b|+1)}{\epsilon \delta}.
    \end{equation*}
    This completes the proof.
\end{proof}
We next construct approximations of indicator functions in general dimensions $d \geq 1$. The following two lemmas characterize neural network approximations of all indicator functions associated with the coarse grid.
\begin{lemma}
    Let $\phi$ satisfy Assumptions~\ref{ass:smooth}--\ref{ass:piece}. Fix $d\in\mathbb{N}_{+}$ and $K\in\mathbb{N}_{+}$ with $K$ sufficiently large.
    Then for any sufficiently small $\epsilon \in (0,1)$ and any $\delta \in (0, \frac{1}{3K})$, there exists a neural network 
    \begin{equation*}
        g \in \mathcal{H}^{\phi,3} \left( d, K^d, 2^{d+1} K^d, B_{\epsilon, \delta}\right),
    \end{equation*}
    with $B_{\epsilon,\delta}$ satisfying
    \begin{equation*}
        B_{\epsilon, \delta} \lesssim \max \left\{  \frac{3}{\epsilon \delta}, \left( \frac{1}{\epsilon}\right)^d\right\},
    \end{equation*}
    such that, for each $\bi\in[K]^d$, the $\bi$-th output $[g(\bx)]_{\bi}$ satisfies:
    \begin{itemize}
        \item (Approximation) $\left|  [g(\bx)]_{\bi} - \mathbbm{1}_{\Omega_{\bi}^{K}} (\bx) \right| \leq \epsilon, \quad \bx \notin \Omega^{K}_{\bi, \mathrm{band}}(\delta).$
        \item (Boundedness) For $\bx \in \RR^d$,
        \begin{equation}
            |[g(\bx)]_{\bi}| \le 
            \begin{cases}
                2^{d+1} (1+ C_1)^d, & \textrm{Heaviside-like} \; \phi ,\\
                2^{d+1}, & \textrm{ReLU-like} \; \phi,
            \end{cases}  
            \label{eq:dD_indicator_bound}
        \end{equation}
        where $C_1$ is the constant appearing in Assumption~\ref{ass:piece}.
    \end{itemize}
    \label{lem:dD_indicator_app}
\end{lemma}
\begin{proof}
   We prove the result for the Heaviside-like $\phi$, and the proof for the ReLU-like $\phi$ follows in the same manner. For $\bx\in\RR^d$, by invoking Lemma~\ref{lem:1D_indicator_app_Heaviside}, for sufficiently small $\epsilon_1 >0$ and $\delta \in (0, \frac{1}{3K})$, there exist $dK$ sub-networks 
    \begin{equation*}
        h_{l,i} \in \mathcal{H}^{\phi,2}\left( d, 1, 2,  B_{\epsilon_1}\right),
    \end{equation*}
    with $B_{\epsilon_1} = \frac{12C_1}{\delta \epsilon_1}$ satisfying, for all $l \in [d]$, $i \in [K]$
    \begin{equation*}
        \left| h_{l,i}(\bx) - \mathbbm{1}_{[\frac{i-1}{K},\frac{i}{K})} (x_l)\right| \leq \epsilon_1, \quad x_l \notin \left[\frac{(i-1)}{K}, \frac{(i-1)}{K} + \delta\right] \cup \left[ \frac{i}{K}-\delta, \frac{i}{K} \right],
    \end{equation*}
    and $\|h_{l,i}\|_{L^{\infty}(\RR^d)} \leq 2(1+C_1)$. Furthermore, by lemma~\ref{lemma:all_monomials}, for sufficiently small $\epsilon_2 \in (0,1)$, there exists a neural network 
    \begin{equation*}
        \psi \in \mathcal{H}^{\phi,2}(d, 1,2^d, B_{\epsilon_2}),
    \end{equation*}
    with $B_{\epsilon_2} \lesssim (1 / \epsilon_2)^d$
    such that for $\bq = [q_1, \cdots, q_d] \in \mathbb{R}^d$
    \begin{equation*}
        \sup_{\bq \in [-2(1+C_1), 2(1+C_1)]^d} \left| \psi(\bq) - q_1 q_2 \cdots q_d \right| \leq \epsilon_2.
    \end{equation*}
    Next, for $\bi = (i_1, \cdots, i_d)$, we define the $\bi$-th output of $g$ as 
    \begin{equation*}
        [g(\bx)]_{\bi} = \psi \left( h_{1, i_1}(\bx), \cdots, h_{d,i_d}(\bx)\right).
    \end{equation*}
    For sufficiently small $\epsilon_1$, if $\bx \in \Omega_{\bi, \mathrm{int}}^{K}(\delta)$, we have $|h_{l,i_l}(\bx) - 1| \leq \epsilon_1$, hence
    \begin{equation*}
        \left| \prod_{l=1}^d h_{l,i_l}(\bx) - 1\right| \leq \max \left\{ (1+\epsilon_1)^d-1, 1 -(1-\epsilon_1)^d\right\} < 2d \epsilon_1, \quad \bx \in \Omega_{\bi, \mathrm{int}}^{K}(\delta).
    \end{equation*}
    Conversely, for $\bx \notin \Omega_{\bi}^{K}$, there exists at least $l$ such that $|h_{l, i_l}(\bx)| \leq \epsilon_2$, therefore 
    \begin{equation*}
        \left|\prod_{l=1}^d h_{l,i_l}(\bx) \right| \leq 2^{d}(1+C_1)^{d} \epsilon_1 \quad \bx \notin \Omega_{\bi}^{M, K}.
    \end{equation*}
    By choosing $\epsilon_1 = \frac{\epsilon}{2^{d+1}(1+C_1)^d}$ and $\epsilon_2 =  \frac{\epsilon}{2}$, we obtain
    \begin{equation*}
        \left|[g(\bx)]_{\bi} - \mathbbm{1}_{\Omega_{\bi}^{M, K}} \left(\bx  \right)\right| < 2^d(1+C_1)^d \epsilon_1 + \epsilon_2 < \epsilon, \quad \bx \notin \Omega_{\bi,\mathrm{band}}^{K}(\delta).
    \end{equation*}
    Moreover, since $\|h_{l,i_l}\|_{L^{\infty}(\RR^d)} \leq 2(1+C_1)$, we obtain
    \begin{equation*}
        \left| [g(\bx)]_{\mathrm{i}} \right| \leq 2^d(1+C_1)^d + \epsilon_2 < 2^{d+1} (1+C_1)^d, \quad \bx \in \RR^d,
    \end{equation*}
    Finally, regarding the network architecture, the neural network $g$ consists of $2dK$ neurons in the first hidden layer and  $2^d K^d$  neurons in the second hidden layer. Therefore, for sufficiently large $K$, the width of $g$ is bounded as $\max\{2dK,2^d K^d\} \le 2^{d+1} K^d$. Given that $\|\theta(h_{l,i})\|_{\infty} = B_{\epsilon_1}, \|\theta(\psi)\|_{\infty} = B_{\epsilon_2}$, and the connecting weights between the networks (output layer for $h_{l,i}$ and the input layer for $\psi$)
    are bounded by $O(1)$, we conclude that $\|\theta(g)\|_{\infty}$ is bounded by 
    \begin{equation*}
        B_{\epsilon, \delta} \lesssim \max \left\{ B_{\epsilon_1}, B_{\epsilon_2}, 1\right\} \lesssim  \max \left\{  \frac{3}{\epsilon \delta}, \left( \frac{1}{\epsilon}\right)^d\right\},
    \end{equation*}
    This completes the proof.
\end{proof}

Using the neural network approximations for all indicator functions over the coarse grid in Lemma~\ref{lem:dD_indicator_app}, we can approximate piecewise constant functions over the coarse grid $\{\Omega_{\bi}^K\}_{\bi \in [K]^d}$.

\begin{lemma}
    Let $\phi$ satisfy Assumptions~\ref{ass:smooth}--\ref{ass:piece}. Let $c_{\bi, \bj}$ be given in~\eqref{eq:pwcf} and assume that $\max_{\bi, \bj} \{|c_{\bi,\bj}|\} \leq c_{\max}$. Fix $d\in\mathbb{N}_{+}$ and $K\in\mathbb{N}_{+}$ with $K$ sufficiently large.
    Then for any sufficiently small $\epsilon \in (0,1)$ and any $\delta \in (0, \frac{1}{3K})$, there exists a neural network 
    \begin{equation*}
        g \in \mathcal{H}^{\phi,3} \left( d, K^d, 2^{d+1} K^d, B_{\epsilon, \delta, K}\right),
    \end{equation*}
    with $B_{\epsilon, \delta, K}$ satisfying
    \begin{equation}
        B_{\epsilon, \delta, K} \lesssim \max \left\{ \frac{c_{\max} K^{d}}{\epsilon \delta}, \frac{c_{\max}^d K^{d^2}}{\epsilon^d}\right\},
        \label{eq:norm_pwcf_coarse}
    \end{equation}
    such that, for each $\bj\in[K]^d$, the $\bj$-th output $[g(\bx)]_{\bj}$ satisfies:
    \begin{itemize}
        \item (Approximation) For $C_{\bj}$ given in~\eqref{eq:pwcf-coarse},
        $\left| [g(\bx)]_{\bj} - C_{\bj}(\bx) \right| \leq \epsilon, \quad \bx \notin \cup_{\bi} \Omega_{\bi, \mathrm{band}}^{K}(\delta).$
        \item (Boundedness) For $\bx \in \RR^d$, 
        \begin{equation*}
            |[g(\bx)]_{\bi}| \le \begin{cases}
                c_{\max}2^{d+2} (1+ C_1)^d, &   \textrm{Heaviside-like} \; \phi ,\\
                c_{\max}2^{d+2}, &  \textrm{ReLU-like} \; \phi,
            \end{cases}   
        \end{equation*}
        where $C_1$ is the constant appearing in Assumption~\ref{ass:piece}.
    \end{itemize}
    \label{lem:dD_pwfc_coarse_app}
\end{lemma}
\begin{proof}
     We prove the result for the Heaviside-like $\phi$, and the proof for the ReLU-like $\phi$ follows in the same manner. Let $\psi$ be the neural network constructed in Lemma~\ref{lem:dD_indicator_app}, sufficiently small $\epsilon_{1} \in (0,1)$, and $\delta \in (0,\tfrac{1}{3K})$. We have $\|\theta(\psi)\|_{\infty}\lesssim \max \{3 /(\epsilon_1\delta), (1/ \epsilon_1)^d\}$.
     Define the neural network $g$ by specifying its $\bj$-th output component as
    \begin{equation*}
        [g(\bx)]_{\bj} = \sum_{\bi \in [K]^d} c_{\bi, \bj} [\psi(\bx)]_{\bi}.
    \end{equation*}
    Then, for any $\bx \notin \bigcup_{\bi} \Omega_{\bi,\mathrm{band}}^{K}(\delta)$ and any $\bj\in[K]^d$, we obtain
    \begin{equation*}
        \begin{aligned}
            \left| [g(\bx)]_{\bj} -C_{\bj}(\bx) \right| & = \left| \sum_{\bi \in [K]^d} c_{\bi, \bj} [\psi(\bx)]_{\bi} -\sum_{\bi \in [K]^d} c_{\bi, \bj}\mathbbm{1}_{\Omega_{\bi}^K}(\bx) \right| \\
            & \leq \sum_{\bi \in [K]^d} |c_{\bi, \bj}|  \left| [\psi(\bx)]_{\bi} - \mathbbm{1}_{\Omega_{\bi}^K} (\bx ) \right| \\
            &\leq  c_{\max} K^{d} \epsilon_1.
        \end{aligned}
    \end{equation*}
    By selecting $\epsilon_{1} = \frac{\epsilon}{c_{\max} K^d}$, we obtain $\bj \in [K]^d$
    \begin{equation*}
        |[g(\bx)]_{\bj} - C_{\bj}(\bx)| \leq \epsilon, \quad \bx \notin \cup_{\bi} \Omega_{\bi, \mathrm{band}}^{K}(\delta),\quad \bj \in [K]^d.
    \end{equation*}
    Moreover, if $\bx \notin [0,1)^d$, then $C_{\bj}(\bx) = 0$. Hence, for any $\bj \in [K]^d$
    \begin{equation*}
        |[g(\bx)]_{\bj}| = |[g(\bx)]_{\bj} - C_{\bj}(\bx)| \leq \epsilon.
    \end{equation*}
    On the other hand, if $\bx \in [0,1)^d$, let $\bi\in[K]^d$ be such that $\bx \in \Omega_{\bi}^{K}$. Then 
    \begin{equation*}
        \begin{aligned}
            |[g(\bx)]_{\bj}| & \leq |c_{\bi, \bj}| |[\psi(\bx)]_{\bi}| + \sum_{\bl \in [K]^d, \bl \neq \bi} |c_{\bl, \bj}| |[\psi(\bx)]_{\bl} - \mathbbm{1}_{\Omega_{\bl}^K} (\bx)| \\
            &\leq 2^{d+1}c_{\max}(1+C_1)^d +   \epsilon,
        \end{aligned}
    \end{equation*}
    Therefore, for $\epsilon$ sufficiently small, we obtain the uniform bound for $\bx \in \RR^d$
    \begin{equation*}
        |[g(\bx)]_{\bj}| \leq c_{\max}  2^{d+2} (1 + C_1)^d, \quad \bj \in [K]^d.
    \end{equation*}
    Finally, the input dimension, output dimension, and width of $g$ coincide with those of $\psi$. Moreover, $\|\theta(g)\|_{\infty}$ is bounded by 
    \begin{equation*}
        B_{\epsilon, \delta, K} \leq c_{\max} \|\theta(\psi)\|_{\infty} \lesssim \max \left\{ \frac{c_{\max} K^{d}}{\epsilon \delta}, \frac{c_{\max}^d K^{d^2}}{\epsilon^d}\right\},
    \end{equation*}
    as stated in~\eqref{eq:norm_pwcf_coarse}. This concludes the proof.
\end{proof}

Next, we construct neural network modules that approximate the mapping
\begin{equation*}
    \bx \mapsto \bx - \sum_{\bi \in [K]^d} \ba_{\bi}^K \mathbbm{1}_{\Omega_{\bi}^K} (\bx), 
\end{equation*}
thereby extracting the relative position of $\bx$ within its associated partition grid cell.
\begin{lemma}
    Let $\phi$ satisfy Assumptions~\ref{ass:smooth}--\ref{ass:piece}. Fix $d\in\mathbb{N}_{+}$ and $K\in\mathbb{N}_{+}$ with $K$ sufficiently large.
    Then for any sufficiently small $\epsilon>0$ and any $\delta \in (0, \frac{1}{3K})$, there exists a neural network 
    \begin{equation*}
        g \in \mathcal{H}^{\phi,2}(d, d, 6dK, B_{\epsilon, \delta, K}),
    \end{equation*}
    with $B_{\epsilon, \delta, K} \lesssim K/(\epsilon \delta)$, such that
    \begin{equation}
        \left\| g(\bx) - \left( \bx - \sum_{\bi \in [K]^d} \ba_{\bi}^K \mathbbm{1}_{\Omega_{\bi}^K}(\bx) \right)\right\|_{\infty} \leq \epsilon, \quad \forall \bx \notin \cup_{\bi \in [K]^d} \Omega_{\bi, \mathrm{band}} ^K(\delta). 
        \label{eq:Linfty_relative_pos}
    \end{equation}
     \label{lem:relative_pos_app}
\end{lemma}
\begin{proof}
    We prove the result for the Heaviside-like $\phi$, and the proof for the ReLU-like $\phi$ follows in the same manner. We first rewrite the  $\bx - \sum_{\bi \in [K]^d} \ba_{\bi}^K \mathbbm{1}_{\Omega_{\bi}^K} (\bx)$ in the component form as:
    \begin{equation*}
        \bx - \sum_{\bi \in [K]^d} \ba_{\bi}^K \mathbbm{1}_{\Omega_{\bi}^K} (\bx) = \left(
        \begin{array}{c}
            x_1 - \sum_{i=1}^K \frac{i-1}{K} \mathbbm{1}_{[\frac{i-1}{K},\frac{i}{K})} ( x_1)\\
            \cdots \\
            x_d - \sum_{i=1}^K \frac{i-1}{K} \mathbbm{1}_{[\frac{i-1}{K},\frac{i}{K})} ( x_d)
            \end{array}
        \right).
    \end{equation*}
    By Corollary~\ref{corol:id-approx}, for sufficiently small $\epsilon_1$, there exist neural networks
    \begin{equation*}
        \psi_l \in \mathcal{H}^{\phi,2}(d, 1, 2, B_{\epsilon_1}), \quad l \in [d] ,
    \end{equation*}
    with $B_{\epsilon_1} \lesssim 1/\epsilon_1$ such that \begin{equation*}
        |\psi_l(\bx) - x_l| \leq \epsilon_1, \quad \bx \in [0,1)^d.
    \end{equation*}
    Besides, by Lemma~\ref{lem:1D_indicator_app_Heaviside},
    there exist $dK$ sub-networks for sufficiently small $\epsilon_2$
    \begin{equation*}
        h_{l,i_l} \in \mathcal{H}^{\phi,2}\left( d, 1, 2, B_{\epsilon_2, \delta}\right),
    \end{equation*}
    with $B_{\epsilon_2, \delta} \lesssim 1 / (\epsilon_2 \delta) $satisfying, for all $l \in [d]$, $i_l \in [K]$
    \begin{equation*}
        \left| h_{l,i_l}(\bx) - \mathbbm{1}_{[\frac{i_l-1}{K},\frac{i_l}{K})} (x_l)\right| \leq \epsilon_2, \quad \bx \notin \Omega_{\bi, \mathrm{band}}^{K}(\delta). 
    \end{equation*}
    Define the $l$-th output of the constructed neural networks $g$ as 
    \begin{equation*}
        [g(\bx)]_l = \psi_l(\bx) - \sum_{i=1}^K \frac{i-1}{K} h_{l, i }(\bx), \quad l \in [d]. 
    \end{equation*}
    Then, for any $l \in [d]$, for $\bx \notin \Omega_{\bi, \mathrm{band}}^{K}(\delta)$, we have 
    \begin{equation*}
        \begin{aligned}
            & \quad \left| [g(\bx)]_l - \left( x_l - \sum_{i=1}^K \frac{i-1}{K} \mathbbm{1}_{[\frac{i-1}{K},\frac{i}{K})} (x_l) \right)\right| \\
            & \leq |\psi_l(\bx) - x_l| + \sum_{i=1}^K \frac{i - 1}{K} \left| h_{l, i}(\bx) - \mathbbm{1}_{[\frac{i-1}{K},\frac{i}{K})} (x_l)\right| \\
            & \leq \epsilon_1 + K \epsilon_{2} \leq \epsilon,
        \end{aligned}
    \end{equation*}
    where we choose $\epsilon_1 = \frac{\epsilon}{2}$ and $\epsilon_2 = \frac{\epsilon}{2 K}$ in the last inequality. This proves the approximation guaranty in~\eqref{eq:Linfty_relative_pos}.

    Finally, each network $\psi_{l}$ and each subnetwork $h_{l,i}$ employs two hidden neurons. Consequently, $g$ comprises a total of $2d(K+1)$ hidden neurons, and its width is therefore bounded by $6dK$ for all sufficiently large $K$. Moreover, $\|\theta(g)\|_{\infty}$ is bounded by
    \begin{equation*}
        B_{\epsilon,\delta, K} \leq \max \left\{ B_{\epsilon_1}, B_{\epsilon_2, \delta}\right\} \lesssim \frac{K}{\epsilon \delta}.
    \end{equation*}
    This completes the proof.
\end{proof}

Using the relative-position extraction established in Lemma~\ref{lem:relative_pos_app},  we subsequently construct neural network approximations of indicator functions on the refined grid cells.
\begin{lemma}
    Let $\phi$ satisfy Assumptions~\ref{ass:smooth}--\ref{ass:piece}. Fix $d\in\mathbb{N}_{+}$ and $K\in\mathbb{N}_{+}$ with $K$ sufficiently large.
    Then for any sufficiently small $\epsilon>0$ and any $\delta \in (0, \frac{1}{3K^2})$, there exists a neural network   
    \begin{equation*}
        g \in \mathcal{H}^{\phi,4} \left(d, K^d,  2^{d+2} K^d, B_{\epsilon, \delta, K} \right),
    \end{equation*}
    with $B_{\epsilon,\delta,K}$ satisfying
    \begin{equation}
        B_{\epsilon,\delta, K} \lesssim \max \left\{ \frac{K^d}{\epsilon \delta^2}, \frac{K^{d^2}}{\epsilon^d} \right\},
        \label{eq:norm_indicator_relative_pos}
    \end{equation}
    such that, for each $\bj\in[K]^d$, the $\bj$-th output $[g(\bx)]_{\bj}$ satisfies:
    \begin{itemize}
        \item (Approximation) For $\bx \in \cup_{\bi, \bj} \Omega_{\bi, \bj, \mathrm{int}}^{K}(\delta)$, 
        \begin{equation*}
            \left|  [g(\bx)]_{\bj} - \mathbbm{1}_{\Omega_{\bone, \bj}^{K}}  \left(\bx - \sum_{\bi \in [K]^d} \ba_{\bi}^K  \mathbbm{1}_{\Omega_{\bi}^K} (\bx)\right) \right| \leq \frac{\epsilon}{K^d}.
        \end{equation*}
        \item (Boundedness) For $\bx \in [0,1]^d$, 
        \begin{equation*}
            \sum_{\bl \in [K]^d} \left| [g(\bx)]_{\bl} \right|\leq \begin{cases}
                2^{d+2} (1+ C_1)^d, &   \textrm{Heaviside-like} \; \phi ,\\
                2^{d+2}, &  \textrm{ReLU-like} \; \phi,
            \end{cases}   
        \end{equation*}
        where $C_1$ is the constant appearing in Assumption~\ref{ass:piece}.
    \end{itemize}
    \label{lem:relative_pos_indicator_app}
\end{lemma}
\begin{proof}
     We prove the result for the Heaviside-like $\phi$, and the proof for the ReLU-like $\phi$ follows in the same manner. Let $\psi_1$ denote the neural network constructed in Lemma~\ref{lem:relative_pos_app} with $\epsilon = \delta / 2$. From the construction of $\psi_1$, we have $\|\theta(\psi_1)\|_{\infty} \lesssim K / \delta^2$. For $\bi, \bj \in [K]^d$
     and $\bx \in \Omega_{\bi, \bj, \mathrm{int}}^{K}(\delta)$, we can deduce the following:
    \begin{equation*}
        \left(\bx - \sum_{\bi \in [K]^d} \ba_{\bi}^K \mathbbm{1}_{\Omega_{\bi}^K }(\bx) \right) \in \Omega_{\bone, \bj, \mathrm{int}}^{K}(\delta),
    \end{equation*}
    which implies that
    \begin{equation*}
        \psi_1(\bx) \in \Omega_{\bone, \bj, \mathrm{int}}^{K} \left( \frac{\delta}{2}\right), \quad \bx \in \Omega_{\bi, \bj, \mathrm{int}}^{K}(\delta) .
    \end{equation*}
    Analogous to the construction in Lemma~\ref{lem:dD_indicator_app}, we construct a neural network $\psi_2$ to approximate the family of indicator functions supported on the refined cells of the bottom-left coarse cell (i.e., $\{\Omega_{\bone, \bj}^K\}_{\bj \in [K]^d}$). For sufficiently small $\epsilon_1 \in (0,1)$ and $\delta \in (0, \frac{1}{3K^2})$, there exists a network $\psi_2 \in \mathcal{H}^{\phi,3}(d, K^d, 2^{d+1} K^d, B_{\epsilon_1, \delta})$ with $B_{\epsilon_1, \delta} \lesssim \max \{1 / (\epsilon_1 \delta), (1/\epsilon_1)^d\}$, such that for all $\bj \in [K]^d$, the following approximation error bound holds:
    \begin{equation*}
        \left| [\psi_2(\bx)]_{\bj} - \mathbbm{1}_{\Omega_{\bone, \bj}^{K}}( \bx)\right| \leq \epsilon_1, \quad \bx \notin \Omega_{\bone, \bj, \mathrm{band}}^{K}\left( \frac{\delta}{2} \right).
    \end{equation*}
    Moreover, the output of $\psi_2$ also satisfies the bound given in \eqref{eq:dD_indicator_bound}.

    Now, define $g = \psi_2 \circ \psi_1$. We then obtain the approximation
    \begin{equation*}
        \left|  [g(\bx)]_{\bj} - \mathbbm{1}_{\Omega_{\bone, \bj}^{K}} \left(\bx - \sum_{\bi \in [K]^d} \ba_{\bi}^K  \mathbbm{1}_{\Omega_{\bi}^K}(\bx)\right)  \right| \leq \epsilon_1, \quad \bx \in \bigcup_{\bi, \bj} \Omega_{\bi, \bj, \mathrm{int}}^K(\delta).
    \end{equation*}
    For a fixed $\bx \in [0,1]^d$, if $\psi_1(\bx) \in \Omega_1^K$, we denote $j$ as the index such that $\psi_1(\bx) \in \Omega_{1,j}^K$. Then due to the boundedness of $\psi_2$, we have the following bound:
    \begin{equation*}
        \begin{aligned}
            \sum_{\bl \in [K]^d} \left| [g(\bx)]_{\bl} \right| & = |g(\bx)|_{\bj} + \sum_{\bl \in [K]^d, \bl \neq \bj} \left|  [\psi_2(\psi_1(\bx))]_{\bl} - \mathbbm{1}_{\Omega_{\bone, \bl}^{K}} \left (\psi_1(\bx) \right) 
            \right| \\
            & \leq 2^{d+1}(1+C_1)^d + K^d \epsilon_1 \leq 2^{d+2}(1+C_1)^d,
        \end{aligned}
    \end{equation*}
    where the last inequality follows from choosing $\epsilon_1 = \epsilon / K^d$ for sufficiently small $\epsilon$. On the other hand, if $\psi_1(\bx) \notin \Omega_1^K$, we can bound the norm similarly by $K^d \epsilon_1 \leq \epsilon$.

    Finally, note that the depth of $\psi_1$ is $2$ and that of $\psi_2$ is $3$, so the depth of composition $g = \psi_2 \circ \psi_1$ is depth $4$. The width of $g$ is bounded by $\max\{6 d K, 2^{d+1} K^d \} \leq 2^{d+2} K^d$ for sufficiently large $K$. Moreover, by the construction of $\psi_1$, the norm of parameters in the output layer of $\psi_1$ is bounded by $O(1/\delta)$, the input layer of $\psi_2$ is bounded by $O(1/(\epsilon_1 \delta))$, therefore the norm of parameters connecting $\psi_1$ and $\psi_2$ are bounded by $O(1/(\epsilon_1 \delta^2))$. Thus, the norm of the network $g$ is bounded by 
    \begin{equation*}
        B_{\epsilon, \delta, K} \lesssim \max \left\{ \|\theta(\psi_1)\|_{\infty}, \|\theta(\psi_2)\|_{\infty}, \frac{1}{\delta^2 \epsilon_1}\right\} \lesssim \max \left\{ \frac{K^d}{\epsilon \delta^2}, \frac{K^{d^2}}{\epsilon^d} \right\}
    \end{equation*}
    as stated in~\eqref{eq:norm_indicator_relative_pos}. This concludes the proof. 
\end{proof}

Finally, by combining the neural network approximations for piecewise constant functions over the coarse grid in Lemma~\ref{lem:dD_pwfc_coarse_app}, with the neural network approximations for indicator functions on the refined grid cells in Lemma~\ref{lem:relative_pos_indicator_app}, we can construct neural network approximations for piecewise constant functions over the refined grid cells, consisting of $K^{2d}$ pieces.
\begin{lemma}
    Let  $\phi$ satisfy Assumptions~\ref{ass:smooth}--\ref{ass:piece}. Let $C(\bx)$ be given in~\eqref{eq:pwcf} and assume that $\max_{\bi, \bj} \{|c_{\bi,\bj}|\} \leq c_{\max}$. Fix $d\in\mathbb{N}_{+}$ and $K\in\mathbb{N}_{+}$ with $K$ sufficiently large.
    Then for any sufficiently small $\epsilon>0$ and any $\delta \in (0, \frac{1}{3K^2})$, there exists a neural network 
     \begin{equation*}
        g \in \mathcal{H}^{\phi,5}(d, 1, 2^{d+3} K^d, B_{\epsilon, \delta, K}),
    \end{equation*}
    with
    \begin{equation}
        B_{\epsilon, \delta, K} \lesssim (c_{\max}+ 1)^{2d} \max \left\{ \frac{K^{d^2}}{\epsilon^d}, \frac{K^{2d}}{\epsilon^{2}}, \frac{K^d}{\epsilon \delta^2} \right\},
        \label{eq:norm_pwcf_refined}
    \end{equation}
   such that the following holds:
    \begin{itemize}
        \item (Approximation) $|g(\bx) - C(\bx)| \leq \epsilon, \quad \bx \in \cup_{\bi,\bj} \Omega^K_{\bi, \bj, \mathrm{int}}(\delta)$.
        \item (Boundedness) For $\bx \in [0,1]^d$,
        \begin{equation*}
            |g(\bx)| \le \begin{cases}
                c_{\max}4^{d+3} (1+ C_1)^{2d}, &   \textrm{Heaviside-like} \; \phi ,\\
                c_{\max}4^{d+3}, &  \textrm{ReLU-like} \; \phi,
            \end{cases}   
        \end{equation*}
        where $C_1$ is the constant appearing in Assumption~\ref{ass:piece}.
        \label{lem:pwcf_refined_app}
    \end{itemize}
\end{lemma}
\begin{proof}
     We prove the result for the Heaviside-like $\phi$, and the proof for the ReLU-like $\phi$ follows in the same manner. Let $\epsilon_1$ be sufficiently small. We begin by considering the neural network constructed in Lemma~\ref{lem:dD_pwfc_coarse_app} with sufficiently small $\epsilon_1$. From the construction of $\psi_1$, we have  $\|\theta(\psi_1)\|_{\infty}\lesssim \max \{c_{\max} K^{d}/(\epsilon_1 \delta), c_{\max}^d K^{d^2} / \epsilon_1^d\}$. By the approximation guaranty for $\psi_1$, we have 
    \begin{equation*}
        \left| [\psi_1(\bx)]_{\bj} - C_{\bj}(\bx)\right| \leq \epsilon_1, \quad \bx \in \cup_{\bi, \bj} \Omega_{\bi, \bj, \mathrm{int}}^{K}(\delta).
    \end{equation*}
    Next, by Corollary~\ref{corol:id-approx}, we obtain $K^d$ subnetworks $\{\psi_{2,\bj}\}_{\bj \in [K]^d}$, each with depth 3 and $\|\theta(\psi_{2,\bj})\|_{\infty}\lesssim 1 / \epsilon_1$ the norm of  parameters in the input layer for each $\psi_{2,\bj}$ is bounded by $O(1)$. These subnetworks satisfy
    \begin{equation*}
        \left| \psi_{2, \bj}(x) - x\right| \leq \epsilon_1, \quad \mathrm{for} \; |x| \leq c_{\max} 2^{d+2}(1+C_1)^d, \quad  \bj \in [K]^d.
    \end{equation*}
    Define the network $\psi_{3}$ by 
    \begin{equation*}
        [\psi_3(\bx)]_{\bj} = \psi_{2, \bj}( [\psi_1(\bx)]_{\bj}),
    \end{equation*}
    so that $\psi_{3}$ has depth $4$ and width $2^{d+2}K^{d}$. The norm of parameters connecting $\psi_1$ and $\psi_{2, \bj}$ is bounded by $\|\theta(\psi_1)\|_{\infty}$, since the since the weights in the input layer of $\psi_{2,\bj}$ are $O(1)$. Therefore, we have the bound for $\|\theta(\psi_3)\|_{\infty}$ 
    \begin{equation*}
        \|\theta(\psi_3)\|_{\infty} \lesssim \max \left\{  \|\theta(\psi_1)\|_{\infty} , \|\theta(\psi_{2,\bj})\|_{\infty} \right\} \lesssim  \max\left\{ \frac{(c_{\max}+1) K^d}{\epsilon_1 \delta}, \frac{c_{\max}^d K^{d^2}}{\epsilon_1^d}\right\}.
    \end{equation*}
    For the constructed $\psi_3$, we have 
    \begin{equation*}
        \left| [\psi_3(\bx)]_{\bj} - C_{\bj}(\bx)\right| \leq 2 \epsilon_1, \quad \bx \in \cup_{\bi,\bj}\Omega_{\bi, \bj,\mathrm{int}}^{K}(\delta)
    \end{equation*}
    and 
    \begin{equation*}
        \left|[\psi_3(\bx)]_{\bj}  \right| \leq |[\psi_1(\bx)]_{\bj}| + \epsilon_1 \leq c_{\max} 2^{d+3} (1+C_1)^d, \quad \bx \in \RR^d, \; \bj \in [K]^d.
    \end{equation*}
    Let $\psi_4$ be the network constructed in Lemma~\ref{lem:relative_pos_indicator_app} with $\epsilon = \epsilon_1$. The depth and width of $\psi_4$ are $4$ and $2^{d+2} K^d$, respectively. Moreover, $\|\theta(\psi_4)\|_{\infty} \lesssim \max \{K^{d} / (\epsilon_1 \delta^2), K^{d^2} / \epsilon_1^d\}$ and the norm of parameters in the output layer of $\psi_4$ is bounded by $O(K^{d^2} / \epsilon_1^d)$.
    For the constructed $\psi_4$, we have 
    \begin{equation*}
        \left| [\psi_4(\bx)]_{\bj} - \tilde{I}_{\bj}(\bx) \right| \leq \frac{\epsilon_1}{K^d}, \quad \bx \in \cup_{\bi, \bj} \Omega_{\bi, \bj, \mathrm{int}}^{K}(\delta),
    \end{equation*}
    where $\tilde{I}_{\bj}$ is the shorthand notation for 
    \begin{equation*}
        \tilde{I}_{\bj}(\bx) = \mathbbm{1}_{\Omega_{\bone, \bj}^{K}}  \left(\bx - \sum_{\bi \in [K]^d} \ba_{\bi}^K  \mathbbm{1}_{\Omega_{\bi}^K}(\bx)\right)
    \end{equation*}
    and 
    \begin{equation*}
        \sum_{\bl \in [K]^d} |[\psi_4(\bx)]_{\bl}| \leq 2^{d+2}(1 +C_1)^d \quad \bx \in \RR^d.
    \end{equation*}
    Again by Lemma~\ref{lemma:all_monomials}, for sufficiently small $\epsilon_2$, there exist $K^d$ subnetworks $\{\psi_{5,\bj}\}_{\bj\in[K]^{d}}$ such that
    \begin{equation*}
        \left| \psi_{5, \bj}(x,y) - xy\right| \leq \epsilon_2, \quad 0\leq |x|, |y| < \max \{c_{\max},1\} 2^{d+3}(1+C_1)^d.
    \end{equation*}
    The norm of $\psi_{5,\bj}$ is bounded by $O\left( \frac{1}{\epsilon_2^2} \right)$, and the norm of the parameters in the input layer of $\psi_{5,\bj}$ is bounded by $O(1)$. Define the final constructed neural network $g$ as 
    \begin{equation*}
        g(\bx) = \sum_{\bj \in [K]^d} \psi_{5, \bj} \left([\psi_3(\bx)]_{\bj}, [\psi_4(\bx)]_{\bj}\right).
    \end{equation*}
    We have the following approximation error bound for constructed $g$
    \begin{equation*}
        \begin{aligned}
            \left| g(\bx) - C(\bx)\right| &  = \left| \sum_{\bj \in [K]^d} \psi_{5, \bj} \left([\psi_3(\bx)]_{\bj}, [\psi_4(\bx)]_{\bj}\right) - \sum_{\bj \in [K]^d} C_{\bj}(\bx) \tilde{I}_{\bj}(\bx) \right| \\
            & \leq \sum_{\bj \in [K]^d} \left| \psi_{5, \bj} \left([\psi_3(\bx)]_{\bj}, [\psi_4(\bx)]_{\bj}\right) - [\psi_3(\bx)]_{\bj} [\psi_4(\bx)]_{\bj}\right| \\
            &+ \sum_{\bj \in [K]^d} \left|[\psi_3(\bx)]_{\bj} [\psi_4(\bx)]_{\bj} - C_{\bj}(\bx) \tilde{I}_{\bj}(\bx) \right| \\
            &\leq K^d \epsilon_2 + \sum_{\bj \in [K]^d} \left| [\psi_4(\bx)]_{\bj}\right| \left| [\psi_3(\bx)]_{\bj} - C_{\bj}(\bx)\right|  + \sum_{\bj \in [K]^d} |C_{\bj}(\bx)| \left| [\psi_4(\bx)]_{\bj} - \tilde{I}_{\bj}(\bx)\right| \\
            & \leq K^d \epsilon_2 + 2 \epsilon_1 \sum_{\bj \in [K]^d} \left| [\psi_4(\bx)]_{\bj}\right| + c_{\max} \epsilon_1 \\
            & \leq K^d \epsilon_2 + \left(c_{\max} + 2^{d+3}(1 + C_1)^d \right) \epsilon_1 \leq \epsilon, \quad \bx \in \cup_{\bi, \bj} \Omega_{\bi, \bj, \mathrm{int}}^{K}(\delta). 
        \end{aligned}
    \end{equation*}
    where we choose 
    \begin{equation*}
        \epsilon_1 = \frac{\epsilon}{2 \left(c_{\max} + 2^{d+3}(1 + C_1)^d \right)}, \quad \epsilon_2 = \frac{\epsilon}{2 K^d}
    \end{equation*}
    in the last inequality. Moreover, we have the uniform bound for the output of $g$ as 
    \begin{equation*}
        \begin{aligned}
            |g(\bx)| &= \left|\sum_{\bj \in [K]^d}  \psi_{5, \bj} \left([\psi_3(\bx)]_{\bj}, [\psi_4(\bx)]_{\bj}\right) - [\psi_3(\bx)]_{\bj} [\psi_4(\bx)]_{\bj} +  [\psi_3(\bx)]_{\bj} [\psi_4(\bx)]_{\bj}\right| \\
            & \leq \sum_{\bj \in [K]^d} \left| \psi_{5, \bj} \left([\psi_3(\bx)]_{\bj}, [\psi_4(\bx)]_{\bj}\right) - [\psi_3(\bx)]_{\bj} [\psi_4(\bx)]_{\bj}\right| + \sum_{\bj \in [K]^d} \left|  [\psi_3(\bx)]_{\bj} [\psi_4(\bx)]_{\bj}\right| \\
            & \leq \sum_{\bj \in [K]^d} \epsilon_2 + \sum_{\bj \in [K]^d } \left( \max_{\bx \in \RR^d} [\psi_3(\bx)]_{\bj}\right)  [\psi_4(\bx)]_{\bj} \\
            &  \leq \frac{\epsilon}{2} + c_{\max} 2^{d+3} (1+C_1)^d \sum_{\bj \in [K]^d } [\psi_4(\bx)]_{\bj} \\
            & \leq \frac{\epsilon}{2} + c_{\max} 2^{2d+5} (1 + C_1)^{ 2d} \leq c_{\max}4^{d+3} (1 + C_1)^{2d}, \quad \bx \in [0,1]^d. 
        \end{aligned}
    \end{equation*}
    Finally, regarding the architecture of $g$, the depth  is $5$, and the width  can be bounded by $2^{d+2}K^d + 2^{d+2} K^d = 2^{d+3} K^d$. The norm of parameters connecting $\psi_3, \psi_4$ to $ \{\psi_{5,\bj}\}_{\bj}$ is bounded by $\max \left\{ \|\theta(\psi_3)\|_{\infty}, \|\theta(\psi_4)\|_{\infty}\right\}$, since the weights in the input layer of $\{\psi_{5,\bj}\}_{\bj}$ are $O(1)$. Therefore, $\|\theta(g)\|_{\infty}$ is bounded by 
    \begin{equation*}
        \begin{aligned}
             B_{\epsilon,\delta, K} & \lesssim \max \left\{ \|\theta(\psi_3)\|_{\infty}, \|\theta(\psi_4)\|_{\infty},  \|\theta(\psi_{5,\bj})\|_{\infty}\right\}\\
             &\lesssim \max\left\{ \frac{(c_{\max}+1) K^d}{\epsilon_1 \delta}, \frac{c_{\max}^d K^{d^2}}{\epsilon_1^d}, \frac{K^d}{\epsilon_1\delta^2}, \frac{K^{d^2}}{\epsilon_1^d}, \frac{1}{\epsilon_2^2} \right\}\\
             &\lesssim (c_{\max}+ 1)^{2d} \max \left\{ \frac{K^{d^2}}{\epsilon^d}, \frac{K^{2d}}{\epsilon^{2}}, \frac{K^d}{\epsilon \delta^2} \right\},
        \end{aligned}
    \end{equation*}
    as stated in~\eqref{eq:norm_pwcf_refined}. This concludes the proof.
\end{proof}

\begin{remark}
    In fact, one can approximate piecewise constant functions with $K^{2d}$ pieces over refined grids by directly approximating all indicator functions on each of the $K^{2d}$ refined grid cells. However, although the number of non-zero parameters in the newly constructed network is $O(K^{2d})$, the width of the network is also $O(K^{2d})$, leading to a total number of parameters that grows as $O(K^{4d})$. This discrepancy between the count of non-zero parameters and the total parameter space imposes an impractical $\ell^0$-sparsity constraint on the newly constructed network.
\end{remark}

\subsection{Approximation of Piecewise Polynomials}
\label{subsec:piecewise-monomials}
By combining the approximations for piecewise constant functions established in Lemma~\ref{lem:pwcf_refined_app}  along with the approximation for monomials derived in Lemma~\ref{lemma:all_monomials}, one can construct neural network architectures for piecewise polynomials.
\begin{lemma}
    Let $\phi$ satisfy Assumptions~\ref{ass:smooth}--\ref{ass:piece}. Let $d \in \mathbb{N}_{+}$ and $\balpha = (\alpha_1, \cdots \alpha_d) \in \mathbb{N}^d$ satisfy $\|\balpha\|_0 \coloneqq  \sum_{j=1}^{d} \alpha_j = m \ge 1$. Let $C(\bx)$ be given in~\eqref{eq:pwcf} and assume that $\max_{\bi, \bj} \{|c_{\bi,\bj}|\} \leq c_{\max}$. Fix $d\in\mathbb{N}_{+}$ and $K\in\mathbb{N}_{+}$ with $K$ sufficiently large.
    Then for any sufficiently small $\epsilon>0$ and any $\delta \in (0, \frac{1}{3K^2})$, there exists a neural network 
    \begin{equation*}
        g \in \mathcal{H}^{\phi,6}(d, 1, 2^{d+4}K^d, B_{\epsilon,\delta, K}),
    \end{equation*}
    with 
    \begin{equation}
        B_{\epsilon,\delta, K} \lesssim (c_{\max}+ 1)^{3d+m} \max \left\{ \frac{K^{d^2}}{\epsilon^d}, \frac{K^{2d}}{\epsilon^{2}}, \frac{K^d}{\epsilon \delta^2}, \frac{1}{\epsilon^m} \right\},
        \label{eq:norm_piece_poly}
    \end{equation}
    such that the following properties hold:
    \begin{itemize}
        \item (Approximation) $\left| g(\bx) - C(\bx) \bx^{\balpha}\right| \leq \epsilon \quad \bx \in \cup_{\bi,\bj} \Omega^K_{\bi, \bj, \mathrm{int}}(\delta).$
        \item (Boundedness) For $\bx \in [0,1)^d$,
        \begin{equation*}
            |g(\bx)| \le \begin{cases}
                c_{\max}4^{d+4} (1+ C_1)^{2d}, &   \textrm{Heaviside-like} \; \phi ,\\
                c_{\max}4^{d+4}, &  \textrm{ReLU-like} \; \phi,
            \end{cases}   
        \end{equation*}
        where $C_1$ is the constant appearing in Assumption~\ref{ass:piece}.
    \end{itemize}
    \label{lem:app_for_piece_poly}
\end{lemma}
\begin{proof}
    We prove the result for the Heaviside-like $\phi$, and the proof for the ReLU-like $\phi$ follows in the same manner. For sufficiently small $\epsilon_1>0$, by Lemma~\ref{lemma:all_monomials}, there exists a neural network $\psi_1$ with depth $2$, width $2^{m}$, and $\|\theta(\psi_1)\|_{\infty}\lesssim (1 / \epsilon_1)^m$ such that 
    \begin{equation*}
        \sup_{\bx \in [-1,1]^d} |\psi_1(\bx) - x_1^{\alpha_1} \cdots x_d^{\alpha_d}| \leq \epsilon_1,
    \end{equation*}
    and $|\psi_1(\bx)| \leq \frac{3}{2}, \bx \in [-1,1]^d$.     
    
    Next, by Corollary~\ref{corol:id-approx}, there exists a neural network $\nu$ with depth 4, width 2, and parameter norm $\|\theta(\nu)\|_{\infty} \lesssim 1/\epsilon_1$, such that
    \begin{equation*}
        \sup_{x \in [-\frac{3}{2}, \frac{3}{2}]} |\nu(x) - x| \leq \epsilon_1.
    \end{equation*}
    Additionally, $|\nu(x)| < 2$ for $-\frac{3}{2} < x < \frac{3}{2}$, and the parameter norms in the input and output layers are bounded by $O(1)$ and $O(1/\epsilon_1)$, respectively.

    We now define the neural network $\psi_2$ as the composition $\psi_2 \coloneqq  \nu \circ \psi_1$.
    By the properties of the network composition, we have the following approximation bound
    \begin{equation*}
        \left| \psi_2(\bx) - x_1^{\alpha_1} \cdots x_d^{\alpha_d}\right| \leq 5 \epsilon_1, \quad \bx \in [-1,1]^d,
    \end{equation*}
    and $|\psi_2(\bx)| \leq 2, \bx \in [-1,1]^d$. Considering the architecture for $\psi_2$, its depth and width are $5$ and $2^m$, respectively, and $\|\theta(\psi_2)\|_{\infty} \lesssim (1 / \epsilon_1)^m$. Moreover, the norm of the parameters in the output layer of $\psi_2$ is bounded by $O(1 / \epsilon_1)$.

    Let $\psi_3$ be the neural network constructed in Lemma~\ref{lem:pwcf_refined_app} with $\epsilon = \epsilon_1$. We have the approximation bound
    \begin{equation*}
        |\psi_3(\bx) - C(\bx)| \leq \epsilon_1, \quad \bx \in \cup_{\bi, \bj} \Omega_{\bi, \bj, \mathrm{int}}^{K}(\delta).
    \end{equation*}
    Furthermore, by Lemma~\ref{lemma:all_monomials}, there exists a neural network $\psi_4$  with depth $2$, width $4$ and $\|\theta(\psi_4)\|_{\infty} \lesssim (1 / \epsilon_1)^2$, such that 
    \begin{equation*}
        |\psi_4(x, y) - x y | \leq \epsilon_1, \quad 0 \leq |x|, |y| < c_{\max} 4^{d+3} (1+C_1)^{2d}.
    \end{equation*}
    Moreover, the norm for parameters in the input layer for $\psi_4$ is bounded by $O(1)$.

    Finally, we define the constructed neural network $g$ as 
    \begin{equation*}
        g(\bx) = \psi_4(\psi_2(\bx), \psi_3(\bx)), \quad \bx \in [0,1)^d,
    \end{equation*}
    then we obtain the approximation bound for $\bx \in \cup_{\bi, \bj} \Omega_{\bi, \bj, \mathrm{int}}^K(\delta)$
    \begin{equation*}
        \begin{aligned}
            \left|g(\bx)  - C(\bx) \bx^{\balpha}\right| &\leq \left| \psi_4(\psi_2(\bx), \psi_3(\bx)) - \psi_2(\bx) \psi_3(\bx)\right| + |\psi_2(\bx) \psi_3(\bx) - C(\bx) \bx^{\balpha}| \\
            & \leq \epsilon_1 + |\psi_2(\bx)||\psi_3(\bx) - C(\bx)| + |C(\bx)||\psi_2(\bx) - \bx^{\balpha}| \\
            & \leq \epsilon_1 + 2\epsilon_1 + 5c_{\max} \epsilon_1 \leq \epsilon, 
        \end{aligned}
    \end{equation*}
    where we choose $\epsilon_1 = \frac{\epsilon}{3 + 5 c_{\max}}$ in the last inequality. Moreover, we can bound $g(\bx)$ for $\bx \in [0,1)^d$ as follows:
    \begin{equation*}
        \begin{aligned}
            |g(\bx)| & \leq \left| \psi_4(\psi_2(\bx), \psi_3(\bx)) - \psi_2(\bx) \psi_3(\bx)\right| + |\psi_2(\bx) \psi_3(\bx) | \\
            & \leq \epsilon_1 + 2 \times c_{\max} 4^{d+3}(1+C_1)^{2d} \leq c_{\max} 4^{d+4}(1+C_1)^{2d}. 
        \end{aligned}
    \end{equation*}

    Finally, the depth of $g$ is $6$, and its width can be bounded by $2^{d+3}K^d + \max \{2^{m}, 4\} \leq 2^{d+4} K^d$. The norm of parameters connecting $\psi_2, \psi_3, \psi_4$ is bounded by $\max \left\{ \|\theta(\psi_3)\|_{\infty}, \|\theta(\psi_4)\|_{\infty}\right\}$, since the weights in the input layer of $\psi_4$ are $O(1)$. Therefore, $\|\theta(g)\|_{\infty}$ can be bounded by 
    \begin{equation*}
        \begin{aligned}
            B_{\epsilon,\delta, K} & \lesssim \max \left\{ \|\theta(\psi_2)\|_{\infty}, \|\theta(\psi_3)\|_{\infty}, \|\theta(\psi_4)\|_{\infty} \right\}\\
            &\lesssim (c_{\max}+ 1)^{3d+m} \max \left\{ \frac{K^{d^2}}{\epsilon^d}, \frac{K^{2d}}{\epsilon^{2}}, \frac{K^d}{\epsilon \delta^2}, \frac{1}{\epsilon^m} \right\},
        \end{aligned}
    \end{equation*}
    as stated in~\eqref{eq:norm_piece_poly}. This concludes the proof.
\end{proof}

\subsection{Proof of Theorem~\ref{thm:app_L2} (\texorpdfstring{$L^{2}$}{L-2} Approximation)}
\label{subsec:app_f_star}
By combining the piecewise polynomial approximation in Lemma~\ref{lem:piece_poly_app} with the neural network constructions in Lemma~\ref{lem:app_for_piece_poly}, we obtain the following approximation guaranty for   $f^{\star} \in W^{s, \infty}([0,1]^d)$.
\begin{theorem}
    Let $\phi$ satisfy Assumptions~\ref{ass:smooth}--\ref{ass:piece}. For any $s>0$ and any $f^{\star} \in W^{s,\infty}([0,1]^d)$ with
$\|f^{\star}\|_{W^{s,\infty}([0,1])^d}\le 1$. Let $\epsilon \in(0,1)$ be sufficiently small, and we define $K \coloneqq \lceil (2 c_1(s,d) /\epsilon)^{1/2s} \rceil$. Then, for any $\delta \in (0, \frac{1}{3K^2})$, there exists a neural network
    \begin{equation*}
        g \in \mathcal{H}^{\phi,6}(d, 1, M_{\epsilon}, B_{\epsilon,\delta}),
    \end{equation*}
    with 
    \begin{equation}
         M_\epsilon \lesssim \left( \frac{1}{\epsilon}\right)^{\frac{d}{2s}},
         \quad  B_{\epsilon, \delta} \lesssim \max \left\{ \frac{1}{\epsilon^{\max \left\{ \frac{d^2}{2s}+d, \frac{d}{s}+2, \lceil s\rceil \right\} }} ,\frac{1}{\delta^2 \epsilon^{\frac{d}{2s}+1}}\right\},
        \label{eq:norm_width_app_infinity_local}
    \end{equation}
    such that the following properties hold:
    \begin{itemize}
        \item (Approximation) \[
        \|g - f^{\star}\|_{L^\infty\left(\bigcup_{\bi,\bj} \Omega^K_{\bi, \bj, \mathrm{int}}(\delta)\right)} \le \epsilon.
        \]
        \item (Boundedness) For $\bx \in [0,1)^d$,
        \begin{equation}
            |g(\bx)| \le \begin{cases}
                \lceil s\rceil^d c_2(s,d) 4^{d+4}(1+C_1)^{2d}, &   \textrm{Heaviside-like} \; \phi ,\\
                \lceil s\rceil^d c_2(s,d) 4^{d+4}, &  \textrm{ReLU-like} \; \phi,
            \end{cases}  
            \label{eq:upper_bound_app_local_infty}
        \end{equation}
    \end{itemize}
\label{thm:app_infinite_local}
\end{theorem}

\begin{proof}
    We prove the result for the Heaviside-like $\phi$, and the proof for the ReLU-like $\phi$ follows in the same manner.
    By Lemma~\ref{lem:piece_poly_app}, for $f^{\star}$, there exists a piecewise polynomial $p = \sum_{|\balpha| < \lceil s\rceil} p_{\balpha}$ defined over partition $\{\Omega_{\bi,\bj}^K\}_{\bi, \bj \in [K]^d}$ such that 
     \begin{equation*}
         \| p -f^{\star}\|_{L^{\infty}([0,1]^d)} \leq c_1(s,d) \|f^{\star}\|_{W^{s,\infty}([0,1]^d)} K^{-2s},
     \end{equation*}
     with the magnitudes of the coefficients in each piecewise monomial $p_{{\balpha}}$ bounded, i.e. $c_{\max} \leq c_2(s,d) \|f^{\star}\|_{W^{s,\infty}([0,1]^d)} \leq c_2(s,d)$. Next, by Lemma~\ref{lem:app_for_piece_poly}, for sufficiently small $\epsilon_1 > 0 $,
     there exists a set of neural networks $\{ \psi_{\balpha}\}_{|\balpha |< \lceil s\rceil}$ such that 
     \begin{equation*}
         |\psi_{\balpha}(\bx) - p_{\balpha}(\bx)| < \epsilon_1, \quad \bx \in \cup_{\bi, \bj} \Omega_{\bi, \bj, \mathrm{int}}^K(\delta),
     \end{equation*}
     with bounded output:
     \begin{equation*}
        |\psi_{\balpha}(\bx)| < c_2(s, d) 4^{d+4}(1+C_1)^{2d}
     \end{equation*}
     Define the neural network $g = \sum_{|\balpha|< \lceil s \rceil} \psi_{\balpha}$. Then we can bound the $L^\infty$ error for $g$ approximating $f^{\star}$ as follows:
     \begin{equation*}
         \begin{aligned}
             \left\| g - f^{\star} \right\|_{L^\infty\left(\bigcup_{\bi,\bj} \Omega^K_{\bi, \bj, \mathrm{int}}(\delta)\right)} & \leq \left\| p - f^{\star} \right\|_{L^\infty\left(\bigcup_{\bi,\bj} \Omega^K_{\bi, \bj, \mathrm{int}}(\delta)\right)} + \left\| g - p \right\|_{L^\infty\left(\bigcup_{\bi,\bj} \Omega^K_{\bi, \bj, \mathrm{int}}(\delta)\right)} \\
             & \leq \|p - f^{\star}\|_{L^{\infty}([0,1]^d)} + \sum_{|\balpha|< \lceil s\rceil} \left\| \psi_{\balpha} - p_{\balpha}\right\|_{L^\infty\left(\bigcup_{\bi,\bj} \Omega^K_{\bi, \bj, \mathrm{int}}(\delta)\right)}\\
             & \leq c_1(s,d) K^{-2s} + \lceil s \rceil^d \epsilon_1,
         \end{aligned}
     \end{equation*}
     Set 
     \begin{equation*}
         K = \left\lceil \left( \frac{2 c_1(s,d)}{\epsilon}\right)^{\frac{1}{2s}} \right\rceil, \quad \epsilon_1 = \frac{\epsilon}{2 \lceil s \rceil^d}.
    \end{equation*}
    Then we obtain
    \begin{equation*}
        \left\| g - f^{\star} \right\|_{L^\infty\left(\bigcup_{\bi,\bj} \Omega^K_{\bi, \bj, \mathrm{int}}(\delta)\right)} \leq \epsilon.
    \end{equation*} 
    And the output for $g$ is bounded by:
    \begin{equation*}
        |g(\bx)| \leq \sum_{|\balpha|< \lceil s\rceil} |\psi_{\balpha}(\bx)| \leq \lceil s \rceil^d c_2(s,d) 4^{d+4}(1 + C_1)^{2d}, \quad \bx \in [0,1)^d.
    \end{equation*}
    And the width of $g$ is bounded by 
    \begin{equation*}
        M_\epsilon = \lceil s\rceil^d 2^{d+4} K^{d} \lesssim \left( \frac{1}{\epsilon} \right)^{\frac{d}{2s}}, 
    \end{equation*}
    and the parameter norm of the network $g$ is bounded by 
    \begin{equation*}
        B_{\epsilon, \delta} \lesssim \max \left\{ \frac{K^{d^2}}{\epsilon_1^d}, \frac{K^{2d}}{\epsilon_1^2}, \frac{K^d}{\epsilon_1 \delta^2}, \max_{|\balpha | < \lceil s\rceil} \frac{1}{\epsilon_1^{|\balpha|}}\right\} \lesssim \max \left\{ \frac{1}{\epsilon^{\max \left\{ \frac{d^2}{2s}+d, \frac{d}{s}+2, \lceil s\rceil \right\} }} ,\frac{1}{\delta^2 \epsilon^{\frac{d}{2s}+1}}\right\},
    \end{equation*}
    as stated in ~\eqref{eq:norm_width_app_L2}. This concludes the proof.
\end{proof}

With the approximation of $f^{\star}$ by $g$ established on $\bigcup_{\bi,\bj\in[K]^d}\Omega^K_{\bi,\bj,\mathrm{int}}(\delta)$ and uniform boundedness ensured on $[0,1)^d$ in Theorem~\ref{thm:app_infinite_local}, we now proceed to the proof of Theorem~\ref{thm:app_L2}.
\begin{proof}[Proof of Theorem~\ref{thm:app_L2}]
    It suffices to prove for sufficiently small $\epsilon < {\epsilon}_{0}$, where $\epsilon_0$ depends on $s$, $d$ and $\phi$. 
    We prove the result for the Heaviside-like $\phi$, and the proof for the ReLU-like $\phi$ follows in the same manner. By Theorem~\ref{thm:app_infinite_local}, for any sufficiently small $\epsilon_1>0$, with $K \coloneqq  \lceil (2 c_1(s,d)/\epsilon_1)^{1/2s} \rceil$ and $\delta \in (0,1/(3K^2))$, there exists a neural network $g$ such that
     \begin{equation}
         \|g - f^{\star}\|_{L^\infty\left(\bigcup_{\bi,\bj} \Omega^K_{\bi, \bj, \mathrm{int}}(\delta)\right)} \le \epsilon_1.
         \label{eq:app_proof}
     \end{equation}
     and, for all $\bx\in[0,1)^d$,
     \begin{equation*}
         |g(\bx)| \leq \lceil s\rceil^d c_2(s,d) 4^{d+4}(1+C_1)^{2d}, 
         \label{eq:bound_proof}
     \end{equation*}
     We estimate the $L^2([0,1]^d)$ approximation  by decomposing the domain into interior ($\cup_{\bi,\bj} \Omega_{\bi,\bj, \mathrm{int}}^K(\delta)$) and boundary regions ($\cup_{\bi,\bj} \Omega_{\bi,\bj, \mathrm{band}}^K(\delta)$):
     \begin{equation*}
         \begin{aligned}
             \|g - f^{\star}\|_{L^2([0,1]^d)} &\leq \underbrace{ \|g - f^{\star}\|_{L^2\left(\bigcup_{\bi,\bj} \Omega^K_{\bi, \bj, \mathrm{int}}(\delta)\right)}}_{(a)} 
             +   \underbrace{\|g - f^{\star}\|_{L^2\left(\bigcup_{\bi,\bj} \Omega^K_{\bi, \bj, \mathrm{band}}(\delta)\right)}}_{(b)}
              \\
             & 
             \leq \epsilon_1 +   \lceil s\rceil^d c_2(s,d)  4^{d+4}(1+C_1)^{2d} \times \sqrt{2d (K^{-2})^{d-1} \delta \times K^{2d}} \\
             & = \epsilon_1 + \sqrt{2d}  c_2(s,d) 4^{d+4}(1+C_1)^{2d} K\sqrt{\delta} .
         \end{aligned}
     \end{equation*}
     Here, term~(a) is controlled by the approximation guarantee in~\eqref{eq:app_proof}, while term~(b) follows from the uniform boundedness~\eqref{eq:bound_proof} together with the measure estimate of the boundary region.

    We now choose
     \begin{equation*}
         \epsilon_1 = \frac{\epsilon}{2}, \quad \delta = \frac{\epsilon^2}{2^{4d+19} c_2^2(s,d) d (1 + C_1)^{4d}K^2},
     \end{equation*}
     which yields
     \begin{equation*}
         \|g - f^{\star}\|_{L^2([0,1]^d)} \leq \epsilon.
     \end{equation*}

     Finally, by the construction in Theorem~\ref{thm:app_infinite_local}, the network width is bounded by $O((1/ \epsilon)^{\frac{d}{2s}})$ and the parameter norm $\|\theta(g)\|_{\infty}$ is bounded by:
     \begin{equation*}
         B_{\epsilon} \lesssim \max \left\{ \frac{1}{\epsilon_1^{\max \left\{ \frac{d^2}{2s}+d, \frac{d}{s}+2, \lceil s\rceil \right\} }} ,\frac{1}{\delta^2 \epsilon_1^{\frac{d}{2s}+1}}\right\} \lesssim \left( \frac{1}{\epsilon}\right)^{\max \left\{\frac{d^2}{2s}+d, \frac{d}{s}+2, \frac{d+4}{2s}+5, \lceil s \rceil\right\}},
     \end{equation*}
    as stated in~\eqref{eq:norm_width_app_L2}. This concludes the proof.
\end{proof}

\subsection{Approximation of Weight Functions}
In the remaining part of this section, we strengthen the approximation result of Theorem~\ref{thm:app_L2} by improving the error guarantee from the $L^2([0,1]^d)$ to $L^{\infty}([0,1]^d)$. The constructed networks given in  Theorem~\ref{thm:app_L2} fail to achieve an uniform $L^{\infty}([0,1]^d)$ approximation on the region $\cup_{\bi, \bj} \Omega_{\bi, \bj, \mathrm{band}}^K(\delta)$. The obstruction stems from the fact that indicator functions associated with this region cannot be uniformly approximated, as shown in Lemma~\ref{lem:dD_indicator_app}. To circumvent this issue, we introduce a weight function that assigns negligible mass to the associated region, thereby suppressing the approximation error.

We begin by formally defining the basis functions employed in the construction of the weight functions, which are tailored to the activation function classes as specified in Assumption~\ref{ass:piece}.
\begin{definition}[Basis function]
    Let $\phi$ satisfy Assumption~\ref{ass:piece}. For any $K \in \mathbb{N}_{+}$ and any $\delta \in \left(0, \frac{1}{12K^2} \right)$, and any $\beta > 0$, we define the basis function $B_{\phi}^{\beta, \delta,K}$ as follows:
    \begin{itemize}
        \item If $\phi$ is Heaviside-like, then:
        \begin{equation*}
            B_{\phi}^{\beta, \delta, K}(x) = \phi \left( \beta ( x - 3 \delta) \right) - \phi \left( \beta (x + 3 \delta - K^{-2})\right).
        \end{equation*}
        \item If $\phi$ is ReLU-like, then:
        \begin{equation*}
            \begin{aligned}
                B_{\phi}^{\beta, \delta, K}(x) = \frac{1}{2\delta \beta} \Big[ \phi\left( \beta(x - 2 \delta) \right) &- \phi \left( \beta(x - 4\delta)\right) \\ 
                &- \phi \left(\beta(x - K^{-2}+4\delta) \right) +  \phi \left(\beta(x - K^{-2}+2\delta) \right)\Big].
            \end{aligned}
        \end{equation*}
    \end{itemize}
    \label{def:basis}
\end{definition}
Building upon the basis functions established in Definition~\ref{def:basis}, we proceed to define the univariate weight functions as follows.
\begin{definition}[Univariate weight function]
     Let $\phi$ satisfy Assumption~\ref{ass:piece}. For any $K \in \mathbb{N}_{+}$ and any $\delta \in \left(0, \frac{1}{12 K^2 }\right) $ and any $\beta>0$, we define the primary weight function $w^{\beta, \delta, K}_{\phi,1}: [0,1] \to \mathbb{R}$  piecewise via translation:
    \begin{equation}
        w^{\beta, \delta, K}_{\phi,1}(x) \coloneqq B_{\phi}^{\beta, \delta, K}\left( x - \frac{k}{K^2}\right), \quad  x \in \left[ \frac{k}{K^2}, \frac{k+1}{K^2}\right),
        \label{eq:w_1}
    \end{equation}
    where $k = 0, 1, \cdots, K^{2}-1$.

    The complementary weight function $w^{\beta, \delta, K}_{\phi, 2}$ is defined as:
   \begin{equation*}
        w^{\beta, \delta, K}_{\phi, 2} (x) \coloneqq 1 - w^{\beta, \delta, K}_{\phi,1}(x), \quad  x \in [0,1).
        \label{eq:w_2}
    \end{equation*}
    Finally, we impose periodic boundary conditions such that $w^{\beta, \delta, K}_{\phi,i}(1) \coloneqq w_{\phi,i}^{\beta, \delta, K}(0)$ for $i \in \{1,2\}$.
    \label{def:univariate_weight}
\end{definition}
The univariate weight functions exhibit the following properties.
\begin{proposition}
    Let $w_{\phi,1}^{\beta, \delta, K}$ and $w_{\phi, 2}^{\beta, \delta, K}$ be the 1D weight functions defined in Definition~\ref{def:univariate_weight}. For any $\epsilon \in (0,1)$ and any $\beta \ge \frac{2}{\delta \epsilon} \max \left\{ C_1,C_2 \right\}$, the following properties hold:
    \begin{itemize}
        \item (Partition of unity) $w_{\phi,1}^{\beta, \delta, K} (x) + w_{\phi,2}^{\beta, \delta, K}(x) = 1, \quad x\in [0,1].$
        \item (Locally quasi-vanishing behavior) The weight functions are effectively supported away from band regions:
        \begin{equation*}
            |w_{\phi,i}^{\beta, \delta, K}(x)| \leq \epsilon, \quad x \in \Omega_{\mathrm{band}}^{K,i}(2 \delta), \quad i=1,2,
        \end{equation*}
        where the band regions \( \Omega_{\mathrm{band}}^{K,i}(2\delta) \) are defined as:
        \begin{equation*}
            \begin{aligned}
                & \Omega_{\mathrm{band}}^{K,1}(2\delta)\coloneqq  \bigcup_{k=0}^{K^2} \left[ -2 \delta + \frac{k}{K^2}, 2\delta + \frac{k}{K^2}\right] \bigcap [0,1], \\
                &\Omega_{\mathrm{band}}^{K,2}(2\delta)\coloneqq  \bigcup_{k=0}^{K^2-1} \left[ \frac{2k+1}{2K^2} - 2\delta, \frac{2k+1}{2K^2}+ 2\delta \right].
            \end{aligned}
        \end{equation*}
        \item (Boundedness) $\max \left\{ \|w_{\phi,1}^{\beta,\delta,K}\|_{L^{\infty}([0,1])} , \|w_{\phi,2}^{\beta,\delta,K}\|_{L^{\infty}([0,1])}\right\} \leq 2C_1 +3$.
    \end{itemize}
    \label{prop:univariate_weight}
\end{proposition}

\begin{proof}
    By Definition~\ref{def:univariate_weight}, the weight functions $w_{\phi,1}^{\beta,\delta,K}$ and $w_{\phi,2}^{\beta,\delta,K}$ are periodic with period $T = K^{-2}$. It therefore suffices to verify the stated properties on the interval $[0,K^{-2})$. Moreover, the partition-of-unity property follows directly from Definition~\ref{def:univariate_weight}. We then prove the remaining two properties. We first consider Heaviside-like activations and then deal with the ReLU-like case.
    \begin{itemize}
        \item \textbf{(Heaviside-like)} Let $\chi$ be the indicator function defined by
        \begin{equation*}
            \chi(x)\coloneqq H(\beta(x - 3\delta)) - H(\beta(x + 3 \delta - K^{-2}))
        \end{equation*}
        We have the following properties.
        \begin{itemize}
            \item For $x \in [0,2\delta] \cup [\frac{1}{K^2}-2\delta,\frac{1}{K^2}]$, $\chi(x) = 0$.
            \item For $x \in [\frac{1}{2K^2} - 2\delta, \frac{1}{2K^2} + 2 \delta]$, $\chi(x) = 1$.
        \end{itemize}
        On the interval $ I\coloneqq [0, 2\delta] \cup [\frac{1}{2K^2}- 2\delta, \frac{1}{2K^2} + 2 \delta]\cup [\frac{1}{K^2} - 2\delta, \frac{1}{K^2}]$,  the deviation between $w_{\phi,1}^{\beta,\delta,K}(x)$ and $\chi(x)$ is bounded by:
        \begin{equation*}
            \begin{aligned}
                |w_{\phi,1}^{\beta,\delta, K}(x) -\chi(x)| & \leq \left| (\phi -H)(\beta (x - 3 \delta)) \right| + |(\phi - H)(\beta (x + 3 \delta - K^{-2}) )| \\
                & \overset{(a)}{\leq} \frac{C_1}{ \beta}\left( {|x - 3 \delta|}^{-1} + {|x + 3\delta - K^{-2}|}^{-1}\right)  \overset{(b)}{\leq}\frac{2C_1}{\beta \delta} \overset{(c)}{\leq} \epsilon,
            \end{aligned}
        \end{equation*}
        where $(a)$ follows from the Heaviside-like assumption in \eqref{eq:Heaviside-like}, $(b)$ follows from the fact that $x \in I$ ensures the distance from $x$ to $3 \delta$ and $K^{-2}- 3\delta$ is at least $\delta$, and $(c)$ follows from $\beta \geq \frac{2 C_1}{\delta \epsilon}$.

        The locally quasi-vanishing properties for weight functions follow directly:
        \begin{equation*}
            \begin{aligned}
                &|w_{\phi,1}^{\beta, \delta, K}(x)| =|w_{\phi,1}^{\beta, \delta, K}(x) - \chi(x)| \leq \epsilon, \quad x \in [0, 2\delta] \cup \left[\frac{1}{K^2}- 2\delta, \frac{1}{K^2} \right], \\
                &|w_{\phi,2}^{\beta, \delta, K}(x)| = |1 - w_{\phi,1}^{\beta, \delta, K}(x)| = |\chi(x) - w_{\phi,1}^{\beta, \delta, K}(x)| \leq \epsilon, \quad x \in \left[\frac{1}{2K^2} - 2\delta, \frac{1}{2K^2}+2 \delta \right].
            \end{aligned}
        \end{equation*}
        Moreover, by the Heaviside-like assumption in \eqref{eq:Heaviside-like} and $\|\chi\|_{L^{\infty}([0,1])} \leq 1 $, we obtain:
        \begin{equation*}
            \|w_{\phi,1}^{\beta, \delta, K}\|_{L^{\infty}([0,1])} \leq 2C_1 + 1, \quad  \|w_{\phi,2}^{\beta, \delta, K}\|_{L^{\infty}([0,1])} \leq 1 +   \|w_{\phi,1}^{\beta, \delta, K}\|_{L^{\infty}([0,1])} \leq 2C_1 +2.
        \end{equation*}

        \item \textbf{(ReLU-like)} We define the nominal trapezoidal profile $g$ via:
        \begin{equation*}
            \begin{aligned}
                \chi(x)\coloneqq  \frac{1}{2\delta \beta} \Big[ \mathrm{ReLU}(\beta(x - 2 \delta)) &- \mathrm{ReLU}(\beta(x - 4 \delta)) \\
                &- \mathrm{ReLU}(\beta(x -K^{-2} + 4 \delta)) +  \mathrm{ReLU}(\beta(x -K^{-2} + 2\delta)) \Big]. 
            \end{aligned}
        \end{equation*}
        Similar to the Heaviside case, the function $g$ satisfies
        \begin{itemize}
            \item For $x \in [0,2\delta] \cup [\frac{1}{K^2}-2\delta,\frac{1}{K^2}]$, $\chi(x) = 0$.
            \item For $x \in [\frac{1}{2K^2} - 2\delta, \frac{1}{2K^2} + 2 \delta]$, $\chi(x) = 1$.
        \end{itemize}
        For $x \in \mathbb{R}$, the deviation between $w_{\phi,1}^{\beta, \delta, K}(x)$ and $\chi(x)$ is bounded as:
        \begin{equation}
            \begin{aligned}
                 \left|w_{\phi,1}^{\beta, \delta, K}(x) - \chi(x) \right| & \leq  \frac{1}{2 \delta \beta}\sum_{j=1}^2|\phi(\beta (x - 2j\delta)) - \mathrm{ReLU}(\beta(x - 2j \delta))|  \\
                 & + \frac{1}{2 \delta \beta} \sum_{j=1}^2 |\phi(\beta (x - K^{-2}+ 2j\delta)) - \mathrm{ReLU}(\beta(x - K^{-2}  +2 j \delta ))| \\
                 & \overset{(a)}{\leq} \frac{2 C_2}{\delta \beta} \overset{(b)}{\leq}  \epsilon,
            \end{aligned}
            \label{eq:error_w1_g}
        \end{equation}
        where $(a)$ follows from the ReLU-like assumption in \eqref{eq:ReLU-like} and $(b)$ follows from $ \beta \geq \frac{2C_2}{\delta\beta} $. The locally quasi-vanishing properties follow the same logic as the Heaviside-like case.

         Regarding boundedness, since $\| \chi \|_{L^{\infty}([0,1])} \leq 1$, and combining the error bound in \eqref{eq:error_w1_g}, we establish the boundedness for the weight functions:
        \begin{equation*}
            \|w_{\phi,1}^{\beta, \delta,K}\|_{L^{\infty}([0,1])} \leq 1 + \epsilon < 2, \quad \|w_{\phi,2}^{\beta, \delta,K}\|_{L^{\infty}([0,1])} \leq 1 + \|w_{\phi,1}^{\beta, \delta,K}\|_{L^{\infty}([0,1])}<3.
        \end{equation*}
    \end{itemize}
\end{proof}

We now generalize the weight functions to arbitrary dimensions using a tensor product approach.

\begin{definition}[Multivariate weight functions]
    Let the parameters $\beta, \delta, K$ and the activation function $\phi$ satisfy the same conditions given in Definition~\ref{def:univariate_weight}. For $d \in \mathbb{N}_{+}$ and a multi-index $\bv = (v_1, \dots, v_d) \in [2]^d$, we define the $d$-variate weight function $w_{\phi, \bv}^{\beta, \delta, K}: [0,1]^d \to \RR$ as:
    \begin{equation*}
        w_{\phi, \bv}^{\beta, \delta, K}(\bx): = \prod_{l=1}^d  w_{\phi, v_l}^{\beta, \delta, K}(x_l).
    \end{equation*}
    \label{def:multi_weight}
\end{definition}
The multivariate weight functions satisfy the following properties.
\begin{proposition}
    Let $w_{\phi, \bv}^{\beta, \delta, K}$ denote the $d$-variate weight functions defined in Definition ~\ref{def:multi_weight}. For any sufficiently small $\epsilon \in (0,1) $ and any $\beta \ge \frac{2(2C_1+3)^{d-1}}{\delta \epsilon} \max\{C_1,C_2\}$, the following properties hold:
    \begin{itemize}
        \item (Partition of unity) $\sum_{\bv \in [2]^d} w_{\phi,\bv}^{\beta, \delta, K}(\bx) = 1, \quad \bx \in [0,1]^d$.
        \item (Locally quasi-vanishing behavior) The weight function is effectively supported away from the band region $\Omega_{\mathrm{band}}^{K, \bv}(2 \delta)$:
        \begin{equation*}
            |w_{\phi, \bv}^{\beta, \delta, K}(\bx)| \leq \epsilon, \quad \bx \in \Omega_{\mathrm{band}}^{K, \bv}(2 \delta).
        \end{equation*}
        \item (Boundedness) $ \left\|w_{\phi, \bv}^{\beta, \delta, K}\right\|_{L^{\infty}([0,1]^d)} \leq (2 C_1 + 3)^d$.
    \end{itemize}
    \label{prop:multi_weight}
\end{proposition}

\begin{proof}
    We now establish the three properties: 
    \begin{itemize}
        \item \textbf{(Partition of unity)} Exploiting the tensor product structure, the summation over the multi-index $\bv \in [2]^d$ factorizes into a product of univariate sums:
        \begin{equation*}
            \begin{aligned}
                \sum_{\bv \in [2]^d} w_{\phi,\bv}^{\beta, \delta, K}(\bx) & = \sum_{v_1 \in [2]} w_{\phi, v_1}^{\beta, \delta, K}(x_1) \left( \sum_{v_2 \in [2]} w_{\phi, v_2}^{\beta, \delta, K}(x_2) \left( \cdots \left(\sum_{v_d \in [2]} w_{\phi, v_d}^{\beta, \delta, K}(x_d)  \right)\right) \right) \\
                & = \prod_{l=1}^d \left(\sum_{v_l \in [2]} w_{\phi, v_l}^{\beta, \delta, K}(x_l)\right) = 1, \qquad \bx \in [0,1]^d,
            \end{aligned}
        \end{equation*}
        where we use the partition of unity property for univariate weight functions $w_{\phi, 1}^{\beta, \delta, K}(x) + w_{\phi, 2}^{\beta, \delta, K}(x) = 1$ from Proposition~\ref{prop:univariate_weight}.
        \item \textbf{(Locally quasi-vanishing behavior)} Let $\epsilon_1 \in (0,1)$ and assume $\beta \ge \frac{2}{\delta \epsilon_1} \max \left\{ C_1,C_2 \right\}$. Consider an arbitrary point $\bx \in \Omega_{\mathrm{band}}^{K,\bv}(2 \delta)$. By the definition of the multivariate band region, there exists at least one coordinate index $l \in [d]$ such that $x_l \in \Omega_{\mathrm{band}}^{K, v_l}(2 \delta)$. Thus, we obtain:
        \begin{equation*}
            \left| w_{\phi,\bv}^{\beta, \delta, K}(\bx)\right| = |w_{\phi,v_l}^{\beta, \delta, K }(x_l)| \prod_{k \neq l} |w_{\phi,v_k}^{\beta, \delta, K }(x_k)| \overset{(a)}{\leq} (2C_1+3)^{d-1} \epsilon_1 \overset{(b)}{\leq} \epsilon, \quad \bx \in \Omega_{\mathrm{band}}^{K, \bv}(2 \delta),
        \end{equation*}
        where $(a)$ follows from the locally quasi-vanishing behavior of the univariate weight function  $w_{\phi, v_l}^{\beta, \delta, K}$ and the boundedness of the remaining $d-1$ univariate weight functions as given in Proposition~\ref{prop:univariate_weight}, and $(b)$ is obtained by setting $\epsilon_1 = \epsilon / (2C_1 + 3)^{d-1}$ and choosing $\beta \ge \frac{2(2C_1+3)^{d-1}}{\delta \epsilon} \max\{C_1,C_2\}$.
        \item \textbf{(Boundedness)} Leveraging the univariate bounds established in Proposition~\ref{prop:univariate_weight}, we deduce the uniform bound for multivariate weight functions:
        \begin{equation*}
            \left\|w_{\phi, \bv}^{\beta, \delta, K}\right\|_{L^{\infty}([0,1]^d)} \leq \prod_{l=1}^d \left\| w_{\phi. v_{l}}^{\beta, \delta, K}\right\|_{L^{\infty}([0,1])}\leq(2 C_1 + 3)^d.
        \end{equation*}
    \end{itemize}
\end{proof}

We now construct neural network approximators for the univariate weight functions. Noticing that $w_{\phi,1}^{\beta, \delta, K}$ and $w_{\phi,2}^{\beta, \delta, K}$ exhibit $K^2$ periods on $[0,1]$, a straightforward construction uses a shallow neural network of $O(K^2)$ width. However, for the case $d=1$, this $O(K^2)$ width dominates the width scaling of $O(K)$  established for approximators constructed in Theorem~\ref{thm:app_L2},
thereby inflating the total parameter complexity for the subsequent $L^{\infty}([0,1]^d)$ approximation.
To mitigate this, we exploit the periodicity to approximate the weight functions with a network of width $O(K)$, while maintaining the same order of total parameters. Specifically, we approximate the weight functions on $[0, K^{-1}]$ with $K$ periods, and for all $x \in [0,1]$, we extract the relative position within its associated coarse intervals. Define $a^K$ as:
\begin{equation*}
    a^K(x): = \sum_{i=0}^{K-1} \frac{i}{K} \mathbbm{1}_{[\frac{i}{K},\frac{i+1}{K})} (x),
\end{equation*}
and use the relative information $x- a^K(x)$ as the input to the networks.

The following lemma establishes neural network approximators for univariate weight functions on $[0, 1)$.
\begin{lemma}
    \label{lem:app_wieght_band}
    Let $\phi$ satisfy Assumptions~\ref{ass:smooth}--\ref{ass:piece}, and let $w_{\phi,i}^{\beta, \delta, K}$ for $i=1,2$ denote the univariate weight functions defined in Definition~\ref{def:univariate_weight} with the same parameters $\beta, \delta, K$. For any sufficiently small $\epsilon \in (0,1)$, sufficiently large $K$ and any $\beta \geq \frac{4K}{\delta \epsilon} \max \{C_1, C_2\}$, there exist neural networks $\{\psi_i\}_{i=1}^2$ such that
    \begin{equation*}
        \psi_i \in \mathcal{H}^{\phi,3}(1,1,M_{K},B_{\beta, \epsilon, \delta, K}), \quad i = 1 , 2,
    \end{equation*}
    with 
    \begin{equation}
        M_{K} \lesssim K, \quad B_{\beta, \epsilon, \delta, K} \lesssim \frac{K^2 \beta^2}{\epsilon \delta^2},
        \label{eq:width-norm-app-weight-band}
    \end{equation}
    such that the following properties hold:
    \begin{itemize}
        \item (Approximation) Let $\Omega_{\mathrm{coarse,band}}^{K,i}({\delta})$ denote the band region associated with the coarse cells on $[0,1]$:
        \begin{equation*}
            \begin{aligned}
                 & \Omega_{\mathrm{coarse,band}}^{K,1}({\delta}) \coloneqq  \bigcup_{k=0}^K \left[ -{\delta}+\frac{k}{K},{\delta} + \frac{k}{K} \right] \bigcap [0,1], \\
                 & \Omega_{\mathrm{coarse,band}}^{K,2}({\delta}) \coloneqq  \bigcup_{k=0}^{K-1} \left[\frac{2k+1}{2K}  - {\delta}, \frac{2k+1}{2K} + {\delta}\right],
            \end{aligned}
        \end{equation*}
        Then, for $x \in [0,1) \setminus \Omega_{\mathrm{coarse,band}}^{K,i}({\delta})$,
        \begin{equation*}
            |\psi_i(x) - w_{\phi,i}^{\beta, K, \delta} (x)| \leq \epsilon,  \quad i = 1, 2.
        \end{equation*}
        \item (Boundedness) $\|\psi_i\|_{L^{\infty}(\RR)} \leq 2C_1+4 , \quad i = 1,2.$
    \end{itemize}
\end{lemma}

\begin{proof}
    It suffices to prove the result for $\psi_1$ and $w_{\phi,1}^{\beta, \delta, K}$, as the argument for the complementary $\psi_2$ and $w_{\phi,2}^{\beta, \delta, K}$ follows symmetrically.

    \noindent
    \paragraph{Heaviside-like case.} We first construct a two-layer network $\eta_1$ designed to approximate the weight function on the reference domain $[0, K^{-1}]$:
    \begin{equation}
        \eta_1 (x) \coloneqq  \sum_{k=0}^{K-1} \left[ \phi\left( \beta \left( x - \frac{k}{K^2} - 3 \delta\right)\right) - \phi\left( \beta \left( x - \frac{k+1}{K^2} + 3 \delta\right)\right)\right].
        \label{eq:psi1_construct}
    \end{equation}
    Fix an interval index $k \in \{0, \dots, K-1\}$. For any $x \in [\frac{k}{K^2}, \frac{k+1}{K^2})$, the periodicity of $w_{\phi,1}^{\beta, \delta, K}$ implies that 
    \begin{equation*}
        w_{\phi,1}^{\beta, \delta, K}(x) = w_{\phi,1}^{\beta, \delta, K} \left(x - \frac{k}{K^2} \right), \quad x\in \left[ \frac{k}{K^2}, \frac{k+1}{K^2}\right), \quad k = 0,1, \cdots, K-1.
    \end{equation*}
    Then the approximation error is determined by the tails of the remaining terms $l \neq k$  in ~\eqref{eq:psi1_construct}:
    \begin{equation}
        \begin{aligned}
            |\eta_1 (x) - w_{\phi,1}^{\beta, \delta, K}(x)| &\leq \sum_{l=0, l\neq k}^{K-1} \left|\phi\left( \beta \left( x - \frac{l}{K^2} - 3 \delta\right)\right) - \phi\left( \beta \left( x - \frac{l+1}{K^2} + 3 \delta\right)\right)\right| \\
            & \overset{(a)}{\leq } \sum_{l=0, l\neq k}^{K-1} \left|H\left( \beta \left( x - \frac{l}{K^2} - 3 \delta\right)\right) - H\left( \beta \left( x - \frac{l+1}{K^2} + 3 \delta\right)\right)\right|  \\
            & \; \; \;+ \frac{C_1}{\beta} \sum_{l=0, l\neq k}^{K-1} \left[ \left|x - \frac{l}{K^2}-3 \delta \right|^{-1} + \left|x - \frac{l+1}{K^2}+ 3 \delta \right|^{-1}\right] \\
            & \overset{(b)}{\leq } \frac{2C_1 K}{3 \beta \delta} \overset{(c)}{\leq } \frac{\epsilon}{2}.
        \end{aligned}
        \label{eq:two-layer-app-Heaviside}
    \end{equation}
    Here, Step $(a)$ follows from the Heaviside-like assumption given in ~\eqref{ass:piece}. Step $(b)$ follows because the distance from $x \in [\frac{k}{K^2}, \frac{k+1}{K^2})$ to the switching points $l/K^2+3\delta, (l+1)/K^2-3\delta$ of any $l \neq k$ term is at least $3\delta$. Step $(c)$ follows from $\beta \geq \frac{4C_1 K}{3 \delta \epsilon}$.

    Next, using the approximation guarantee given in ~\eqref{eq:two-layer-app-Heaviside}, we obtain:
    \begin{equation*}
        \|\eta_1\|_{L^{\infty}([0,K^{-1}])} \leq \|w_{\phi,1}^{\beta, \delta, K}\|_{L^{\infty}([0,K^{-1}])} + \frac{\epsilon}{2} \overset{(a)}{\leq } 2C_1 + 3 +\frac{\epsilon}{2} < 2C_1 +4,
    \end{equation*}
    where $(a)$ comes from the boundedness of $w_{\phi,1}^{\beta, \delta, K^{-1}}$ stated in Proposition~\ref{prop:univariate_weight}. For $x \notin [0,K^{-1}]$, a similar analysis to~\eqref{eq:two-layer-app-Heaviside} yields:
    \begin{equation*}
        |\eta_1(x)| \leq \frac{C_1}{\beta}\sum_{k=0}^{K-1} \left[\left| x - \frac{k}{K^2} - 3 \delta \right|^{-1} + \left| x - \frac{k+1}{K^2} + 3 \delta\right|^{-1}\right] \leq \frac{2C_1 K}{3 \beta \delta} \leq \frac{\epsilon}{2} < 1.
    \end{equation*}
    Thus, $\|\eta_1\|_{L^{\infty}(\mathbb{R})} \leq 2C_1 + 4$.

    By Lemma~\ref{lem:relative_pos_app}, there exists a network $\pi \in \mathcal{H}^{\phi,2}(1,1,6K,B_{\epsilon_1, {\delta}, K})$, where $B_{\epsilon_1, {\delta}, K} \lesssim \frac{K}{\epsilon_1 {\delta}}$, approximating the relative position map $x \mapsto x - a^K(x)$ such that:
    \begin{equation}
        |\pi(x) - (x - a^K(x))| \leq \epsilon_1, \quad x \in [0,1] \setminus \Omega_{\mathrm{coarse,band}}^{K, 1} ({\delta}).
        \label{eq:app_nu_relative}
    \end{equation}
    We define the final approximation $\psi_1 \coloneqq \eta_1 \circ \pi$.  The uniform boundedness of $\psi_1$ follows immediately from the boundedness of the outer function $\eta_1$:
    \begin{equation*}
        \|\psi_1\|_{L^{\infty}(\RR)} \leq \|\eta_1\|_{L^{\infty}(\RR)} \leq 2C_1 + 4.
    \end{equation*}
     For any $x \in [0,1] \setminus \Omega_{\mathrm{coarse,band}}^{K, 1} ({\delta})$, the  error between $\psi_1(x)$ and $w_{\phi,1}^{\beta, \delta, K}(x)$ is bounded as:
    \begin{equation*}
        \begin{aligned}
            |\psi_1(x) - w_{\phi,1}^{\beta, \delta, K}(x)| & = |\eta_1(\pi(x)) - w_{\phi,1}^{\beta, \delta, K}(x - a^K(x))| \\
            & \leq \underbrace{|\eta_1 (\pi(x)) - \eta_1(x- a^K(x))|}_{(a)}  + \underbrace{|\eta_1(x - a^K(x)) - w_{\phi,1}^{\beta, \delta,K}(x - a^K(x)) |}_{(b)}\\
            & \leq 2 \beta K \|\phi\|_{\mathrm{Lip}} \epsilon_1 + \frac{\epsilon}{2} \overset{(c)}{\leq} \epsilon,
        \end{aligned}
    \end{equation*}
    where term $(a)$ is bounded by the approximation guarantee given in ~\eqref{eq:app_nu_relative}, and there are $2K$ activations with Lipschitz constant $\|\phi\|_{\mathrm{Lip}}$ in the construction of \eqref{eq:psi1_construct}, while term $(b)$ follows from $x - a^K(x) \in [0, K^{-1}]$ and the approximation guarantee in \eqref{eq:two-layer-app-Heaviside}. Finally, $(c)$ follows from choosing $\epsilon_1 = \frac{\epsilon}{4K \beta \|\phi\|_{\mathrm{Lip}}}$.

    The widths of $\eta_1$ and $\pi$ are bounded by $6K$ and $2K$, respectively; thus, the width of the composition $\psi_1$ is bounded by $6K$. Furthermore, $\|\theta(\psi_1)\|_{\infty}$ is bounded by :
    \begin{equation*}
        \|\theta(\psi_1)\|_{\infty} \leq \|\theta(\eta_1)\|_{\infty} \|\theta(\pi)\|_{\infty} \lesssim \beta B_{\epsilon_1, {\delta},K} \lesssim \frac{K^2 \beta^2}{\epsilon {\delta}}.
    \end{equation*}

    \noindent
    \paragraph{ReLU-like case.} The proof for ReLU-like activations $\phi$ follows similarly to that for Heaviside-like activations. The network $\eta_1$ is reconstructed as:
    \begin{equation*}
        \begin{aligned}
            \eta_1(x): &= \frac{1}{2 \delta \beta}\sum_{k=0}^{K-1} \Bigg[ \phi \left( \beta \left( x - \frac{k}{K^2} - 2\delta\right)\right) - \phi \left( \beta \left( x - \frac{k}{K^2} - 4\delta\right)\right) \\
            & \qquad \qquad \qquad \qquad \qquad - \phi \left( \beta \left( x - \frac{k+1}{K^2} +4\delta\right)\right) + \phi \left( \beta \left( x - \frac{k+1}{K^2} +2\delta\right)\right)
            \Bigg].
        \end{aligned}
    \end{equation*}

    Consider an interval $[\frac{k}{K^2}, \frac{k+1}{K^2})$. For any index $l \neq k$, the associated ReLU trapezoid vanishes on this interval. Consequently, the error is dominated by the deviation of $\phi$ from $\mathrm{ReLU}$:
    \begin{equation}
        \begin{aligned}
            &|\eta_1 (x) - w_{\phi,1}^{\beta, \delta, K}(x)|  \overset{(a)}{\leq} \frac{1}{2\delta \beta} \sum_{l =0, l \neq k}^{K-1} \Big| \mathrm{ReLU}(\beta (x - l/K^2-2 \delta)) -  \mathrm{ReLU}(\beta (x - l/K^2-4 \delta)) \\
           &   -  \mathrm{ReLU}(\beta (x - (l+1)/K^2+4 \delta))  +   \mathrm{ReLU}(\beta (x - (l+1)/K^2+2 \delta))\Big| + 
           \frac{2C_2 K}{\delta \beta} \overset{(b)}{\leq} \frac{\epsilon}{2},
        \end{aligned}
        \label{eq:two-layer-app-ReLU}
    \end{equation}
    where $(a)$ follows from the ReLU-like assumption given in~\eqref{eq:ReLU-like} and $(b)$ follows from $\beta \geq \frac{4C_2 K}{\delta \epsilon} $.

    Define $\psi_1 \coloneqq \eta_1 \circ \pi$. Following the same analysis for Heaviside-like activations, we obtain the boundedness $ \|\psi_1 \|_{L^{\infty}(\RR)} \leq \|\eta_1\|_{L^{\infty}(\RR)} \leq 2C_1 + 4$ and the approximation error between $\psi_1$ and $w_{\phi,1}^{\beta, \delta, K}$ is bounded as:
    \begin{equation*}
        \begin{aligned}
            |\psi_1(x) - w_{\phi,1}^{\beta, \delta, K}(x)| & \leq \underbrace{|\eta_1 (\pi(x)) - \eta_1(x- a^K(x))|}_{(a)}  + \underbrace{|\eta_1(x - a^K(x)) - w_{\phi,1}^{\beta, \delta,K}(x - a^K(x)) |}_{(b)}\\
            & \leq \frac{4  K\|\phi\|_{\mathrm{Lip}} \epsilon_1}{\delta} + \frac{\epsilon}{2} \leq \epsilon, \qquad   x \in [0,1] \setminus \Omega_{\mathrm{coarse, band}}^{K,1}({\delta}),
        \end{aligned}
    \end{equation*}
    where $(a)$ is bounded by the approximation guarantee given in ~\eqref{eq:app_nu_relative}, and there are $4K$ activations with Lipschitz constant $\|\phi\|_{\mathrm{Lip}}$ in the construction of \eqref{eq:psi1_construct}, $(b)$ is bounded by the approximation guarantee given in ~\eqref{eq:two-layer-app-ReLU} and we set $\epsilon_1 = \frac{\delta \epsilon}{8 K\|\phi\|_{\mathrm{Lip}}}$.

    Finally, the depth of $\psi_1$ is 3, and its width is bounded by ${O}(K)$. The parameter norm satisfies. The parameter norm satisfies:
    \begin{equation*}
         \|\theta(\psi_1)\|_{\infty} \leq \|\theta(\eta_1)\|_{\infty}  \|\theta(\pi)\|_{\infty} \lesssim \max \left\{ \beta, \frac{1}{\beta \delta}\right\} B_{\epsilon_1, {\delta},K} \lesssim \frac{K^2 \beta^2}{\epsilon  \delta^2},
    \end{equation*}
    as stated in ~\eqref{eq:width-norm-app-weight-band}. This concludes the proof.
\end{proof}

Lemma~\ref{lem:app_wieght_band} establishes the approximation of the univariate weight functions outside the coarse band region. To achieve a uniform approximation over the entire interval $[0,1]$, we multiply the network output by a spatial quasi-indicator function supported on the complement of the coarse bands. This strategy exploits the fact that within the coarse band regions, the target weight functions are quasi-vanishing, while the neural network output remains uniformly bounded. Consequently, this multiplication effectively suppresses the error in the band regions.

In the following, We provide the definition for the quasi-indicator functions.

\begin{definition}[Quasi-indicator function]
    Let $\phi$ satisfy Assumption~\ref{ass:piece}, with $K \in \mathbb{N}_{+}, {\delta} \in (0, \frac{1}{8K}) $ and $\tilde{\beta}>0$. The primary quasi-indicator function is defined as:
    \begin{itemize}
        \item If $\phi$ is Heaviside-like, then
        \begin{equation}
            I_{\phi,1}^{\tilde{\beta}, {\delta}, K}(x): = \sum_{k=0}^K \left[ \phi \left( \tilde{\beta} \left( x - \frac{k}{K} - \frac{3}{2} {\delta}\right) \right) - \phi \left( \tilde{\beta} \left( x - \frac{k+1}{K} + \frac{3}{2} {\delta}\right)\right)\right].
            \label{eq:quasi-indicator-Heaviside-like}
        \end{equation}
        \item If $\phi$ is ReLU-like, then 
        \begin{equation}
             \begin{aligned}
            I_{\phi,1}^{\tilde{\beta}, {\delta}, K}(x): &= \frac{1}{2 \delta \tilde{ \beta}}\sum_{k=0}^{K} \Bigg[ \phi \left( \tilde{\beta} \left( x - \frac{k}{K} - 2 {\delta}\right)\right) - \phi \left( \tilde{\beta} \left( x - \frac{k}{K} - 4 {\delta}\right)\right) \\
            & \qquad \qquad  - \phi \left( \tilde{\beta} \left( x - \frac{k+1}{K} +4 {\delta}\right)\right) + \phi \left( \tilde{\beta} \left( x - \frac{k+1}{K} +2 {\delta}\right)\right)
            \Bigg].
        \end{aligned}
        \label{eq:quasi-indicator-ReLU-like}
        \end{equation}
    \end{itemize}
    In both cases, the complementary quasi-indicator function is defined by the translation:
    \begin{equation*}
        I_{\phi,2}^{\tilde{\beta}, {\delta},K}(x) \coloneqq  I_{\phi,1}^{\tilde{\beta}, {\delta},K} \left(x+ \frac{1}{2K} \right), \quad x \in [0,1].
    \end{equation*}
    \label{def:quasi-indicator}
\end{definition}

The quasi-indicator functions have the following properties:
\begin{proposition}
    Let $I_{\phi,i}^{\tilde{\beta},{\delta},K}$ denote the quasi-indicator functions defined in Definition~\ref{def:quasi-indicator}. For any sufficiently small $\epsilon \in (0,1)$ and any $\tilde{\beta} \geq \frac{8 K}{{\delta} \epsilon } \max\{C_1, C_2\}$, the following properties hold for $i=1,2$:
    \begin{itemize}
        \item (Indicator approximation) The function approximates the indicator function as:
        \begin{equation*}
            \left| I_{\phi,i}^{\tilde{\beta},{\delta},K} (x)- \mathbbm{1}_{\left(\Omega_{\mathrm{coarse,band}}^{K,i}(2 \delta)\right)^c} \left( x \right)\right| \leq \epsilon, \quad x \in \mathcal{D}^{K,i}({\delta}),
        \end{equation*}
        where $\left(\Omega_{\mathrm{coarse,band}}^{K,i}(2 \delta)\right)^c \coloneqq [0,1]  \setminus \Omega_{\mathrm{coarse,band}}^{K,i}(2 \delta)$ and the region $\mathcal{D}^{K,i}({\delta})$ is defined as:
        \begin{equation*}
            \mathcal{D}^{K,i}({\delta})\coloneqq \left( [0,1]  \setminus \Omega_{\mathrm{coarse,band}}^{K,i}(2 {\delta})  \right) \bigcup  \Omega_{\mathrm{coarse,band}}^{K,i}( {\delta}).
        \end{equation*}
        \item (Boundedness) $\|I_{\phi,i}^{\tilde{\beta}, {\delta}, K}\|_{L^{\infty}([0,1])} \leq 2C_1 +2$.
    \end{itemize}
    \label{prop:quasi-indicator}
\end{proposition}

\begin{proof}
    The proof proceeds analogously to that of Proposition~\ref{prop:univariate_weight} and Lemma~\ref{lem:app_wieght_band}. First, regarding the approximation error, we observe that substituting the exact Heaviside function into \eqref{eq:quasi-indicator-Heaviside-like} (or the exact ReLU function into \eqref{eq:quasi-indicator-ReLU-like}) recovers the indicator function supported on $[0,1] \setminus \Omega_{\mathrm{coarse,band}}^{K,i}(2 {\delta})$ for all $x \in \mathcal{D}^{K,i}({\delta})$. To establish the error bound for the approximator $I_{\phi,1}^{\tilde{\beta},{\delta},K}$, consider an arbitrary $x \in \mathcal{D}^{K,i}({\delta})$. By construction, the distance between $x$ and the transition points $\frac{k+i-1}{2K} \pm \frac{3}{2} {\delta}$ is bounded from below by $\frac{1}{2} {\delta}$. Consequently, invoking the asymptotic properties of $\phi$ (specifically the Heaviside-like condition \eqref{eq:Heaviside-like} or the ReLU-like condition \eqref{eq:ReLU-like}) ensures that the approximation error can be made arbitrarily small for sufficiently large $\tilde{\beta}$. The boundedness arguments for quasi-indicator functions  mirror those established for $\psi_i$ in Lemma~\ref{lem:app_wieght_band}.
\end{proof}

Leveraging the quasi-indicator functions, we now construct neural approximators for the univariate weight functions $w_{\phi,i}^{\beta, \delta, K}$ over the global domain $[0,1]$.
\begin{lemma}
    Let $\phi$ satisfy Assumptions~\ref{ass:smooth}--\ref{ass:piece}. 
    Let $w_{\phi,i}^{\beta, \delta, K}$ ($i=1,2$) denote the univariate weight functions defined in Definition~\ref{def:univariate_weight} with parameters $\beta, \delta, K$. For any sufficiently small $\epsilon \in (0,1)$, sufficiently large $K \in \mathbb{N}_{+}$, $\delta \in (0, \frac{1}{12K^2})$, any $\beta \geq \frac{(24C_1+ 60) K}{\delta \epsilon} \max \{C_1, C_2\}$, there exist neural networks $\{g_i\}_{i=1,2}$ such that
    \begin{equation*}
        g_i \in \mathcal{H}^{\phi,4}(1,1, M_K, B_{\beta, \epsilon,\delta,K}), \quad i=1,2,
    \end{equation*}
    with 
    \begin{equation}
        M_K \lesssim K, \quad B_{\beta, \epsilon,\delta,K} \lesssim \frac{K^2 \beta^2}{ \epsilon^2 \delta^2},
        \label{eq:width-norm-app-weight-global}
    \end{equation}
    such that the following properties hold for $i=1,2$:
    \begin{itemize}
        \item (Uniform approximation) For all $x \in [0,1]$, 
        \begin{equation*}
            |g_i(x) - w_{\phi,i}^{\beta,\delta,K} (x)|  \leq \epsilon.
        \end{equation*}
        \item (Boundedness) $\|g_i\|_{L^{\infty}([0,1])} \leq (2C_1 + 5)^2$.
    \end{itemize}
    \label{lem:app_univariate_weight_Linfty}
\end{lemma}

\begin{proof}
   We focus our analysis on the construction of the approximator for $w_{\phi,1}^{\beta, \delta, K}$; the construction for $w_{\phi,2}^{\beta, \delta, K}$ proceeds analogously.

    Firstly, for sufficiently small $\epsilon_1 \in (0,1)$, by Lemma~\ref{lem:app_wieght_band},  there exists a neural network $\psi_1$ of depth $3$, width $O(K)$, and parameter norm $O(\frac{K^2 \beta^2}{\epsilon_1 \delta^2})$ such that 
    \begin{equation}
        |\psi_1(x) - w_{\phi,1}^{\beta, \delta, K}(x) | \leq \epsilon_1, \quad x \in [0,1] \setminus \Omega_{\mathrm{coarse, band}}^{K,1}(\delta),
        \label{eq:app_psi1_univariate_weight}
    \end{equation}
    provided that $\beta \geq \frac{4K}{\delta \epsilon_1} \max \{C_1, C_2\}$. Furthermore, the network satisfies the uniform bound $\|\psi_1\|_{L^{\infty}([0,1])} \leq 2C_1+4$.

    Secondly, according to Definition~\ref{def:quasi-indicator}, the quasi-indicator $I_{\phi,1}^{\tilde{\beta}, {\delta}, K}$ is represented by a neural network of depth $2$, width $O(K)$, and parameter norm $O({\tilde{\beta}}/{{\delta}})$. Furthermore, by Proposition~\ref{prop:quasi-indicator}, upon choosing $\tilde{\beta} = \frac{16K}{\delta \epsilon_2} \max \{C_1, C_2\}$, we obtain the approximation error bound:
     \begin{equation}
       \left| I_{\phi,1}^{\tilde{\beta},{\delta},K} (x)- \mathbbm{1}_{\left(\Omega_{\mathrm{coarse,band}}^{K,1}(2 \delta)\right)^c} \left( x \right) \right| \leq \frac{\epsilon_2}{2}, \quad x \in \mathcal{D}^{K,1}(\delta),
       \label{eq:app_I_indicator}
    \end{equation}
    and the uniform boundedness $\|I_{\phi,i}^{\tilde{\beta},{\delta},K}\|_{L^{\infty}([0,1])} \leq 2C_1 + 2$.

    Next, by invoking Corollary~\ref{corol:id-approx}, for sufficiently small $\epsilon_2 \in (0,1)$, there exists a neural network $\mu$ of depth $2$, width $2$ and parameter norm $O({1}/{\epsilon_2})$, such that
    \begin{equation}
        |\mu(x) - x| \leq \frac{\epsilon_2}{2}, \quad x \in [-(2C_1 + 2), 2C_1+2].
        \label{eq:app_mu_x}
    \end{equation}
    Now, we define the composition of $\mu$ with the quasi-indicator function as $\nu_1 = \mu \circ I_{\phi,1}^{\tilde{\beta}, \delta, K}$. By ~\eqref{eq:app_I_indicator} and \eqref{eq:app_mu_x}, we obtain the following approximation error
    \begin{equation}
        \left|\nu_1(x) -  \mathbbm{1}_{\left(\Omega_{\mathrm{coarse,band}}^{K,1}(2 \delta)\right)^c} \left( x \right) \right| \leq \epsilon_2, \quad x \in \mathcal{D}^{K,1}(\delta).
        \label{eq:app_nu1_coarse_band}
    \end{equation}
    and the uniform boundedness $\|\nu_1\|_{L^{\infty}([0,1])} \leq 2 C_1 + 4$. Structurally, $\nu_1$ is neural network of depth $3$, width $O(K)$, and parameter norm $O(\max\{\tilde{\beta}, 1/ \epsilon_2\}) = O(K / (\delta \epsilon_2))$.

    Subsequently, for sufficiently small $\epsilon_3 \in (0,1)$, by Lemma~\ref{lemma:all_monomials}, there exists a neural network $\eta$ of depth $2$, width $4$, and parameter norm $O(1/\epsilon_3^2)$ such that 
    \begin{equation}
        |\eta(x,y) - x y| \leq  \epsilon_3, \quad 0<   |x|, |y| \leq 2C_1 + 4.
        \label{eq:app_xy}
    \end{equation}
    We define the final approximator as $g_1(x) \coloneqq  \eta(\psi_1(x), \nu_1(x))$. We proceed to bound the global approximation error between $g_1(x)$ and $I_{\phi,1}^{\beta, \delta, K}$ by partitioning the domain $[0,1]$ into three regions. We choose parameters $\epsilon_1 = \epsilon_2 = \frac{\epsilon}{6C_1+15}$ and $\epsilon_3 = \frac{\epsilon}{3}$.

    \begin{itemize}
        \item For $x \in \left(\Omega_{\mathrm{coarse,band}}^{K,1}(2 \delta)\right)^c$, the approximation error is bounded by 
        \begin{equation*}
            \begin{aligned}
                |g_1(x  ) -w_{\phi,1}^{\beta, \delta, K}(x)| &\leq |\eta(\psi_1(x), \nu_1(x)) - \psi_1(x) \nu_1(x)| + \left|\psi_1(x) \nu_1(x) - w_{\phi,1}^{\beta, \delta, K}(x)  \right| \\
                & \leq |\eta(\psi_1(x), \nu_1(x)) - \psi_1(x) \nu_1(x)|  + |\nu_1(x)||\psi_1(x) - w_{\phi,1}^{\beta, \delta, K} (x)| \\
                & \quad + |w_{\phi,1}^{\beta, \delta, K}(x)| \left| \nu_1(x) - \mathbbm{1}_{\left(\Omega_{\mathrm{coarse,band}}^{K,1}(2 \delta)\right)^c} \left( x \right) \right| \\
                & \overset{(a)}{\leq} \epsilon_3 + (2C_1 + 4) \epsilon_1 + (2C_1 + 3) \epsilon_2 <\epsilon,
            \end{aligned}
        \end{equation*}
        where $(a)$ follows from the boundedness of $\nu_1$, $w_{\phi,1}^{\beta, \delta, K}$, and the approximation guarantees in \eqref{eq:app_psi1_univariate_weight}, \eqref{eq:app_nu1_coarse_band}, and \eqref{eq:app_xy}.
        
        \item When $x \in \Omega_{\mathrm{coarse, band}}^{K, 1}(\delta)$, the approximator nearly vanishes:
        \begin{equation*}
            |g_1(x)| \overset{(a)}{\leq} |\psi_1(x)| |\nu_1(x) | + \epsilon_3  \overset{(b)}{\leq} (2C_1 + 4) \epsilon_2 + \epsilon_3 \overset{(c)}{\leq} \frac{2}{3} \epsilon,
        \end{equation*}
        where $(a)$ follows from the approximation guarantee in \eqref{eq:app_xy}, $(b)$ follows from the approximation guarantee in ~\eqref{eq:app_nu1_coarse_band} and the boundedness of $\psi_1$. Moreover, by Proposition~\ref{prop:univariate_weight}, the target univariate function also exhibits near-vanishing behavior as $|w_{\phi,1}^{\beta, \delta, K}(x)| \leq \epsilon_1$ for $x \in \Omega_{\mathrm{coarse,band}}^{K,1}(\delta) \subset \Omega_{\mathrm{coarse,band}}^{K,1}(2 \delta) $ provided $\beta \geq \frac{4K}{\delta \epsilon_1} \max \{C_1, C_2\}$. Thus, the total approximation error is bounded as
        \begin{equation*}
            |g_1(x) - w_{\phi,1}^{\beta, \delta,K}(x)| \leq \epsilon_1 + \frac{2}{3} \epsilon < \epsilon, \quad x \in \Omega_{\mathrm{coarse,band}}^{K,1}(\delta).
        \end{equation*}
        \item When $x \in \Omega_{\mathrm{coarse,band}}^{K,1}(2 \delta) \setminus \Omega_{\mathrm{coarse,band}}^{K,1}( \delta)$, we obtain 
        \begin{equation*}
            \begin{aligned}
               |g_1(x) - w_{\phi,1}^{\beta, \delta, K}(x)| & \overset{(a)}{\leq}|\psi_1(x)||\nu_1(x)| + |w_{\phi,1}^{\beta, \delta, K}(x)| + \epsilon_3 \\
                & \leq |\psi_1(x) - w_{{\phi},1}^{\beta, \delta, K}(x) | |\nu_1(x)| + ( |\nu_1(x)| + 1) |w_{\phi,1}^{\beta, \delta, K}(x)| + \epsilon_3 \\
                & \overset{(b)}{\leq} \epsilon_1 (2C_1 + 4) + \epsilon_1(2C_1 + 5) + \epsilon_3 \leq \epsilon,
            \end{aligned}
        \end{equation*}
        where $(a)$ follows from the approximation guarantee in \eqref{eq:app_xy}, $(b)$ follows from the approximation guarantee in \eqref{eq:app_psi1_univariate_weight}, the boundedness of $\nu_1$, and the locally quasi-vanishing property of $w_{\phi,1}^{\beta, \delta, K}$ when $\beta \geq \frac{4K}{\delta \epsilon_1} \max \{C_1, C_2\}$.
    \end{itemize}

    The uniform boundedness of $g_1$ follows directly from the product bound:
    \begin{equation*}
        |g_1(x)| \leq |\psi_1(x)| |\nu_1(x)| + \epsilon_3 \leq (2C_1+4)^2 + \epsilon_3 \leq (2C_1 + 5)^2, \quad x\in[0,1].
    \end{equation*}
    
    Regarding the network architecture, the composition of the constituent sub-networks yields a final architecture for
    $g_1$ with depth $4$, width $O(K)$, and its parameter norm bounded by
    \begin{equation*}
        \|\theta(g_1)\|_{\infty} \lesssim \max \left\{ \frac{K^2 \beta^2}{\epsilon_1 \delta^2}, \frac{K}{\delta \epsilon_2^2}, \frac{1}{\epsilon_3^2}\right\} \lesssim \frac{K^2 \beta^2}{ \epsilon^2 \delta^2}, 
    \end{equation*}
    as stated in ~\eqref{eq:width-norm-app-weight-global}.   Finally, to ensure the validity of the preceding analysis, we require the parameter $\beta$ to satisfy
    \begin{equation*}
        \beta \geq \frac{4K}{\delta \epsilon_1} \max \{C_1 ,C_2\} = \frac{(24C_1 + 60 K)}{\delta \epsilon} \max\{C_1, C_2\}.
    \end{equation*}
    This concludes the proof.
\end{proof}

\begin{figure}[htbp]
    \centering
    \includegraphics[width=0.95\textwidth]{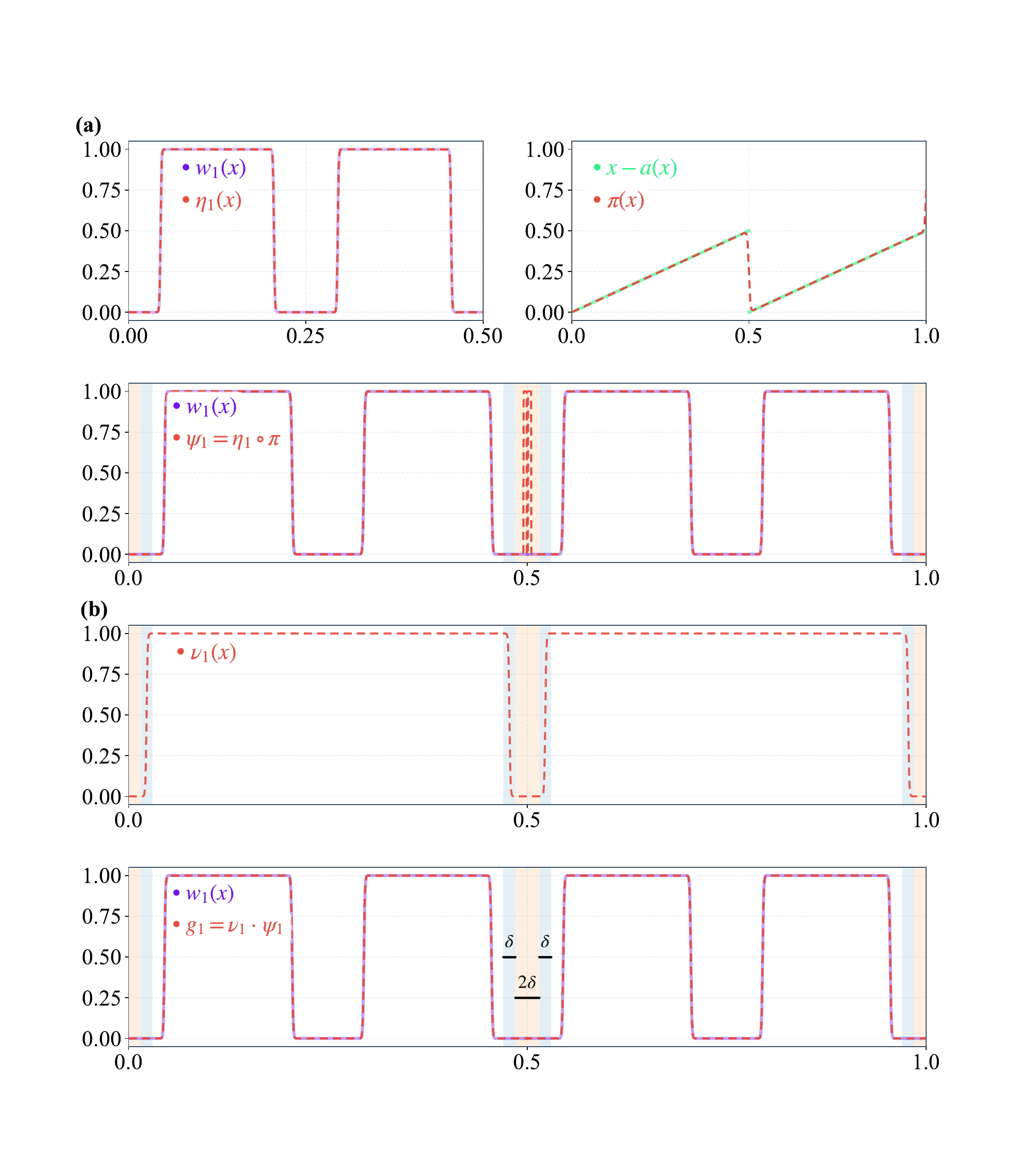}
    \caption{Illustration of the constructive approximation for the weight function $w_1$,  instantiated with $K=2$. Dependencies on indices $\phi, \beta$ and $\delta$ are suppressed for clarity. Panels (a) and (b) visualize the approximators constructed in Lemma~\ref{lem:app_wieght_band} and Lemma~\ref{lem:app_univariate_weight_Linfty}, respectively. The orange shaded region denotes the domain $\Omega_{\mathrm{coarse,band}}^{K,1}(\delta)$, while the blue region corresponds to the difference set $\Omega_{\mathrm{coarse,band}}^{K,1}(2\delta) \setminus \Omega_{\mathrm{coarse,band}}^{K,1}(\delta)$. }
    \label{fig:Linfty}
    \end{figure}

Figure~\ref{fig:Linfty} illustrates our constructive approximation procedure. Panel (a) depicts the construction in Lemma~\ref{lem:app_wieght_band}, where we approximate the weight function $w_1$ on $[0,1]$ up to an exceptional boundary band of width $\delta$. The resulting approximator is denoted by $\psi_1$. This construction combines a local approximator $\eta_1$ on $[0,K^{-1}]$ with an extractor that maps $x$ to its relative position $x-a(x)$. Panel (b) corresponds to Lemma~\ref{lem:app_univariate_weight_Linfty}, where we upgrade the approximation to a uniform one over $[0,1]$ and obtain an approximator $g_1$. The key step is to multiply the local approximator $\psi_1$ by an indicator approximator $\nu_1$ that suppresses the approximation error in the boundary band.

Building upon the neural approximators constructed in Lemma~\ref{lem:app_univariate_weight_Linfty} for the univariate weight functions on $[0,1]$, we now proceed to construct neural approximators for the multivariate functions on $[0,1]^d$.

\begin{lemma}
    \label{lem:app-multi-weight-global}
    Let $\phi$ satisfy Assumptions~\ref{ass:smooth}--\ref{ass:piece}. Let $d \in \mathbb{N}_{+}$, and let $w_{\phi, \bv}^{\beta, \delta, K}$ be the $d$-variate weight functions from Definition~\ref{def:multi_weight} with parameters $\beta, \delta, K$. For any sufficiently small $\epsilon \in (0,1)$, sufficiently large $K$, $\delta \in (0, \frac{1}{12K^2})$, and $\beta \geq \frac{24d(2C_1+5)^{2d+1}}{\delta \epsilon} \max\{ C_1, C_2\}$, there exist neural networks $\{g_{\bv}\}_{\bv \in [2]^d}$ satisfying
    \begin{equation*}
        g_{\bv} \in \mathcal{H}^{\phi,5}(1,1, M_K, B_{\beta, \epsilon, \delta, K}),
    \end{equation*}
    with 
    \begin{equation}
        M_{K} \lesssim K, \quad B_{\beta, \epsilon, \beta, K} \lesssim \max \left\{ \frac{K^2 \beta^2}{\epsilon^2 \delta^2}, \frac{1}{\epsilon^d}\right\},
        \label{eq:app-width-norm-phi-multi-weight}
    \end{equation}
    such that the following properties hold for  $\bv \in [2]^d$:
    \begin{itemize}
        \item (Uniform approximation) For all $x \in [0,1]^d$:
        \begin{equation*}
            |g_{\bv}(\bx) - w_{\phi,\bv}^{\beta, \delta, K}(\bx)| \leq 
            \epsilon.
        \end{equation*}
        \item (Boundedness) $\|{g}_{\bv}\|_{L^{\infty}([0,1]^d)} \leq (2C_1+5)^d$.
    \end{itemize}
\end{lemma}

\begin{proof}
    Consider a sufficiently small $\epsilon_1 \in (0,1)$.
    By Lemma~\ref{lem:app_univariate_weight_Linfty}, there exist $2d$ neural networks $\{\psi_{l,v}\}_{l \in [d], v\in[2]}$ such that  for all $\bx \in [0,1]^d$
    \begin{equation}
        \left| \psi_{l,v} (\bx)-  w_{\phi, v}^{\beta, \delta, K}(x_l) \right| \leq  \epsilon_1, \quad \quad l =1, \cdots, d, \quad v = 1, 2,
        \label{eq:app_psi_univariate_weight_lv}
    \end{equation}
    provided that $\beta \geq \frac{(24C_1 + 60) K}{\delta \epsilon_1} \max \{C_1, C_2\}$. Moreover, each $\psi_{l,v_l}$ has depth $4$, width $O(K)$, parameter norm $O(\frac{K^2 \beta^2}{\epsilon_1 \delta^2})$, and satisfies the uniform bound $\|\psi_{l,v_l}\|_{L^{\infty}} \leq (2C_1+5)^2$.

    Next, for sufficiently small $\epsilon_2 \in (0,1) $, by Lemma~\ref{lemma:all_monomials}, there exists a neural network $ \nu $ of depth 2, width $ 2^d $, and parameter norm $O(1 / \epsilon_2^d)$ such that
    \begin{equation}
       |\nu(x_1, \cdots ,x_d) - x_1 x_2 \cdots x_d| \leq \epsilon_2, \quad 0\leq |x_1|, \cdots,  |x_d| \leq (2C_1+5)^2.
       \label{eq:app_nu_x1xd}
    \end{equation}

    We now define the final neural approximator as $g_{\bv}(\bx) \coloneqq  \nu(\psi_{1,v_1}(\bx), \cdots, \psi_{d, v_d}(\bx))$. The approximation error between $ g_{\bv}(\bx)$ and $w_{\phi, \bv}^{\beta, \delta, K}$ is bounded as:
    \begin{equation}
        \begin{aligned}
            \left|  g_{\bv}(\bx) - w_{\phi, \bv}^{\beta, \delta, K}(\bx)\right| & \overset{(a)}{\leq} \epsilon_2 + \left| \prod_{l=1}^d\psi_{l,v_l}(\bx) - \prod_{l=1}^d w_{\phi, v_l}^{\beta, \delta, K}(x_l)\right| \\
            & \leq \epsilon_2 + \sum_{k=1}^{d} |\psi_{k,v_k}(\bx) - w_{\phi, v_k}^{\beta, \delta, K}(x_l)| \prod_{l=1}^{k-1} |w_{l, v_{l}}^{\beta, \delta,K}(x_l)| \prod_{l^{\prime}=k+1}^d |\psi_{l^{\prime},v_{l^{\prime}}}(\bx) | \\
            & \overset{(b)}{\leq}\epsilon_2 + d (2C_1+5)^{2d} \epsilon_1 \overset{(c)}{\leq} \epsilon,
        \end{aligned}
        \label{eq:app_varphiv_multi_weight}
    \end{equation}
    where $(a)$ follows from the approximation guarantee in \eqref{eq:app_nu_x1xd}, $(b)$ uses the approximation guarantee in \eqref{eq:app_psi_univariate_weight_lv}, along with the boundedness of $\psi_{l,v}$ and $w_{\phi, i}^{\beta, \delta, K}$, and $(c)$  holds by setting $\epsilon_1 = \frac{\epsilon}{2d(2C_1+5)^{2d}}, \epsilon_2 = \frac{\epsilon}{2} $.

    By the boundedness of $ w_{\phi,v_l}^{\beta, \delta, K} $ and the approximation guarantee in \eqref{eq:app_varphiv_multi_weight}, we further deduce that
    \begin{equation*}
        |g_{\bv} (\bx)| \leq \epsilon + (2C_1+3)^d \leq (2C_1+5)^d.
    \end{equation*}
    By the composition of the neural networks involved, we conclude that $g_{\bv}$ is a neural network of depth $5$, width $O(K)$, and its parameter norm is bounded by
    \begin{equation*}
        \|\theta(g_{\bv})\|_{\infty} \lesssim \max \left\{ \frac{K^2 \beta^2}{\epsilon_1 \delta^2}, \frac{1}{\epsilon_2^d}\right\} \lesssim \max \left\{ \frac{K^2 \beta^2}{\epsilon \delta^2}, \frac{1}{\epsilon^d}\right\},
    \end{equation*}
    as stated in ~\eqref{eq:app-width-norm-phi-multi-weight}.
    Finally, to ensure the validity of the above analysis, we require the parameter
     $\beta$ to satisfy
    \begin{equation*}
        \beta \geq \frac{(24C_1+60)K}{\delta \epsilon_1} \max \{C_1, C_2\} = \frac{24 d (2C_1+5)^{2d+1}}{\delta \epsilon} \max\{C_1, C_2\}.
    \end{equation*}
    This concludes the proof. 
\end{proof}

\subsection{Proof of Theorem \ref{thm:app_Linfty} (\texorpdfstring{$L^{\infty}$}{L-infinity} Approximation)}
\label{subapp:L-infty}

Building upon the neural approximators constructed in Lemma~\ref{lem:app-multi-weight-global} for the multivariate weight functions on the global domain $[0,1]^d$, we proceed to construct neural approximators for $f^{\star}$ with respect to the $L^{\infty}([0,1]^d)$ norm. To achieve this, we must extend the approximation guaranties of Theorem~\ref{thm:app_infinite_local}---established in the interior region---to the shifted interior regions.


\begin{theorem}
    \label{thm:app_shifted}
    Suppose $\phi$ and $f^{\star}$ satisfy the assumptions of Theorem~\ref{thm:app_infinite_local}, and let the parameters $\epsilon, K$, and $\delta$ be defined as therein. Then, for any shift index $\bv \in [2]^d$, there exists a neural network
    \begin{equation*}
        g_{\bv} \in \mathcal{H}^{\phi,6} (d, 1, M_{\epsilon}, B_{\epsilon,\delta}),
    \end{equation*}
    with the identical $M_{\epsilon}$ and $B_{\epsilon,\delta}$ specified in \eqref{eq:norm_width_app_infinity_local} of Theorem~\ref{thm:app_infinite_local}, such that the following properties hold for $g_{\bv}$ for $\bv \in [2]^d$:
    \begin{itemize}
        \item (Approximation) \[
        \|g_{\bv} - f^{\star}\|_{L^\infty\left( \Omega_{\mathrm{int}}^{K, \bv}(\delta)\right)} \le \epsilon.
        \]
        \item (Boundedness) For all $\bx \in [0,1)^d$, $g_{\bv}(\bx)$ satisfies the uniform bound established in \eqref{eq:upper_bound_app_local_infty} of Theorem~\ref{thm:app_infinite_local}.
    \end{itemize}
\end{theorem}
\begin{proof}
    The construction proceeds analogously to that of Theorem~\ref{thm:app_infinite_local}. The proof decomposes into two steps. First, we construct polynomial approximations for $f^{\star}$ locally within each shifted refined cell $\Omega_{\bi, \bj}^{K,\bv}$. Second, we approximate the resulting piecewise polynomial defined over the partition $\{\Omega_{\bi, \bj}^{K,\bv}\}_{\bi,\bj \in [\tilde{K}]^d}$ via a neural network. The second step employs the identical approximation scheme established in Lemma~\ref{lem:app_for_piece_poly}, adapted here to the shifted partition of $[0,1]^d$.
\end{proof}

We now present the proof of Theorem~\ref{thm:app_Linfty}.

\begin{proof}[Proof of Theorem~\ref{thm:app_Linfty}]
    We provide the proof for Heaviside-like activation functions; the proof for ReLU-like activations follows similarly. 

    First, let $\epsilon_1\in(0,1)$ be a sufficiently small parameter to be specified later. Following Theorem~\ref{thm:app_infinite_local}, define $K \coloneqq \lceil (2 c_1(s,d) /\epsilon_1)^{1/2s} \rceil$. Throughout the remainder of the proof, fix a parameter $\delta\in\bigl(0,\frac{1}{12K^2}\bigr)$, which will also be chosen later. By Theorem~\ref{thm:app_shifted}, there exists a family of $2^d$ neural networks $\{\psi_{\bv}\}_{\bv \in[2]^d}$ such that
    \begin{equation}
        |\psi_{\bv}(\bx) - f^{\star}(\bx)| \leq \epsilon_1, \quad \bx \in \Omega_{\mathrm{int}}^{K, \bv}(\delta).
        \label{eq:app-psiv-fstar}
    \end{equation}
    Each network $\psi_{\bv}$ has depth $6$, width $O(\epsilon_1^{-d/(2s)})$, and parameter norm $M_{\epsilon_1, \delta}$ as specified in \eqref{eq:norm_width_app_infinity_local}. Moreover, these networks are uniformly bounded as $\|\psi_{\bv}\|_{L^{\infty}([0,1]^d)} \leq \lceil s\rceil^dc_2(s,d) 4^{d+4}(1+C_1)^{2d}\eqqcolon M_{\psi}$.

     Secondly, let $\epsilon_2 \in (0,1)$ be a sufficiently small parameter to be determined later. Invoking Lemma~\ref{lem:app-multi-weight-global} (approximating the weight functions) and Corollary~\ref{corol:id-approx} (approximating the identity function), there exist $2^d$ neural networks $\{\nu_{\bv}\}_{\bv \in [2]^d}$ such that
     \begin{equation}
        |\nu_{\bv}(\bx) - w_{\phi, \bv}^{\beta, \delta,K}(\bx)| \leq \epsilon_2, \quad  \bx \in [0,1]^d,
        \label{eq:app-nuv-w}
    \end{equation}
    provided we choose $\beta = \frac{24d(2C_1+5)^{2d+1} K}{\delta \epsilon_2 } \max\{ C_1, C_2\}$. Each $\nu_{\bv}$ has depth $6$, width $O(K)$, parameter norm $O(\max \{ K^2 \beta^2/(\epsilon_2^2 \delta^2), \epsilon_2^{-d}\})$, and uniform output bound $\|\nu_{\bv} \| \leq (2C_1+6)^d \eqqcolon M_{\nu}$.

    We now bound the approximation error of the weighted combination $\sum_{\bv} \psi_{\bv} \nu_{\bv}$ against $f^{\star}$.
    For any $\bx \in [0,1]^d$, define the index set of ``active'' shifted interiors:
    \begin{equation*}
        V(\bx) \coloneqq  \left\{ \bv \in [2]^d : \bx \in \Omega_{\mathrm{int}}^{K, \bv}(\delta) \right\}.
    \end{equation*}
    For indices $\bv \in V(\bx)$, the approximation $|\psi_{\bv}(\bx) - f^{\star}(\bx)| \leq \epsilon_1$ holds. Conversely, let $\bar{V}(\bx) \coloneqq  [2]^d \setminus V(\bx)$ be the indices where $\bx$ falls into the shifted band region $\Omega_{\mathrm{band}}^{K, \bv}(\delta)$. By Proposition~\ref{prop:multi_weight}, the weight function nearly vanishes in this shifted band region, i.e., $|w_{\phi, \bv}^{\beta,\delta, K}(\bx)| \leq \epsilon_2, \quad \bv \in \bar{{V}}(\bx)$, as we have $\beta \ge \frac{2(2C_1+3)^{d-1}}{\delta \epsilon_2} \max\{C_1,C_2\}$. Since $\delta  < \frac{1}{4K^2}$, every $\bx$ must fall into at least one shifted interior region; hence, ${V}(\bx) \neq \varnothing$.

    Using the partition of unity property $\sum_{\bv} w_{\phi, \bv}^{\beta, \delta, K}(\bx) = 1$ in Proposition~\ref{prop:multi_weight}, we decompose the approximation error as :
    \begin{equation}
        \begin{aligned}
                    & \quad \quad \Bigg| \sum_{\bv \in [2]^d} \psi_{\bv}(\bx) \nu_{\bv}(\bx)  - f^{\star}(\bx)\Bigg| =  \Bigg|\sum_{\bv \in [2]^d} \psi_{\bv}(\bx) \nu_{\bv}(\bx) - \sum_{\bv \in [2]^d}f^{\star}(\bx) w_{\phi, \bv}^{\beta, \delta, K}(\bx)\Bigg| \\
                    & \leq \sum_{\bv \in [2]^d} \left|\psi_{\bv}(\bx)\right| \left|\nu_{\bv}(\bx) - w_{\phi, \bv}^{\beta, \delta, K}(\bx) \right| +  \sum_{\bv \in {V}(\bx)} \Big| \psi_{\bv}(x) - f^{\star}(\bx) \Big| \Big| w_{\phi, \bv}^{\beta, \delta,K}(\bx)\Big|    \\
                    &  \quad +   \sum_{\bv \in \bar{{V}}(\bx) } \left( |f^{\star}(\bx)| + \nu_{\bv}(\bx) \right) |w_{\phi, \bv}^{\beta, \delta, K}(x)| \\
                    & \overset{(a)}{\leq } 2^d M_{\psi} \epsilon_2 + 2^d(2C_1+3) \epsilon_1 + 2^d (1 + M_{\nu}) \epsilon_2  \overset{(b)}{\leq} \frac{\epsilon}{2},
        \end{aligned}
        \label{eq:app-psinu-f*}
    \end{equation}
    where $(a)$ follows from the approximation guarantee in \eqref{eq:app-psiv-fstar}, and \eqref{eq:app-nuv-w}, and the boundedness of $\psi_{\bv}, \nu_{\bv},f^{\star}$ and the locally quasi-vanishing property for $w_{\phi,\bv}^{\beta,\delta, K}$ on $\Omega_{\mathrm{band}}^{K, \bv}(\delta)$ from Proposition~\ref{prop:multi_weight}, $(b)$ holds by choosing 
    \begin{equation*}
        \epsilon_1 = \frac{\epsilon}{2^{d+2}(2C_1+3)}, \quad \epsilon_2 = \frac{\epsilon}{2^{d+2}(M_{\psi} + M_{\nu} + 1)}.
    \end{equation*}

    By Lemma~\ref{lemma:all_monomials}, there exists neural networks $\eta$ of depth $2$, width $4$ and parameter norm $O(1 / \epsilon^2)$, such that 
    \begin{equation}
        |\eta(x,y) - xy| \leq \frac{\epsilon}{2^{d+1}}, \quad 0 \leq |x|, |y| \leq \max \left\{ M_{\nu}, M_{\psi}\right\}.
        \label{eq:app-w-xy}
    \end{equation}
    We define the final approximator $g(\bx) \coloneqq  \sum_{\bv \in [2]^d} \eta(\psi_{\bv}(\bx), \nu_{\bv}(\bx))$. The total error is bounded by: 
    \begin{equation*}
        \begin{aligned}
                |g(\bx) - f^{\star}(\bx)| & \leq  \Big| \sum_{\bv \in [2]^d} \psi_{\bv}(\bx) \nu_{\bv}(\bx)  - f^{\star}(\bx)\Big| + \sum_{\bv \in [2]^d}  \Big| \eta(\psi_{\bv}(\bx), \nu_{\bv}(\bx)) - \psi_{\bv}(\bx) \nu_{\bv}(\bx)  \Big| \\
                & \overset{(a)}{\leq} \frac{\epsilon}{2} + 2^d \frac{\epsilon}{2^{d+1}} \leq \epsilon,
        \end{aligned}
    \end{equation*}
    where $(a)$ follows from the approximation guarantee in~\eqref{eq:app-psinu-f*} and \eqref{eq:app-w-xy}.

    Finally, regarding the architecture, $g$ is realized by summing the compositions of $\eta$ with the parallel sub-networks $\psi_{\bv}$ and $\nu_{\bv}$. By construction, the depth of $g$ is $7$. Regarding the width, recall that $\psi_{\bv}$ scales as $O(\epsilon^{-d/(2s)})$, since $\epsilon_1 \eqsim \epsilon$. Furthermore, given the scaling $K = O(\epsilon^{-1/(2s)})$ from Theorem~\ref{thm:app_infinite_local}, the width of $\nu_{\bv}$ scales as $O(K) = O(\epsilon^{-1/(2s)})$. Consequently, the total width of $g$ is dominated by the former term, bounded by $O(\epsilon^{-d/(2s)})$.
    Noting that $\beta = O(K / (\delta \epsilon_2))$ , $\epsilon_2 \eqsim \epsilon$, choosing $\delta = (24K^2)^{-1}$ (so that $\delta^{-1} = 24 K^2 = O((1 / \epsilon)^{\frac{1}{s}})$), then the parameter norm is bounded by 
    \begin{equation*}
        \max \left\{ \frac{1}{\epsilon_1^{\max \left\{ \frac{d^2}{2s}+d, \frac{d}{s}+2, \lceil s\rceil \right\} }} ,\frac{1}{\delta^2 \epsilon_1^{\frac{d}{2s}+1}}, \frac{K^2 \beta^2}{\delta^2 \epsilon_2^2},\frac{1}{\epsilon_2^d}, \frac{1}{\epsilon^d}\right\} \lesssim \left( \frac{1}{\epsilon}\right)^{\max \left\{ \frac{d^2}{2s}+d, \frac{d}{s}+2, \frac{d+4}{2s}+1, \lceil s \rceil, \frac{6}{s}+4\right\}},
    \end{equation*}
    as stated in~\eqref{eq:app-Linfty-width-norm}. This concludes the proof.
\end{proof}

\section{Learning Theory: Proofs and Technical Details}
\label{app:est}

\begin{lemma}[\citet{schmidt2020nonparametric}]
    Let $n \in \mathbb{N}{\geq 1}$, and let $f^{\star}$ and $\{(\bx_i, y_i)\}_{i=1}^n$ be given in~\eqref{eq:data_nonparametric}. Let $\mathcal{F}_n$ denote a model class, and let $\widehat{f}_n$ be the estimator defined as 
    \begin{equation}
        \widehat{f}_n =  \argmin_{f \in \mathcal{F}_n}  \frac{1}{n} \sum_{i=1}^n (f(\bx_i) - y_i)^2.
        \label{eq:est_ERM}
    \end{equation}
    Assume that ${f^{\star}} \cup \mathcal{F}_n \subset \{f:[0,1]^d \to [-F, F]\}$ for some $F \geq 1$. If the covering number $\mathcal{N}_n \coloneqq  \mathcal{N}(\tau, \mathcal{F}_n, \|\cdot\|_{\infty}) \ge 3 $, then, 
    \begin{equation*}
        \mathbb{E} \left[ \left\| \widehat{f}_n - f^{\star}\right\|_{L^2(\rho)}^2\right] \leq 4 \left[ \inf_{f \in \mathcal{F}_n}  \left\| f - f^{\star} \right\|_{L^2(\rho)}^2 + F^2 \frac{18 \log \mathcal{N}_n + 72 F}{n} + 32 \tau F\right],
    \end{equation*}
    for all $\tau \in (0,1]$.
    \label{lem:aos_lem}
\end{lemma}
\begin{remark}
    Lemma~\ref{lem:aos_lem} follows directly from Lemma 4 in \citet{schmidt2020nonparametric} by taking $\epsilon = 1$ and $P_X$ as the uniform distribution on $[0,1]^d$, and identifying $\widehat{f}$ with the empirical risk minimization estimator given in \eqref{eq:est_ERM}.
\end{remark}

\subsection{Covering number bounds}

Lemma~\ref{lem:aos_lem} characterizes the trade-off between approximation accuracy and the complexity of the model class, as measured by the covering number. We now provide an upper bound for the covering number of the model class defined in Section~\ref{sec:preliminary}.

\begin{lemma}[Covering number bound]
    Let $\phi$ satisfy Assumption~\ref{ass:Lip}.  The covering number of $\mathcal{H}^{\phi,L}(d, 1, W, B)$ with input $\bx \in [0,1]^d$
    can be bounded by 
    \begin{equation}
        \log \mathcal{N}(\tau, \mathcal{H}^{\phi,L}(d, 1, M, B), \left\|\cdot\right\|_\infty) \leq 2 (L + d) M^2 \log \left( \frac{4^{L+1} d( \max \{\|\phi\|_{\mathrm{Lip}},1 \} M)^{L} B^{L+1}}{\tau}\right),
        \label{eq:covering_entropy_bound}
    \end{equation}
    for $B \geq \max \left\{ 1, \frac{|\phi(0)|}{\max\{\|\phi\|_{\mathrm{Lip}},1\}(d+1)}\right\}$.
    \label{lem:covering_number}
\end{lemma}

\begin{proof}
    Without loss of generality we assume $\|\phi\|_{\mathrm{Lip}} \ge 1$.
    Now suppose that a pair of different two networks $g, \widehat{g} \in \mathcal{H}^{\phi,L}(d, 1, M,B)$ given by 
    \begin{equation*}
        g(\bx) = (\bW_{L} \phi(\cdot) + \bb_{L}) \circ \cdots \circ (\bW_1 \bx + \bb_{1}), \quad \widehat{g}(\bx) = (\widehat{\bW}_{L} \phi(\cdot) + \widehat{\bb}_{L}) \circ \cdots \circ (\widehat{\bW}_1 \bx + \widehat{\bb}_{1})
    \end{equation*}
    with
    \begin{equation*}
        \|\bW_{l} - \widehat{\bW}_l\|_{\infty, \infty} \leq \varrho, \quad \|\bb_l - \widehat{\bb}_l\|_{\infty} \leq \varrho, \quad 1 \leq l \leq L.
    \end{equation*}
    For the network $g, \widehat{g}$, we recurrently define $\{\bz_l\}_{l=0}^L, \{\widehat{\bz}_l\}_{l=0}^L$ as 
    \begin{equation*}
        \begin{aligned}
            & {\bz}_0 \coloneqq  \bx \in [0,1]^d, \quad \bz_1 \coloneqq  \bW_1 \bz_0 + \bb_1, \quad  \bz_l \coloneqq  \bW_l \phi(\bz_{l-1}) + \bb_l, \quad \text{for } 2 \le l \le L, \\
            & \widehat{\bz}_0 \coloneqq  \bx \in [0,1]^d, \quad \widehat{\bz}_1 \coloneqq  \widehat{\bW}_1 \widehat{\bz}_0 + \widehat{\bb}_1, \quad  \widehat{\bz}_l \coloneqq  \widehat{\bW}_l \phi(\widehat{\bz}_{l-1}) + \widehat{\bb}_l, \quad \text{for } 2 \le l \le L,
        \end{aligned}
    \end{equation*}
    The network outputs are then $g(\bx) = \bz_L, \widehat{g}(\bx) = \widehat{\bz}_L$. Firstly, we prove by induction that 
    \begin{equation*}
        \| \phi (\widehat{\bz}_l)\|_{\infty} \leq d(4 \|\phi\|_{\mathrm{Lip}}B)^{l} M^{l-1}, \quad 1 \leq l \leq L-1.
    \end{equation*}
    For the base case $l=1$, we have 
    \begin{equation*}
        \begin{aligned}
            \|\phi(\widehat{\bz}_1)\|_{\infty}  & \leq \|\phi\|_{\mathrm{Lip}} \|\widehat{\bz}_1\|_{\infty} + |\phi(0)| \\
            & \leq \|\phi\|_{\mathrm{Lip}} \left( \| \widehat{\bW}_1\|_{1, \infty} \|\bx\|_{\infty} + \|\widehat{\bb}_1\|_{\infty}\right) + |\phi(0)|  \\
            & \leq \|\phi\|_{\mathrm{Lip}}(d+1) B + |\phi(0)|  \leq 2 \|\phi\|_{\mathrm{Lip}}(d+1)B \leq 4 \|\phi\|_{\mathrm{Lip}} d B,
        \end{aligned}
    \end{equation*}
    For inductive steps, assume that for $k \leq l$, the inequality $\| \phi (\widehat{\bz}_k)\|_{\infty} \leq d(4 \|\phi\|_{\mathrm{Lip}}B)^{k} W^{k-1}$ holds. For $k = l+1$, we have 
    \begin{equation*}
        \begin{aligned}
            \| \phi (\widehat{\bz}_{l+1})\|_{\infty} & \leq \|\phi\|_{\mathrm{Lip}} \left( \| \widehat{\bW}_{l+1}\|_{1, \infty} \| \phi (\widehat{\bz}_l)\|_{\infty} + \|\widehat{\bb}_{l+1}\|_{\infty}\right) + |\phi(0)|  \\
            & \leq \|\phi\|_{\mathrm{Lip}} \left( d(4 \|\phi\|_{\mathrm{Lip}} B)^{l} M^{l-1} \times (MB) + B \right)  + |\phi(0)| \\
            & \leq 2 \|\phi\|_{\mathrm{Lip}} \left( d(4 \|\phi\|_{\mathrm{Lip}} M)^{l} B^{l+1} + B \right) \leq d (4 \|\phi\|_{\mathrm{Lip}} B)^{l+1} M^{l}.
        \end{aligned}
    \end{equation*}
    Thus, by induction, the bound holds for all $1 \leq l \leq L-1$.

    We now bound the error between two networks. In particular, we will prove by induction that
    \begin{equation*}
        \|\bz_{l} - \widehat{\bz}_l\|_{\infty} \leq 2d (4 \|\phi\|_{\mathrm{Lip}} M B)^{l-1}  \varrho.
    \end{equation*}
    For the base case $l=1$, we have:
    \begin{equation*}
        \begin{aligned}
            \|\bz_1 - \widehat{\bz}_1\|_{\infty} & =  \left\|(\bW_1 - \widehat{\bW}_1)\bx + (\bb_1 - \widehat{\bb}_1)\right\|_\infty \\
            & \leq \left\| \bW_1 - \widehat{\bW}_1\right\|_{1, \infty} \|\bx\|_{\infty} + \| \bb_1 - \widehat{\bb}_1 \|_{\infty} \leq (d+1) \varrho < 2 d M \varrho.
        \end{aligned}
    \end{equation*}
    For inductive steps, assume that for $k \leq l$, the inequality  $\|\bz_{l} - \widehat{\bz}_l\|_{\infty} \leq (4 \|\phi\|_{\mathrm{Lip}})^{l} M^{l} B^{l-1} \varrho$ holds. For $k = l+1$, we have 
    \begin{equation*}
        \begin{aligned}
            \left\|\bz_{l+1} - \widehat{\bz}_{l+1}\right\|_\infty &= \left\|\bW_{l+1} \phi(\bz_{l}) + \bb_{l+1} - (\widehat{\bW}_{l+1} \phi(\widehat{\bz}_{l}) + \widehat{\bb}_{l+1})\right\|_\infty \\
            &\le \left\|\bW_{l+1}\right\|_{1,\infty} \left\|\phi(\bz_{l}) - \phi(\widehat{\bz}_{l})\right\|_\infty + \left\|\bW_{l+1} - \widehat{\bW}_{l+1}\right\|_{1, \infty} \left\|\phi(\widehat{\bz}_{l})\right\|_\infty + \left\|\bb_{l+1} - \widehat{\bb}_{l+1}\right\|_\infty \\
            &\le (MB) \|\phi\|_{\mathrm{Lip}} \left\|\bz_{l} - \widehat{\bz}_{l}\right\|_\infty + (M\varrho) \left\|\phi(\widehat{\bz}_{l})\right\|_\infty + \varrho, \\
            & \leq (MB) \|\phi\|_{\mathrm{Lip}} \times 2d (4\|\phi\|_{\mathrm{Lip}} M B)^{l-1} \varrho + (M \varrho) \times  d (4 \|\phi\|_{\mathrm{Lip}} B)^{l} M^{l-1}+ \varrho \\
            & \leq 2 d (4\|\phi\|_{\mathrm{Lip}} MB )^{l}  \varrho.
        \end{aligned}
    \end{equation*}
    Thus, by induction, the approximation bound holds for all $1 \leq l \leq L$. Then by choosing $\varrho = \frac{\tau}{ 2d(4 \|\phi\|_{\mathrm{Lip}} MB)^{L}}$, we have 
    \begin{equation*}
        \|g - \widehat{g}\|_{L^{\infty}([0,1]^d)} \leq \|\bz_L - \widehat{\bz}_L\|_{\infty} \leq (4\|\phi\|_{\mathrm{Lip}} )^{L} M^{L} B^{L-1} \varrho \leq \tau. 
    \end{equation*}
    The total number parameters for $g, \widehat{g}$ is given by 
    \begin{equation*}
        \begin{aligned}
            P  = (M d + M) + \sum_{l=2}^{L-1} (M^2 + M) + (1 \cdot M + 1) = (L - 2) M^2 + (L + d ) M + 1 \leq 2 (L + d) M^2.
        \end{aligned}
    \end{equation*}
    Therefore, the covering number is bounded by 
    \begin{equation*}
        \mathcal{N}(\tau, \mathcal{H}^{\phi,L}(d, 1, M, B), \left\|\cdot\right\|_\infty) \leq \left( \frac{2B}{\varrho}\right)^{P} \leq \left( \frac{4^{L+1} d ( \|\phi\|_{\mathrm{Lip}} M)^{L} B^{L+1}}{\tau}\right)^{2(L+d) M^2},
    \end{equation*}
    which implies that 
    \begin{equation*}
        \log \mathcal{N}(\tau, \mathcal{H}^{\phi,L}(d, 1, M, B), \left\|\cdot\right\|_\infty)  \leq 2 (L+d) M^2 \log \left( \frac{4^{L+1}d( \|\phi\|_{\mathrm{Lip}} M)^{L} B^{L+1}}{\tau}\right),
    \end{equation*}
    as stated in~\eqref{eq:covering_entropy_bound}. This concludes the proof.
\end{proof}

\subsection{Proof of Theorem \ref{thm:sample-com} (Risk Bound)}
By Lemma~\ref{lem:aos_lem}, together with the approximation bound from Theorem~\ref{thm:app_L2} and the complexity estimate for our model class in Lemma~\ref{lem:covering_number}, we prove Theorem~\ref{thm:est} as follows.
\begin{proof}[Proof of Theorem~\ref{thm:est}] The estimator $\mathbb{T}_{F} \widehat{f}_n$ obtained by~\eqref{eq:ERM} can be interpreted as the following excess risk minimization estimator in  $\mathbb{T}_{F} \mathcal{H}^{\phi,L}(d, 1, M_n, B_n)$
\begin{equation*}
    \mathbb{T}_{F} \widehat{f}_n = \argmin_{h \in \mathbb{T}_{F} \mathcal{H}^{\phi,L}(d, 1, M_n, B_n)} \frac{1}{n} \; \sum_{i=1}^n \left(y_i - h(\bx_i)\right)^2,
\end{equation*}
where $\mathbb{T}_{F} \mathcal{H}^{\phi,L}(d, 1, M, B)$ is defined as 
\begin{equation*}
    \mathbb{T}_{F} \mathcal{H}^{\phi,L}(d, 1, M, B) = \left\{ \mathbb{T}_{F} f: f \in \mathcal{H}^{\phi,L}(d, 1, M, B) \right\}.
\end{equation*}
For sufficiently small $\epsilon>0$, invoking the $L^{\infty}$ approximation result established in Theorem~\ref{thm:app_Linfty} with $F = 2$ and recalling that $f^{\star}$ is continuous on $[0,1]^d$, 
we obtain the following approximation bound 
\begin{equation}
    \inf_{h \in \mathbb{T}_{F} \mathcal{H}^{\phi,7}(d, 1, M_{\epsilon}, B_\epsilon) } \|h - f^{\star}\|_{L^2(\rho)} \leq \epsilon,
    \label{eq:app-TF}
\end{equation}
provided that
\begin{equation*}
     M_{\epsilon} \eqsim \left( \frac{1}{\epsilon}\right)^{\frac{d}{2s}}, \quad B_{\epsilon} \eqsim \left( \frac{1}{\epsilon}\right)^{\max \left\{ \frac{d^2}{2s}+d, \frac{d}{s}+2, \frac{d+4}{2s}+1, \lceil s \rceil, \frac{6}{s}+4\right\}}.
\end{equation*}
For $g, \widehat{g} \in \mathcal{H}^{\phi,L}(d, 1, W, B)$, we have the inequality 
\begin{equation*}
    \|\mathbb{T}_{F} g - \mathbb{T}_{F} \widehat{g}\|_{L^{\infty}([0,1]^d)} \leq \|g  - \widehat{g}\|_{L^{\infty}([0,1]^d)},
\end{equation*}
which implies the following covering number bound
\begin{equation}
    \begin{aligned}
        & \quad \; \log \mathcal{N} \left(\tau, \mathbb{T}_{F}\mathcal{H}^{\phi,7}(d, 1, M_{\epsilon}, B_{\epsilon}), \left\|\cdot\right\|_\infty \right) \\
        &\leq \log \mathcal{N} \left(\tau, \mathcal{H}^{\phi,7}(d, 1, M_{\epsilon}, B_{\epsilon}), \left\|\cdot\right\|_\infty \right) \lesssim \left( \frac{1}{\epsilon}\right)^{\frac{d}{s}} \left(\log \frac{1}{\epsilon} + \log \frac{1}{\tau}\right),
    \end{aligned}
    \label{eq:cov-TF}
\end{equation}
where the ``$\lesssim$'' is due to the covering number bound established in Lemma~\ref{lem:covering_number}. Applying Lemma~\ref{lem:aos_lem}, we obtain the following bound
\begin{equation*}
    \mathbb{E}\left[ \left\| \mathbb{T}_{F} \widehat{f}_n - f^{\star}\right\|_{L^2(\rho)}^2  
    \right] \lesssim  \epsilon^2 + \frac{1}{n} \left( \frac{1}{\epsilon}\right)^{\frac{d}{s}} \left(\log \frac{1}{\epsilon} + \log \frac{1}{\tau}\right) + \tau.
\end{equation*}
By selecting $\epsilon \eqsim n^{-\frac{s}{2s + d}}$ and $ \tau \eqsim n^{-\frac{2s}{2s+d}}$, we derive the following convergence rate for excess risk
\begin{equation*}
    \mathbb{E} \left[ \left\| \mathbb{T}_{F} \widehat{f}_n - f^{\star}\right\|_{L^2(\rho)}^2  
    \right] \lesssim n^{-\frac{2s}{2s + d}} \log n,
\end{equation*}
as stated in~\eqref{eq: gen_rate_upper}. Additionally, we obtain the bounds for $M_n$ and $B_n$:
\begin{equation*}
     M_n \eqsim    n^{\frac{d}{4s+2d}} , \quad  B_n \eqsim n^{\max \left\{ \frac{d}{2}, 1, \frac{2s+d+4}{2(2s+d)}, \frac{s \lceil s\rceil}{2s+d},\frac{4s+6}{2s+d} \right\}},
\end{equation*}
as stated in ~\eqref{eq:est_LMBF}. This concludes the proof.
\end{proof}

\subsection{Optimal Risk Bound under \texorpdfstring{$\ell^2$}{l-2} Norm Constraints}

In this part we establish learning guarantees for ERM over neural networks subject to practically relevant $\ell^2$ parameter norm constraints. We begin by formally defining this hypothesis space, denoted by $\widetilde{\mathcal{H}}^{\phi}$, subject to an $\ell^2$ parameter bound:
\begin{equation}
    \begin{aligned}        &\widetilde{\mathcal{H}}^{\phi,L}(d_{\text{in}}, d_{\text{out}},M,B) = \Big\{  \bx \mapsto ( \bW_{L} \phi(\cdot) + \bb_{L}) \circ \cdots \circ (\bW_1 \bx + \bb_{1}): \\
        & \bW_1 \in\mathbb{R}^{M \times d_{\text{in}}} , \bW_L \in \mathbb{R}^{d_{\text{out}} \times M}, \bW_l \in \mathbb{R}^{M \times M}, \;  2 \leq l \leq L-1; \\
        & \; \bb_L \in \mathbb{R}^{d_{\text{out}}}, \bb_l \in \mathbb{R}^{M}, \;  1 \leq l \leq L-1; \; \sqrt{\sum_{l} \|\bW_l\|_F^2 + \|\bb_l\|_2^2} \leq B  \Big\}.
    \end{aligned}
    \label{eq:model-class-l2}
\end{equation}
Subsequently, we define the estimator $\widetilde{f}_n$ obtained by ERM over such class as
\begin{equation} 
    \tilde{f}_n = \argmin_{f \in \widetilde{\mathcal{H}}^{\phi,L}(d, 1,M_n,B_n)} \frac{1}{n} \sum_{i=1}^n \Big( y_i -(\mathbb{T}_F f)(\bx_i)\Big)^2.
    \label{eq:ERM-l2}
\end{equation}
The estimation error of this estimator is characterized by the following theorem.

\begin{theorem}
    \label{thm:est-l2}
    Suppose the assumptions on $\phi$ and $f^{\star}$ from Theorem~\ref{thm:est} hold. For sufficiently large $n \in \mathbb{N}_{+}$ (depending only on $\phi, s, d$), if we choose 
    \begin{equation*}
        L = 7, \quad M_n ~\eqsim n^{\frac{d}{4s+2d}}, \quad B_n ~\eqsim n^{\frac{d}{4s+2d}+\max \left\{ \frac{d}{2}, 1, \frac{2s+d+4}{2(2s+d)}, \frac{s \lceil s\rceil}{2s+d},\frac{4s+6}{2s+d} \right\}} , \quad F=2,
    \end{equation*}
    and let $\widetilde{f}_n$ be the estimator obtained via ~\eqref{eq:ERM-l2}, then 
    \begin{equation}
        \mathbb{E} \left[ \left\| \mathbb{T}_{F} \widetilde{f}_n - f^{\star}\right\|_{L^2(\rho)}^2  
        \right] \lesssim  n^{-\frac{2s}{2s + d}} \log n.
        \label{eq: gen_rate_upper-l2}
    \end{equation}
\end{theorem}
\begin{proof}
    We begin by establishing the inclusion relationship between the function classes. Observe that
    \begin{equation}
        \mathcal{H}^{\phi,L}(d, 1, M, B) \subset \widetilde{\mathcal{H}}^{\phi,L}(d, 1, M, \sqrt{P} B) \subset \mathcal{H}^{\phi,L}(d, 1, M, \sqrt{P} B),
        \label{eq:subset}
    \end{equation}
    where $P$ denotes the total number of parameters for networks within these classes, given by
    \begin{equation}
        P= M^2(L-2) + M(L+d) + 1 = O(M^2).
        \label{eq:para-number}
    \end{equation}
    By~\eqref{eq:app-TF}, \eqref{eq:subset} and~\eqref{eq:para-number}  recalling $f^{\star}$ is continuous on $[0,1]^d$, we have the following approximation error bound under $L^2(\rho)$ metric
    \begin{equation}
        \inf_{g \in \mathbb{T}_{F} \widetilde{\mathcal{H}}^{\phi,7}(d, 1, M_{\epsilon}, B_\epsilon) } \|g - f^{\star}\|_{L^2(\rho)} \leq  \epsilon,
    \end{equation}
    provided that 
    \begin{equation*}
        M_{\epsilon} \eqsim \left( \frac{1}{\epsilon}\right)^{\frac{d}{2s}}, \quad B_{\epsilon} \eqsim \left( \frac{1}{\epsilon}\right)^{\frac{d}{2s}+ \max \left\{ \frac{d^2}{2s}+d, \frac{d}{s}+2, \frac{d+4}{2s}+1, \lceil s \rceil, \frac{6}{s}+4\right\}}. 
    \end{equation*}
    Using the covering number bound \eqref{eq:cov-TF}  and the inclusion relationship in~\eqref{eq:subset}, we have
    \begin{equation*}
        \begin{aligned}
             &\quad \; \log \mathcal{N} \left(\tau, \mathbb{T}_{F} \widetilde{\mathcal{H}}^{\phi,7}(d, 1, M_{\epsilon}, B_{\epsilon}), \left\|\cdot\right\|_\infty \right) \\
             &\leq \log \mathcal{N} \left(\tau, \mathbb{T}_{F}{\mathcal{H}}^{\phi,7}(d, 1, M_{\epsilon}, B_{\epsilon}), \left\|\cdot\right\|_\infty \right) \lesssim \left( \frac{1}{\epsilon}\right)^{\frac{d}{s}} \left(\log \frac{1}{\epsilon} + \log \frac{1}{\tau}\right),
        \end{aligned}
    \end{equation*}
    when the input space is $[0,1]^d$.
    By following the same analysis as in the proof of Theorem~\ref{thm:est} to balance approximation error and model complexity, we establish the rate given in \eqref{eq: gen_rate_upper-l2}, provided that:
    \begin{equation*}
        M_n ~\eqsim n^{\frac{d}{4s+2d}}, \quad B_n ~\eqsim n^{\frac{d}{4s+2d}+\max \left\{ \frac{d}{2}, 1, \frac{2s+d+4}{2(2s+d)}, \frac{s \lceil s\rceil}{2s+d},\frac{4s+6}{2s+d} \right\}} .
    \end{equation*}
    This completes the proof. 
\end{proof}

\section{Supplementary Details to Section~\ref{sec:separation}}
\label{app:separation}

In this section, we establish the lower bounds on the approximation error for finite-depth neural networks employing non-smooth activation functions. Additionally, we detail the experimental setup used to compare the generalization error of smooth versus non-smooth activation functions.

\subsection{Proof of Proposition~\ref{prop:relu_lower_bound} (ReLU Lower Bound)}
\label{app:lower-bound}

We derive an \(L^2([0,1])\) lower bound for approximating functions in the Sobolev ball
\(\{f:\|f\|_{W^{s,\infty}([0,1])}\le 1\}\) by constant-depth ReLU networks. The proof consists of two parts:
(i) a direct piecewise-linear lower bound based on the quadratic test function
\(f^\star(x)=\tfrac12 x^2\), and (ii) a general lower bound for ReLU networks. 
While a similar argument for the first part appears in \citet{yarotsky2017error}, we present a self-contained and clarified proof for completeness.



\paragraph{Step 1: A local $L^2$ lower bound for linear approximation of $x^2$.}
\begin{lemma}
\label{lemma:diam}
Let $I\subset\mathbb{R}$ be a closed interval of length $l$, and let $h$ be any linear function. Then
\begin{equation}
\int_I (x^2-h(x))^2\,\dd x \;\ge\; \frac{1}{180}\,l^5.
\label{eq:app-x2L2}
\end{equation}
\end{lemma}

\begin{proof}
Let $I=[u,v]$ with $l=v-u$ and midpoint $m=(u+v)/2$. Define $e(x)=h(x)-x^2$. With the shift
$t=x-m$, since $h$ is linear we may write $e(t+m)=-t^2+at+b$ for some $a,b\in\mathbb{R}$. Hence
\begin{align*}
\int_I (e(x))^2\,\dd x
&=\int_{-l/2}^{l/2}(-t^2+at+b)^2\,\dd t \\
&=2\int_{0}^{l/2}\Big(t^4+(a^2-2b)t^2+b^2\Big)\,\dd t \\
&=\frac{l^5}{80}+\frac{l^3}{12}a^2-\frac{l^3}{6}b+lb^2.
\end{align*}
Completing the square in $b$ gives
\[
\int_I (e(x))^2\,\dd x
=\frac{l^5}{80}-\frac{l^5}{144}+\frac{l^3}{12}a^2
+l\Big(b-\frac{l^2}{12}\Big)^2
\;\ge\;\frac{1}{180}l^5,
\]
which proves~\eqref{eq:app-x2L2}.
\end{proof}

\paragraph{Step 2: A lower bound for piecewise-linear approximators.}

We next derive a quantitative lower bound for approximating the quadratic target using Step 1.

\begin{lemma}[Lower bound for approximating $x^2$ by ReLU networks]
\label{lem:lower_bound}
Let $f^\star(x)=\tfrac12 x^2$ on $[0,1]$. For any depth $L\ge 2$ and width $M\ge 2$,
\[
\inf_{g\in \mathcal{H}^{\mathrm{ReLU},L}(1,1,M,\infty)}
\|g-f^\star\|_{L^2([0,1])}
\;\ge\;
\frac{1}{12\sqrt{5}}\,(M+1)^{-2(L-1)}.
\]
\end{lemma}

\begin{proof}
Any ReLU network $g$ is a continuous piecewise linear function on $[0,1]$. Therefore there exists
a partition $\{I_j\}_{j=1}^K$ of $[0,1]$ into intervals such that $g$ is linear on each $I_j$.
In one dimension, each hidden layer of width $M$ can introduce at most $M$ new breakpoints within
each interval produced by the previous layers, so the total number of linear pieces satisfies
\begin{equation}
K \;\le\; (M+1)^{L-1}.
\label{eq:K_bound}
\end{equation}
Since $\sum_{j=1}^K |I_j|=1$ and $x\mapsto x^5$ is convex, Jensen's inequality yields
\begin{equation}
\sum_{j=1}^K |I_j|^5 \;\ge\; K\Big(\frac{1}{K}\sum_{j=1}^K |I_j|\Big)^5 \;=\; \frac{1}{K^4}.
\label{eq:Il-bound}
\end{equation}
On each interval $I_j$, $g$ is linear, so applying Lemma~\ref{lemma:diam} and scaling by $(1/2)^2$
gives
\[
\int_{I_j}\big(f^\star(x)-g(x)\big)^2\,\dd x
=\frac14\int_{I_j}\big(x^2-\tilde h_j(x)\big)^2\,\dd x
\;\ge\;\frac{1}{720}|I_j|^5,
\]
for some linear $\tilde h_j$ (namely $\tilde h_j=2g|_{I_j}$). Summing over $j$ and using
\eqref{eq:K_bound}--\eqref{eq:Il-bound}, we obtain
\[
\|g-f^\star\|_{L^2([0,1])}^2
=\sum_{j=1}^K\int_{I_j}\big(g(x)-f^\star(x)\big)^2\,\dd x
\;\ge\;\frac{1}{720}\sum_{j=1}^K |I_j|^5
\;\ge\;\frac{1}{720}\,K^{-4}
\;\ge\;\frac{1}{720}\,(M+1)^{-4(L-1)}.
\]
Taking square roots proves the claim.
\end{proof}

\paragraph{Step 3: A lower bound for ReLU networks.}
The following is a specialization of the lower bound in \citet{lu2021deep} and \citet{siegel2023optimal}. We restate it here for completeness.

\begin{lemma}[\citet{lu2021deep, siegel2023optimal}]
\label{lem:another-lowerbound}
Fix $L\ge 2$ and let $s>0$. Then there exists a constant $C_{s,L}>0$ such that for every $M\ge 2$,
\[
\sup_{\|f^\star\|_{W^{s,\infty}([0,1])}\le 1}\;
\inf_{g\in \mathcal{H}^{\mathrm{ReLU},L}(1,1,M,\infty)}
\|g-f^\star\|_{L^2([0,1])}
\;\ge\;
C_{s,L}\,(M^2\log M)^{-s}.
\]
\end{lemma}

\begin{proof}[Proof of Proposition~\ref{prop:relu_lower_bound}]
By Lemma~\ref{lem:lower_bound}, choosing $f^\star(x)=\tfrac12 x^2$ (which belongs to the Sobolev
unit ball under the standard $W^{s,\infty}$ normalization up to a constant factor) yields
\[
\sup_{\|f^\star\|_{W^{s,\infty}([0,1])}\le 1}\;
\inf_{g\in \mathcal{H}^{\mathrm{ReLU},L}(1,1,M,\infty)}
\|g-f^\star\|_{L^2([0,1])}
\;\ge\;
c\,(M+1)^{-2(L-1)}
\]
for some $c>0$.
On the other hand, Lemma~\ref{lem:another-lowerbound} gives
\[
\sup_{\|f^\star\|_{W^{s,\infty}([0,1])}\le 1}\;
\inf_{g\in \mathcal{H}^{\mathrm{ReLU},L}(1,1,M,\infty)}
\|g-f^\star\|_{L^2([0,1])}
\;\ge\;
C_{s,L}\,(M^2\log M)^{-s}.
\]
Combining the two bounds and using $M+1\le 2M$ and $\log M\ge \log 2$ for $M\ge 2$, we obtain
\[
\sup_{\|f^\star\|_{W^{s,\infty}([0,1])}\le 1}\;
\inf_{g\in \mathcal{H}^{\mathrm{ReLU},L}(1,1,M,\infty)}
\|g-f^\star\|_{L^2([0,1])}
\;\ge\;
C'_{s,L}\,(M\log M)^{-2\min\{L-1,s\}},
\]
after absorbing constants into $C'_{s,L}$.
\end{proof}

\subsection{Setup for Generalization Experiments}
\label{subapp:est_lower}

In this section, we describe the experimental setup of the generalization separation in detail.

\paragraph{Data generation and target function.}
We consider a target $f^\star:[0,1]^d\to\mathbb{R}$ of the form
\[
f^\star(\bx)=\sum_{k=1}^{K} a_k \cos(\bw_k^\top \bx + b_k).
\]
Throughout, we fix the input dimension $d=5$ and the number of random Fourier features $K=50$. The parameters are sampled independently as
\begin{itemize}
    \item $\bw_k \sim \mathcal{N}(0,I_d)$,
    \item $b_k \sim \mathcal{U}(0,2\pi)$,
    \item $a_k \sim \mathcal{N}(0,1)$.
\end{itemize}
For a given sample size $n$, the training dataset $\{(\bx_i,y_i)\}_{i=1}^n$ is generated by sampling $\bx_i \sim \mathcal{U}([0,1]^d)$ and setting
\[
y_i = f^\star(\bx_i) + \xi_i,\qquad \xi_i \sim \mathcal{N}(0,\sigma^2),
\]
with $\sigma=0.1$. We evaluate generalization performance across sample sizes
\[
n \in \{1024, 1448, 2048, 2896, 4096, 5792\}.
\]

\paragraph{Model architecture.}
We use a fully connected network with a single hidden layer to compare different activation functions. The hidden width is fixed to $M=6000$ for all experiments. We compare the non-smooth ReLU activation with two smooth activations, namely GELU and $\tanh$.


\paragraph{Training and hyperparameter tuning.}
We minimize the empirical mean-squared error (MSE)
\[
\frac{1}{n}\sum_{i=1}^{n}\bigl(y_i - f(\bx_i)\bigr)^2,
\]
where $f$ denotes the neural network predictor. Optimization is carried out using full-batch Adam for $50{,}000$ epochs with a cosine learning-rate decay schedule. For each sample size $n$, we perform a grid search over
\begin{itemize}
    \item learning rates $\eta \in \{10^{-4},10^{-3},10^{-2}\}$,
    \item $L^2$ regularization coefficients $\lambda \in \{10^{-5}, 5\times 10^{-5}, 10^{-4}, 5\times 10^{-4}, 10^{-3}, 5\times 10^{-3}, 10^{-2}, 5\times 10^{-2}, 10^{-1}\}$.
\end{itemize}
The hyperparameter pair $(\eta,\lambda)$ is selected by the smallest generalization error on a noiseless test set. We repeat the entire procedure over $5$ independent runs and report the average of the resulting best generalization errors. This protocol mitigates sensitivity to hyperparameter choices and yields a robust estimate of the empirical convergence behavior.

\end{document}